%% file: arxiv.tex
\documentclass[twoside]{article}

\usepackage[accepted]{aistats2024}
\usepackage{amsmath}
\usepackage{amssymb}
\usepackage{mathtools}
\usepackage{amsthm}
\theoremstyle{plain}
\newtheorem{theorem}{Theorem}
\newtheorem{proposition}{Proposition}
\newtheorem{lemma}{Lemma}
\newtheorem{corollary}{Corollary}
\theoremstyle{definition}
\newtheorem{definition}{Definition}
\newtheorem{assumption}{Assumption}
\theoremstyle{remark}
\newtheorem{remark}{Remark}
\usepackage{blindtext}
\usepackage[
    textcolor=blue,
    linecolor=blue,
    bordercolor=gray,
    backgroundcolor=white,textsize=tiny]{todonotes}
\input{defs}

\usepackage[round]{natbib}

\usepackage{pgf}
\usepackage{diagbox}
\makeatletter
\def\equalr#1{\protected@xdef\@thanks{\@thanks
        \protect\footnotetext{$^\dagger$#1}}}
\makeatother
 \makeatletter
\def\correspond#1{\protected@xdef\@thanks{\@thanks
        \protect\footnotetext{$^\ast$#1}}}
\usepackage{environ}
\newcommand{\acksection}{\section*{Acknowledgments and Disclosure of Funding}}
\NewEnviron{ack}{%
  \acksection
  \BODY
}

\begin{document}
\twocolumn[

\aistatstitle{Stability and Generalization of Stochastic Compositional Gradient Descent Algorithms}

\aistatsauthor{  Ming Yang$^\dagger$ \equalr{Equal contribution}\And  Xiyuan Wei$^\dagger$ \And  Tianbao Yang \And Yiming Ying }

\aistatsaddress{ University   at Albany, SUNY\\
       Albany, NY, United States\\
  \texttt{myang6@albany.edu} \\ \And  Texas A\&M University \\ 
College Station, TX, USA \\
  \texttt{xwei@tamu.edu}\\ \And   Texas A\&M University \\ 
College Station, TX, USA \\
  \texttt{tianbao-yang@tamu.edu}\\\And University  at Albany, SUNY\\
       Albany, NY, United States\\
  \texttt{yying@albany.edu}\\} ]

\input{arxiv_main_content}

\begin{ack}
 The work is partially supported by NSF grants under DMS-2110836, IIS-2103450, and IIS-2110546.
\end{ack}


\onecolumn
\newpage

\appendix
\input{appendix_lemma}

\input{appendix_problem_formulation.tex}
\input{appendix_convex.tex}
\input{appendix_strongly_convex.tex}

\end{document}

%% file: defs.tex
\usepackage{bm}
\usepackage{enumitem,amsthm,amsmath,amsfonts,amssymb}
\usepackage{indentfirst}
\usepackage{float}
\usepackage{mathtools}
\usepackage{changepage}
\usepackage{algorithmic}
\usepackage{algorithm}
\usepackage{xcolor}
\usepackage{graphicx}
\usepackage{tikz}
\usepackage{multirow}
\usepackage{multicol}
\usepackage{hyperref}
\usepackage{mathrsfs, fancyhdr}
\usepackage{amscd}
\usepackage{bbding}
\usepackage{pifont}
\usepackage{booktabs}

\newcommand{\R}{\mathbb{R}}

\newcommand\numberthis{\addtocounter{equation}{1}\tag{\theequation}}

\def\bw{\mathbf{w}}

\def\gep{\epsilon}

\def\X{\mathcal{X}}

\def\bx{\mathbf{x}}

\def\EX{\mathbb{E}}

\def\bw{\mathbf{w}}

\def\gep{\epsilon}

\def\0{\mathbf{0}}

\def\var{\hbox{Var}}

\def\begeqn{\begin{equation}}
\def\endeqn{\end{equation}}
\def\begth{\begin{theorem}}
\def\endth{\end{theorem}}
\def\begprop{\begin{proposition}}
\def\endprop{\end{proposition}}
\def\begcor{\begin{corollary}}
\def\endcor{\end{corollary}}
\def\begdef{\begin{definition}}
\def\enddef{\end{definition}}
\def\beglemm{\begin{lemma}}
\def\endlemm{\end{lemma}}
\def\begexm{\begin{example}}
\def\endexm{\end{example}}
\def\begrem{\begin{remark}}
\def\endrem{\end{remark}}

\def\begdef{\begin{definition}}
\def\enddef{\end{definition}}

\def\bw{\mathbf{w}}

\def\bx{\mathbf{x}}

\def\EX{\mathbb{E}}

\def\R{\mathbb{R}}

\def\X{\mathcal{X}}

%% file: arxiv_main_content.tex
\begin{abstract}
Many machine learning tasks can be formulated as a stochastic compositional optimization (SCO) problem such as reinforcement learning, AUC maximization, and meta-learning, where the objective function involves a nested composition associated with an expectation. While a significant amount of studies has been devoted to studying the convergence behavior of SCO algorithms, there is little work on understanding their generalization, i.e., how these learning algorithms built from training examples would behave on future test examples. In this paper, we provide the stability and generalization analysis of stochastic compositional gradient descent algorithms through the lens of algorithmic stability in the framework of statistical learning theory. Firstly, we introduce a stability concept called {\em compositional uniform stability} and establish its quantitative relation with generalization for SCO problems. Then, we establish the compositional uniform stability results for two popular stochastic compositional gradient descent algorithms, namely SCGD and SCSC.  Finally, we derive {\em dimension-independent} excess risk bounds for SCGD and SCSC by trade-offing their stability results and optimization errors. To the best of our knowledge, these are the first-ever-known results on stability and generalization analysis of stochastic compositional gradient descent algorithms.
\end{abstract}

\section{Introduction} \label{section1}
Recently,  {\em stochastic compositional optimization (SCO)} has gained considerable interests \cite[e.g.]{chen2021solving,chen2021tighter,dentcheva2017statistical,ghadimi2020single,hu2020sample,tolstaya2018nonparametric,wang2017stochastic,wang2016accelerating,zhang2020optimal} in machine learning. It  has the following form: 
\begeqn\label{eq:sco-true} \min_{x\in \X} \Bigl\{F(x) = {f}\circ g(x) = \EX_\nu[ f_\nu ( \EX_\omega[ g_\omega(x) ] ) ]\Bigr\},
\endeqn 
where $f\circ g(x) = f(g(x))$ denotes the function composition, $f: \R^d \to \R$ and $g: \R^p \to 
\R^d$ are differentiable functions,  $\nu$, $\omega$ are random variables, and $\X$ is a convex domain in $\R^p.$  
SCO generalizes the classic (non-compositional) stochastic optimization where its objective function $F(\cdot)$ involves nested compositions of functions and  each composition is  associated with an 
expectation.

SCO problem \eqref{eq:sco-true} instantiates a number of application domains.  For instance, reinforcement learning \citep{sutton2018reinforcement,szepesvari2010algorithms} aims to get 
a value function of the given policy  which can be regarded as an SCO problem \citep{wang2017stochastic}. The risk-averse
portfolio optimization \citep{shapiro2021lectures}, bias-variance issues in supervised learning \citep{dentcheva2017statistical,tolstaya2018nonparametric}, and group distributionally robust optimization \citep{jiang2022optimal} can also be formulated in similar SCO forms. Model-agnostic meta-learning (MAML) \citep{finn2017model} finds a common initialization for a quickly adaption  to  new tasks which was essentially an SCO problem as pointed out in \cite{chen2021solving}. The recent task of AUC maximization \cite[e.g.][]{kar2013generalization,liu2018fast,ying2016stochastic,AUCSurvey2023,zhao2011online} for imbalanced classification  aims to rank positive examples above  negative ones. In \cite{lei2021stochastic,AUCSurvey2023}, it can be regarded as an SCO problem: $ \min_{\bw\in \mathbb{R}^d}\mathbb{E}[\left(h_\mathbf{w}(\bx)) - a(\bw))\right)^2|y=1]+\mathbb{E}[\left(h_\mathbf{w}(\bx^{\prime})-b(\bw)\right)^2|y^{\prime}=1]+\left(1-a\left(\mathbf{w}\right)+b\left(\mathbf{w}\right)\right)^2,$ where $h_\bw(\cdot)$ is the decision function, $a(\bw) = \EX[h_\bw(\bx)|y=1]$, and $b(\bw) = \EX[h_\bw(\bx')|y'=-1].$ Likewise, other important  learning tasks such as the maximization of the area  under precision-recall curves (AUCPRC) and other compositional performance measures  can be cast in a similar fashion \citep{yang2022algorithmic}. 

There is a  substantial amount of studies devoted to studying the convergence behavior of stochastic compositional optimization algorithms for solving \eqref{eq:sco-true}. \cite{wang2017stochastic} pioneered the non-asymptotic   analysis of the so-called stochastic compositional gradient decent algorithms (SCGD) which employed two time scales with a slower stepsize for updating the variable   and a faster one used in the moving average sequence $y_{t+1}$ to track the inner function $g(x_t)$. An accelerated version of SCGD   has been
 analyzed in \cite{wang2016accelerating} and its adapted variant was studied in \cite{tutunov2020compositional}.  In particular, \cite{chen2021solving} proposed the stochastically corrected SCGD called SCSC which was shown to enjoy the same convergence rate as that of the standard SGD in the non-compositonal setting. Further extensions and their convergence analysis were investigated in different settings such that the  single timescale \citep{ghadimi2020single,ruszczynski2021stochastic},   variance reduction techniques \citep{hu2019efficient,devraj2019stochastic,lin2018improved}, and applications to non-standard learning tasks \citep{yang2022algorithmic}.  

On the other important front, one crucial aspect of machine learning is the development of learning algorithms that can achieve strong generalization performance. Generalization refers to the ability of a learning algorithm to perform well on unseen or future test data, despite being trained on a limited set of historical training data.  In the last couple of years, we have witnessed a large amount of work on addressing the generalization analysis of the vanilla stochastic gradient descent (SGD) with focus on the classical ERM formulation in the {\em non-compositional} setting. In particular,  stability and generalization of SGD have been studied using the uniform argument stability \citep{bassily2020stability,charles2018stability,hardt2016train,kuzborskij2018data} and on-average  model stability \citep{lei2020fine}.  In \cite{farnia2021train,lei2021stability,zhang2021generalization}, different stability and generalization measures are investigated for minimax optimization algorithms. However, to the best of our knowledge, there is no work on understanding the important stability and generalization properties of stochastic compositional optimization algorithms despite its  surging popularity in solving many machine learning tasks \cite[e.g.][]{chen2021solving,dentcheva2017statistical,jiang2022optimal,wang2017stochastic,AUCSurvey2023,yang2022algorithmic}.

\noindent\textbf{Our Contributions.} In this paper, we are mainly interested in the stability and generalization of stochastic compositional optimization algorithms in the framework of Statistical Learning Theory\cite[e.g.]{vapnik2013nature,bousquet2004introduction}.
Our main contributions are summarized as follows. 
\vspace*{-2mm}
\begin{itemize}[leftmargin=4mm]
 \setlength\itemsep{0.5mm}
\item We introduce a stability concept called {\em compositional uniform stability} which is tailored to handle the function composition structure in SCO problems. Furthermore, we show the qualitative connection between  this stability concept and the generalization error for randomized SCO algorithms. Regarding to the technical contributions, we show that this connection can mainly be derived by estimating the stability terms involving the outer function $f_\nu$ and the vector-valued generalization term of the inner function $g_\omega$ which will be further estimated using the sample-splitting argument \cite[e.g.]{bousquet2020sharper,lei2022nonsmooth}.

\item More specifically,  we establish the compositional uniform stability   of SCGD and SCSC in the convex and smooth case. Our stability bound mainly involves two   terms, i.e., the empirical variance  associated with the inner function $g_\omega$ and the convergence of  the moving-average sequence   to track $g_S(x_t)$. Then we establish the excess risk bounds $\mathcal{O}(1/\sqrt{n}+1/\sqrt{m})$ for both SCGD and SCSC by balancing the stability results and optimization errors, where $n$ and $m$ denote the numbers of training data  involving $\nu$ and $\omega$, respectively. Our results demonstrate that, to achieve the same excess risk rate of $\mathcal{O}(1/\sqrt{n}+1/\sqrt{m})$, SCGD requires a larger number of iterations, approximately $T\asymp \max(n^{3.5},m^{3.5})$, while SCSC only needs $T\asymp \max(n^{2.5},m^{2.5})$.

\item  We further extend the analysis of stability and generalization for SCGD and SCSC in the strongly convex and smooth case. Specifically, we demonstrate that SCGD requires approximately $T\asymp \max(n^{5/3},m^{5/3})$ iterations, while SCSC only needs $T\asymp \max(n^{7/6},m^{7/6})$ iterations to achieve the excess risk rate of $\mathcal{O}(1/\sqrt{n}+1/\sqrt{m})$. 

\end{itemize}

\vspace*{-2mm}
\subsection{Related Work} 
\vspace*{-2mm}
In this section, we review related works on algorithmic stability and generalization analysis of stochastic optimization algorithms, and algorithms for compositional problems.

\textbf{Stochastic Compositional Optimization.} 
The seminal work by \cite{wang2017stochastic} introduced SCGD with two time scales, and \cite{wang2016accelerating} presented an accelerated version. \cite{lian2017finite} incorporated variance reduction, while \cite{ghadimi2020single} proposed a modified SCGD with a single timescale. \cite{chen2021solving} introduced SCSC, a stochastically corrected version with the same convergence rate as vanilla SGD. \cite{ruszczynski2021stochastic, zhang2020optimal} explored problems with multiple levels of compositions, and \cite{wang2022finite} proposed SOX for compositional problems. Recently, there has been growing interest in applying stochastic compositional optimization algorithms to optimize performance measures in machine learning, such as AUC scores \cite[e.g.]{qi2021stochastic, lei2021stochastic, yang2022algorithmic}. Most of these studies have primarily focused on convergence analysis.

\textbf{Algorithmic Stability and Generalization for the Non-Compositional Setting.} Uniform stability and generalization of ERM were established by \cite{bousquet2002stability} in the strongly convex setting. \cite{elisseeff2005stability} studied stability of randomized algorithms, and \cite{feldman2019high, bousquet2020sharper} derived high-probability generalization bounds for uniformly stable algorithms. \cite{hardt2016train} established uniform argument stability and generalization of SGD in expectation for smooth convex functions. \cite{kuzborskij2018data} established data-dependent stability results for SGD. On-average model stability and generalization of SGD were derived in \cite{lei2020fine} for convex objectives in both smooth and non-smooth settings. Stability and generalization of SGD with convex and Lipschitz continuous objectives were studied in \cite{bassily2020stability}. For non-convex and smooth cases, stability of SGD was investigated in \cite{charles2018stability, lei2020fine, lei2022stability}. Further extensions were conducted for SGD in pairwise learning \citep{shen2019stability, yang2021simple}, Markov Chain SGD \citep{wang2022stability}, and minimax optimization algorithms \citep{farnia2021train, lei2021stability}. However, existing studies have primarily focused on SGD algorithms  and their variants for the standard ERM problem in the non-compositional setting.

Recently, \cite{hu2020sample} studied the generalization and uniform stability of the exact minimizer of the ERM counterpart for the SCO problem using the uniform convergence approach \citep{bartlett2002rademacher,vapnik2013nature,zhou2002covering}. They also showed uniform stability of its ERM minimizer under the assumption of a H\"{o}lderian error bound condition that instantiates strong convexity. Their bounds are algorithm-independent. To the best of our knowledge, there is no existing work on stability and generalization for stochastic compositional optimization algorithms, despite their popularity in solving machine learning tasks.

\noindent\textbf{Organization of the Paper.} The paper is organized as follows. Section 2 formulates the learning problem and introduces necessary stability concepts. Two popular stochastic compositional optimization algorithms, SCGD \citep{wang2017stochastic} and SCSC \citep{chen2021solving}, for solving \eqref{eq:sco-true} are presented. The main results on stability and generalization for SCGD and SCSC algorithms are illustrated in Section 3. Finally, Section 4 concludes the paper.

\section{Problem Setting}   \label{sec:problem}
In this section, we illustrate the target of generalization analysis and the stability concept used in the framework of Statistical Learning Theory \citep{vapnik2013nature,bousquet2004introduction}. Then, we describe two popular optimization schemes, i.e.,  SCGD and SCSC, for solving the SCO problems as well as other necessary notations.

\noindent\textbf{Target of Generalization Analysis.} For simplicity, we are mainly concerned with the case that the random variables $\nu$ and $\omega$ are independent which means that $g(\bx) = \EX[g_\omega (\bx)] = \EX[g_\omega (\bx) | \nu]$ for any $\nu$. This is the case which was considered in \cite{wang2017stochastic}. In practice, we do not know the population distributions for $\nu$ and $\omega$ for SCO problem \eqref{eq:sco-true} but only have access to a set of training data $ S  = S_\nu \cup S_\omega$ where both 
$ S_\nu = \bigl\{ \nu_i:  i =1, \ldots, n \bigr\}$ and $   S_\omega = \bigl\{\omega_j:  j =1, \ldots m \bigr\}$ are {\em distributed independently and identically (i.i.d.)}.  
As such, SCO problem \eqref{eq:sco-true} is reduced to the following nested empirical risk for SCO:   
\begeqn\label{eq:sco-erm} \min_{x\in \X}\bigl\{F_S(x) := f_S(g_S(x)) = \frac{1}{n}\sum_{i=1}^n f_{\nu_i}\bigl( \frac{1}{m} \sum_{j=1}^m g_{\omega_j}(x) \bigr)\bigr\},
\endeqn 
where $g_S: \R^p \to \R^d $ and $f_S: \R^d \to \R$ are the empirical versions of $f$ and $g$ in \eqref{eq:sco-true} and are defined, respectively, by $g_S(x)= \frac{1}{m} \sum_{j=1}^m g_{\omega_j}(x)$  and $ f_S(y) = {1\over n}\sum_{i=1}^n f_{\nu_i}(y).$
We refer to $F(x)$ and $F_S(x)$ as the {\em (nested) true risk and empirical risk}, respectively, in this stochastic compositional setting. 

Denote the least (nested) true and empirical  risks, respectively, by $F(x_\ast) = \inf_{x\in \X} F(x)$ and  $F(x^S_\ast) = \inf_{x\in \X} F_S(x).$ For a randomized algorithm $A$, denote by $A(S)$ its output model  based on  the training data $S$.   Then, our ultimate goal is to analyze the {\em excess generalization error (i.e., excess risk)} of $A(S)$  which is given by 
 $F(A(S)) - F(x_\ast).$
 It can be decomposed as follows:  {\small \begin{align*}\label{eq:error-decomp}
    &\EX_{S,A}[F(A(S)) - F(x_*) ] = \EX_{S,A}[F(A(S)) - F_S(A(S)) ] \\ & + \EX_{S,A}[F_S(A(S)) - F_S(x_*) ]\\ &
     \leq \EX_{S,A}[F(A(S)) - F_S(A(S)) ] \\ & + \EX_{S,A}[F_S(A(S)) - F_S(x_*^S) ], \numberthis 
\end{align*} }
where we have used the fact that $F_S(x_*^S)\leq F_S(x_*)$ by the definition of $x_*^S.$  The first term on the right hand side of \eqref{eq:error-decomp} is called the {\em generalization (error) gap} (i.e., estimation error) and the second term is the optimization error.  The optimization error (convergence analysis) in our study builds upon the analysis conducted in previous works such as \citep{wang2017stochastic,chen2021solving}. However, our main focus is on estimating the generalization gap using the algorithmic stability approach \citep{bousquet2002stability,hardt2016train,lei2020fine}. In order to achieve this, we introduce a proper definition of stability in the compositional setting, which will be outlined below. 

\noindent\textbf{Uniform Stability for SCO.} Existing work of stability analysis \cite[e.g.]{hardt2016train,kuzborskij2018data,lei2020fine} focused on SGD algorithms in the non-compositional ERM setting.  We will extend the algorithmic stability analysis to estimate the estimation error (i.e., generalization gap) for SCO problems.

In our new setting,   when we consider  neighboring training data sets differing  in one single data point,  the change of one data point can happen in either  $S_\nu$ or $S_\omega.$  In particular, for any $i\in [1,n]$ and $j\in [1,m]$, let $ S^{i,\nu}$ be the i.i.d copy of $S$ where only $i$-th data point $\nu_i$ in $S_\nu$ is changed to $\nu'_i$ while $S_\omega$ remains the same. Likewise,   denote by $ S^{j,\omega}$ the i.i.d copy of $S$ where only $j$-th data point $\omega_\ell$ in $S_\omega$ is changed to $\omega'_j$ while $S_\nu$ remains unchanged. Throughout the paper, we also denote by $S'= S'_\nu \cup S'_\omega$ the i.i.d. copy of $S$ where $S'_\nu = \{ \nu'_1,\ldots, \nu'_n\}$ and $S'_\omega = \{ \omega'_1,\ldots, \omega'_m\}.$

\begin{definition}[Compositional Uniform Stability]\label{def:stability} We say that a randomized algorithm $A$ is $(\gep_\nu,\gep_\omega)$-uniformly stable for SCO problem \eqref{eq:sco-true}  if, any $i\in [1,n]$, $j \in [1,m]$, there holds 
\begin{align*}
    \EX_A [\| A(S) - A(S^{i,\nu})\|] \le \gep_\nu, \\\hbox{ and } \EX_A[\| A(S) - A(S^{j,\omega})\|] \le \gep_\omega,\numberthis \label{eq:stable}
\end{align*}
where the expectation $\EX_A[\cdot]$ is taken w.r.t. the internal randomness of $A$ not the data points.
\end{definition}

We will show the  relationship between the {\em compositional uniform stability}  (i.e., Definition \ref{def:stability}) and the generalization error (gap) which holds true for any randomized algorithm. To this end, we need the following assumption. 
\begin{assumption} \label{assum:1}
    We assume  that $f_\nu$ and $g_\omega$ are Lipschitz continuous  with parameters $L_f$ and $L_g$, respectively, i.e., 
    \begin{enumerate}[label=({\roman*})]
        \item \label{assum:1a} $ \sup_{\nu}\|f_\nu(y)-f_\nu(\hat{y})\|\leq L_f\|y-\hat{y}\|$ for all $ y,\hat{y}\in \mathbb{R}^d.$
        \item \label{assum:1b} $\sup_{\omega}\|g_\omega(x)-g_\omega(\hat{x})\|\leq L_g\|x-\hat{x}\|$ for all $x,\hat{x}\in \mathbb{R}^p.$
    \end{enumerate}
\end{assumption}

The following theorem establishes the relationship between the stability of SCGD and its generalization.  
 
\begth  \label{thm:1} If Assumption \ref{assum:1} is true and the randomized algorithm $A$ is $\gep$-uniformly stable then 
\begin{align*}
        \EX_{S,A}\Big[  F(A(S))  - F_S(A(S)) \Big] \leq L_f L_g\epsilon_\nu+4L_fL_g\epsilon_\omega\\
     +L_f\sqrt{m^{-1}\mathbb{E}_{S,A}[\text{Var}_\omega(g_\omega(A(S)))]},
\end{align*}
where the variance term $\var_\omega(g_\omega(A(S))) = \EX_\omega\bigl[\| g_\omega(A(S)) - g(A(S))\|^2 \bigr].$
\endth

\begin{remark} Theorem 1 describes the relationship between the compositional uniform stability and generalization (gap) for any randomized algorithm for SCO problems. It can be regarded as an extension of the counterpart for the non-compositional setting \citep{hardt2016train}. Indeed,  if we let $g_\omega(x) = x$, then $g_S(x) = g_{\omega}(x) = x$ for any $\omega$ and $S$, the SCO problem is reduced to the standard non-compositional setting, i.e., $F(x) = \EX_\nu[f_{\nu}(x)]$ and $F_S(x) = {1\over n}\sum_{i=1}^n f_{\nu_i}(x).$  In this case, our result in Theorem 1 indicates, since there is no randomness w.r.t. $ \omega,$ that $\EX_{S,A}\big[F(A(S)) - F_S(A(S)) \big] \leq L_f \gep_\nu  $ which is exactly the case in the non-compositional setting \citep{hardt2016train}. 
\end{remark}

\begin{remark}\label{remark2} There are major technical challenges in deriving the relation between stability and generalization for SCO algorithms. To clearly see this, recall that, in the classical (non-compositional) setting, given i.i.d. data $S= \{z_1,\ldots,z_n\}$, the empirical and population risks are given by $F_S(A(S)) = \frac{1}{n} \sum_{i=1}^n f(A(S); z_i)$ and $F(A(S)) = \EX_z[ f(A(S); z]$, respectively. Let $S^i = \{z_1,\ldots, z_{i-1}, z'_i, z_{i+1},\ldots,z_n\}$ be the i.i.d. copy of $S$ but differs in the $i$-th data point. Using the symmetry between the i.i.d. datasets $S= \{z_1,\ldots, z_n\}$ and $S'= \{z'_1, z'_2, \ldots, z'_n\},$ one can immediately relate $\EX_{S,A} [F(A(S))-F_S(A(S))] = \EX_{S,A, S'} \big[\frac{1}{n} \sum_{i=1}^n f(A(S^i); z_i)-\frac{1}{n} \sum_{i=1}^n f(A(S); z_i)\big] \le L_f\|A(S^i) -A(S)\|.$ However, in our compositional setting, $\EX_{S,A}\big[  F(A(S))  - F_S(A(S)) \big]  = \EX_{S,A}\big[ \EX_\nu [f_\nu \big( g(A(S))\big)]  - \frac{1}{n}\sum_{i=1}^n f_{\nu_i}\bigl( g(A(S)) \bigr) \big] +  \EX_{S,A}\big[ \frac{1}{n}\sum_{i=1}^n \big( f_{\nu_i}\big( g(A(S)) \big) - f_{\nu_i}\big( \frac{1}{m} \sum_{j=1}^m g_{\omega_j}(A(S)) \bigr) \big)\big].$ The first term on the right-hand side of the above equality can be handled similarly to the non-compositional setting.  The main challenge comes from the second term  which, by the Lipschitz property of $f_\nu$,  involves a vector-valued generalization   $\EX_{S,A}\big[\big\|g(A(S))  - \frac{1}{m} \sum_{j=1}^m g_{\omega_j}(A(S)) \big \|\big]$ because one can not interchange the expectation and the norm. We will overcome this obstacle using the sample-splitting argument \cite[e.g.]{bousquet2020sharper,lei2022nonsmooth}
\end{remark}

\begin{algorithm}[t]
  \caption{\it (Stochastically Corrected) Stochastic Compositional Gradient Descent} \label{alg:1}
  \begin{algorithmic}[1]
    \STATE {\bf Inputs:} Training data $S_\nu = \bigl\{ \nu_i:  i =1, \ldots, n \bigr\}, \quad  S_\omega = \bigl\{\omega_j:  j =1, \ldots, m \bigr\}$; Number of iterations $T$, parameters $\{\eta_t\}, \{\beta_t\}$
    \STATE {Initialize $x_0 \in \mathcal{X}$ and  $y_0\in \mathbb{R}^d$} 
    \FOR {$t=0$ to $T-1$}
      \STATE {Randomly sample ${j_t}\in [1,m]$, obtain $g_{\omega_{j_t}}(x_t)$ and $\nabla g_{\omega_{j_t}}(x_t)\in \mathbb{R}^{p\times d}$ }
      \STATE {\textbf{SCGD update}: $y_{t+1}  = (1-\beta_t) y_{t} + \beta_t g_{\omega_{j_t}}(x_t)  $}
      \STATE {\textbf{SCSC update}: $y_{t+1}  = (1-\beta_t) y_{t} + \beta_t g_{\omega_{j_t}}(x_t)+ (1- \beta_t)(g_{\omega_{j_t}}(x_t)- g_{\omega_{j_t}}(x_{t-1}))$}
      \STATE {Randomly sample ${i_t}\in [1,n]$, obtain $\nabla f_{\nu_{i_t}}(y_{t+1})\in \mathbb{R}^{d}$}
      \STATE \textbf{Update:}
      \STATE {$x_{t+1} =\Pi_{\mathcal{X}}\bigl(x_{t} -\eta_t \nabla  g_{\omega_{j_t}}(x_t) \nabla f_{\nu_{i_t}}(y_{t+1})  \bigr)$} \
    \ENDFOR
    \STATE {\bf Outputs:} $A(S)  =x_{T}$ or $x_\tau \sim \texttt{Unif}(\{x_t\}_{t=1}^T)$
  \end{algorithmic}
\end{algorithm}

\noindent\textbf{Optimization Algorithms.} We will study two popular optimization algorithms for solving \eqref{eq:sco-erm}, i.e., SCGD \citep{wang2017stochastic} and SCSC \citep{chen2021solving}. Their pseudo-code is given in Algorithm \ref{alg:1} where  a sequence of $$y_{t+1}=\left(1-\beta_t\right)y_t+\beta_t g_{\omega_{j_t}}\left(x_t\right)$$  is used to track the expectation of $g_S(x_t) = \EX_{j_t}[g_{\omega_{j_t}}(x_t)] = {1\over m} \sum_{j=1}^m g_{\omega_j}(x_t) $ (see Line 5 in Algorithm \ref{alg:1}).    
 As shown in \cite{wang2017stochastic}, SCGD needs to choose a smaller  stepsize $\eta_t$ than the stepsize $\beta_t$ to be convergent. This prevents the SCGD from choosing the same stepsize as SGD for the non-compositional stochastic problems.  To address this issue,   \cite{chen2021solving} proposed a stochastically corrected version of SCGD which is referred to as SCSC. In particular, the sequence of $y_{t+1}$ is now given as follows (see Line 6 in Algorithm \ref{alg:1}):   $$ y_{t+1}=(1-\beta_t)(y_t+g_{\omega_{j_t}}(x_t)-g_{\omega_{j_t}}(x_{t-1}))+\beta_t g_{\omega_{j_t}}(x_t).$$

We list definitions about strong convexity and smoothness which will be used in subsequent sections. 
 \begin{definition}
    A function $F: \mathbb{R}^p \to \mathbb{R} $ is $\sigma$-strongly convex with some $\sigma\ge 0$ if, for any $u, v\in \mathbb{R}^p$, we have $ 
         F\bigl(u\bigr)\ge F\bigl(v\bigr)+ \left \langle \nabla F\bigl(v\bigr), u-v\right \rangle +\frac{\sigma}{2} \|u-v\|^2.$
    If $\sigma=0$, we say that $F$ is convex. 
\end{definition}
The following is that the smoothness property of $F$ leads to a bound on the gradient update.
\begin{definition}\label{def 3}
    A function $F: \mathbb{R}^p \to \mathbb{R} $ is $L$-smooth if, for any $u, v\in \mathbb{R}^p$, we have $
        \|\nabla F \bigl(u\bigr)-\nabla F \bigl(v\bigr)\|\leq L \|u-v\|.$ 
\end{definition}
In general, smoothness implies the gradient update of F cannot be overly expansive.  Also the convexity and $L$-smooth of F implies that the gradients are co-coercive, hence we have\vspace{-3mm}
\begin{align}\label{co-coercive}
    \left \langle \nabla F\bigl(u\bigr)-\nabla F\bigl(v\bigr), u-v\right \rangle\ge \frac{1}{L}\|\nabla F\bigl(u\bigr)-\nabla F\bigl(v\bigr)\|^2.
\end{align}
Note that if $F$ is $\sigma$ strongly convex, then $\varphi\bigl(x\bigr)=F\bigl(x\bigr)-\frac{\sigma}{2}\|x\|^2$ is convex with $\bigl(L-\sigma\bigr)$-smooth. Then, applying (\ref{co-coercive}) to $\varphi$ yields the following inequality:\vspace{-3mm}
\begin{align*} 
    \langle \nabla F\left(u\right)-\nabla F(v), u-v \rangle 
        \ge\frac{L\sigma}{L+\sigma}\|u-v\|^2\\+\frac{1}{L+\sigma}\|\nabla F(u)-\nabla F(v)\|^2.\numberthis   \label{sm+str}
\end{align*}

\section{Stability and Generalization}\label{sec:results}
In this section, we will present our main results on estimating the stability bounds for SCGD and SCSC which subsequently can lead to estimation of their generalization gaps from Theorem \ref{thm:1}. Then,  we start from the error decomposition \eqref{eq:error-decomp} to derive the   bounds for their excess risks by trade-offing the bounds for the above  generalization (error) gaps and  optimization errors.
We will present results  in two different cases, i.e.,   convex and strongly convex settings,  in different subsections.
For brevity, we summarize our results for the excess risks for both SCGD and SCSC in Table \ref{table1}. Before illustrating our main results, we list some assumptions.  
\begin{assumption} \label{assum:2}
    We assume that the following conditions hold true.
    \vspace{-3mm}
    \begin{enumerate}[label=({\roman*})]
        \item \label{part1} With probability 1 w.r.t $S$, there holds $\sup_{x\in \X} {1\over m}\sum_{j=1}^m\bigl[\| g_{\omega_j}(x) - g_S(x)\|^2 \bigr]\leq V_g$.
        \item \label{part1b} With probability 1 w.r.t $S$, there holds $\sup_{x\in \X}{1\over m}\sum_{j=1}^m[\|\nabla g_{\omega_{j}}(x)-\nabla g_S(x)\|^2]\leq C_g$.
        \item \label{part2} With probability $1$ w.r.t. $\nu$, the function $f_\nu(\cdot)$ has Lipschitz continuous gradients, i.e., $ \|\nabla f_{\nu}(y)- \nabla f_{\nu}(\bar{y})\|\leq C_f\|y- \bar{y}\|$ for all $y, \bar{y}\in \mathbb{R}^d$.
         \item \label{part3} With probability 1 w.r.t $\nu$ and $S$, the function  $f_\nu(g_S(\cdot))$ is $L$-smooth, i.e., $\|\nabla g_S(x) \nabla f_\nu(g_S(x)) - \nabla g_S(x') \nabla f_\nu(g_S(x'))\| \le L \|x-x'\|$ for any $x,x'\in \X$.
    \end{enumerate}
\end{assumption}

\subsection{Convex Setting}
In this subsection, we present our main results for SCGD and SCSC in the convex setting. 

\textbf{Stability Results.} The following theorem  establishes  the \textit{compositional uniform Stability} (See Definition \ref{def:stability}) for SCGD and SCSC in the convex setting.   

 \begin{table*}[t]
	\centering
 \caption{Number of Iterations T that Achieves Excess Risk for SCGD And SCSC Algorithm}
 \label{table1}
\begin{tabular}{|c|c|c|c|}
	\hline
	\multicolumn{2}{|c|}{\diagbox{setting}{algorithm}} &  SCGD&SCSC\\
	\hline
	\multirow{3}*{Convex \(F_S\)}
	&\# Iterations&$T \asymp \max(n^{3
 .5},m^{3.5})$&$T \asymp \max(n^{2.5},m^{2.5})$\\
	\cline{2-4}
	&Excess risk  &$\mathcal{O}\left(\frac{1}{\sqrt{n}}+\frac{1}{\sqrt{m}}\right)$&$\mathcal{O}\left(\frac{1}{\sqrt{n}}+\frac{1}{\sqrt{m}}\right)$\\
	\hline
	\multirow{3}*{Strongly Convex \(F_S\)}
	&\# Iterations&$T \asymp \max(n^{5/3},m^{5/3})$&$T \asymp \max(n^{7/6},m^{7/6})$\\
	\cline{2-4}
	& Excess risk &$\mathcal{O}\left(\frac{1}{\sqrt{n}}+\frac{1}{\sqrt{m}}\right)$&$\mathcal{O}\left(\frac{1}{\sqrt{n}}+\frac{1}{\sqrt{m}}\right)$\\
	\hline
\end{tabular}
\end{table*}

 \begin{theorem}[Stability, Convex]\label{thm:cov_sta}
Suppose that Assumption \ref{assum:1} and \ref{assum:2} hold true and $f_{\nu}(g_S(\cdot))$ is  convex.  Consider Algorithm \ref{alg:1} with $\eta_t=\eta\leq \frac{1}{2L}$, and $\beta_t=\beta \in (0,1)$ for any $t\in [0,T-1]$. Then,   the outputs $A(S)  =x_{T}$ of both SCGD and SCSC at iteration $T$ are compositional uniformly stable with
\vspace{-2mm}
\begin{align*}
    \gep_\nu+ \gep_\omega\! = \!\mathcal{O}\Bigl( \frac{L_fL_g}{n} \eta T +\frac{L_fL_g}{m} \eta T   
+\sqrt{C_g}L_f\eta\sqrt{T} \\+ C_fL_g\sup_{S}\sum_{j=0}^{T-1} \eta \bigl(\mathbb{E}_{A}[\| y_{j+1}-g_S(x_j)\|^2 ] \bigr)^{1\over 2}\Bigr).\numberthis\label{stab_convex}
\end{align*}
\vspace{-6mm}
\end{theorem}
The proof for the above theorem will be given in Appendix \ref{pr_stab_con}. 
\begin{remark}\label{rem:stability} In this remark, we discuss how the function composition  plays a role in the stability analysis for SCGD and SCSC and then compare our results with that for SGD in the non-compositional setting \citep{hardt2016train}. To  this end, 
considering the step sizes $\eta_t = \eta$ and $n=m$,  then \eqref{stab_convex} is reduced to the following estimation:
\vspace{-2mm}
\begin{align*}
    \gep_\nu+ \gep_\omega= & \mathcal{O}\Bigl( \frac{\eta T}{n}  +\sqrt{C_g}\eta\sqrt{T}\\ &+ \eta \sup_{S} \sum_{j=0}^{T-1} \bigl(\mathbb{E}_{A}[\|y_{j+1}-g_S(x_j)\|^2]\bigr)^{\frac{1}{2}}\Bigr).\numberthis\label{eq:inter}\vspace{-3mm}
\end{align*}
It was shown in \cite{hardt2016train} that  the uniform stability for SGD with convex and smooth losses is   of the order $\mathcal{O}\bigl( {\eta T \over n}\bigr)$. By comparing these two results,  we can see how the compositional structure plays a role in the stability analysis. Indeed, in contrast to the result for SGD, there are two extra terms in \eqref{eq:inter} for SCGD and SCSC, i.e., $\sqrt{C_g}\eta\sqrt{T}$  and $ \eta \sup_{S} \sum_{j=0}^{T-1} \bigl(\mathbb{E}_{A}[\|y_{j+1}-g_S(x_j)\|^2]\bigr)^{\frac{1}{2}}.$ Here, $C_g$ is  the (empirical) variance of the gradient of inner function, i.e., $\sup_{x\in \X}{1\over m}\sum_{j=1}^m\|\nabla g_{\omega_{j}}(x)-\nabla g_S(x)\|^2\leq C_g$ given in Assumption \ref{assum:2} and the other extra term arises when the moving-average sequence $y_{t+1}$ is used to track $g_S(x_t)$. Notice that, if we let $g_\omega(x) = x$, then $g_S(x) = g_{\omega}(x) = x$ for any $\omega$ and $S$, then SCGD and SCSC reduce to the classical SGD, and our stability result \eqref{eq:inter} is the same as that of SGD   since two extra terms mentioned above will be all zeros due to the fact that $y_{j+1} = g_S(x_j) = x_j$ and $C_g = 0$ in this  case.   
\end{remark}

Combining \eqref{stab_convex} with the estimation for  $\mathbb{E}_{A}[\| y_{j+1}-g_S(x_j)\|^2]$ \citep{wang2017stochastic,chen2021solving} (see also Lemma \ref{Mengdi-lemma} and its self-contained proof in Appendix \ref{tech_lemma}), one can get the following explicit stability results.  
\begin{corollary}\label{cor:1}Let Assumption \ref{assum:1} and \ref{assum:2} hold true and $f_{\nu}(g_S(\cdot))$ be convex. Consider Algorithm \ref{alg:1} with $\eta_t=\eta\leq \frac{1}{2L}$, and $\beta_t=\beta \in (0,1)$ for any $t\in [0,T-1]$ and the output  $A(S)  =x_{T}$. Let $c$ be an arbitrary constant. Then, we have the following results: 
\vspace{-3mm}
\begin{itemize}[leftmargin=3mm]
 \setlength\itemsep{0.3mm}
\item SCGD is compositional uniformly stable with\vspace{-1mm}
\begin{align*}
     \epsilon_\nu + \epsilon_\omega= \mathcal{O}\bigl(&\eta T n^{-1}+\eta T m^{-1}+\eta T^{\frac{1}{2}}\\&+\eta T^{-c/2+1}\beta^{-c/2}+\eta^2 \beta^{-1}T+\eta\beta^{1/2}T\bigr).
\end{align*}
\vspace{-8mm}\item SCSC is compositional uniformly stable with 
\begin{align*}
    \epsilon_\nu + \epsilon_\omega=\mathcal{O}\bigl(&\eta T n^{-1}+\eta T m^{-1}+\eta T^{\frac{1}{2}}\\&+\eta T^{-c/2+1}\beta^{-c/2}+\eta^2 \beta^{-\frac{1}{2}}T+\eta\beta^{1/2}T\bigr).
\end{align*}
\end{itemize}\vspace{-3mm}
\end{corollary}

\textbf{Generalization results.} Using the error decomposition  \eqref{eq:error-decomp},  Corollary \ref{cor:1} and Theorem \ref{thm:1}, we can derive the excess risk rates. To this end, we need the following results to estimate the optimization error, i.e., $F_S(A(S))- F_S(x_*^S).$
 
\begin{theorem}[Optimization, Convex] \label{thm:opt_convex}
    Suppose Assumption \ref{assum:1} and \ref{assum:2} \ref{part1}, \ref{part2} hold for the empirical risk $F_S$ and $F_S$ is convex, \(\EX_A \|x_t- x_*^S\|^2\) is bounded by \(D_x\) for all $t\in [0,T-1]$ and \(\EX_A \|y_1- g_S(x_0)\|^2\) is bounded by \(D_y\). Let \(A(S)= \frac{1}{T}\sum_{t= 1}^T x_t\) be the solution produced by Algorithm \ref{alg:1} with SCGD or SCSC update, $\eta_t=\eta$ and $\beta_t=\beta$ for some $a,b\in \left(0,1\right]$. Let $c$ be an arbitrary constant.
    \vspace{-3mm}
    \begin{itemize}[leftmargin=3mm]
        \setlength\itemsep{0.5mm}
        \item For SCGD update,  there holds
        \vspace{-1mm}
            \begin{align*}
                \EX_A[F_S(A(S))- &F_S(x_*^S)] \\
                = \mathcal{O}\Bigl(D_x(\eta T)^{-1}&+ L_f^2L_g^2\eta+ C_fD_y(\beta T)^{1-c}(\eta T)^{-1}\\&+ C_fV_g\beta^2 \eta^{-1}+ C_fL_f^2L_g^3D_x\eta \beta^{-1}\Bigr).
            \end{align*}
        \vspace{-8mm}\item For SCSC update, there holds \vspace{-1mm}
            \begin{align*}
                \EX_A[F_S(A(S))- &F_S(x_*^S)] \\
                = \mathcal{O}\Bigl(D_x(\eta T)^{-1}&+ L_f^2L_g^2\eta+ C_fD_y(\beta T)^{-c}\beta^{-\frac{1}{2}}\\+ C_fV_g\beta^{\frac{1}{2}}&+ C_fL_f^2L_g^3\eta^2 \beta^{-\frac{3}{2}}+ C_fL_g^2D_x \beta^{\frac{1}{2}}\Bigr).
            \end{align*}
    \end{itemize}
\end{theorem}
\vspace{-3mm}
The boundedness assumptions are satisfied if the domain $\X$ is bounded in $\mathbb{R}^p$. The detailed proofs are given in Appendix \ref{pf_opt_cv} and \ref{gene_cv}. Note that the upper-bounds for the optimization error given in the above theorem hold   true uniformly for any training data $S$.



Combining the above results with the stability bounds in Corollary \ref{cor:1} and Theorem \ref{thm:1}, we can derive the following excess risk bounds for SCGD and SCSC. 

\begin{theorem}[Excess Risk Bound, Convex] \label{thm:gen_convex}
    Suppose Assumptions \ref{assum:1} and \ref{assum:2} hold true and $f_{\nu}(g_S(\cdot))$ is convex, \(\EX_A [\|x_t- x_*^S\|^2]\) is bounded by \(D_x\) for all $t\in [0,T-1]$ and \(\EX_A [\|y_1- g_S(x_0)\|^2]\) is bounded by \(D_y\). Let $A(S)= \frac{1}{T}\sum_{t= 1}^T x_t$ be a solution produced by Algorithm \ref{alg:1} with SCGD or SCSC update and $\eta=T^{-a}$ and $\beta=T^{-b}$ for some $a,b\in \left(0,1\right]$.  \vspace{-4mm}
\begin{itemize}[leftmargin=3mm]
 \setlength\itemsep{0.5mm}
\item     If we select $T \asymp \max(n^{3.5},m^{3.5})$, $\eta=T^{-\frac{6}{7}}$ and $\beta=T^{-\frac{4}{7}}$, then, for the SCGD update, we have that  $ 
        \EX_{S,A}\big[F(A(S)) - F(x_*) \big]= \mathcal{O}\bigl(\frac{1}{\sqrt{n}}+\frac{1}{\sqrt{m}}\bigr).$ 
\item If we select $T \asymp \max(n^{2.5},m^{2.5})$, $\eta=T^{-\frac{4}{5}}$ and $\beta=T^{-\frac{4}{5}}$, then, for the SCSC update, there holds $
        \EX_{S,A}\Big[F(A(S)) - F(x_*) \Big]= \mathcal{O}(\frac{1}{\sqrt{n}}+\frac{1}{\sqrt{m}}).$
        \end{itemize}
\end{theorem}

\begin{remark}
In the recent work \citep{hu2020sample}, the uniform convergence using the concentration inequalities and the covering number are used to study the generalization gap (estimation error) of the ERM minimizer related to SCO problems. Applying their results to our case,   they proved the following results:  assuming that $\X$  is a bounded domain, $f_\nu$ and $g_\omega$ are both Lipschitz continuous and bounded, there holds, with high probability, 
$ F(A(S)) - F_S(A(S)) \le \sup_{x\in \X} |F(x - F_S(x)| = \mathcal{O}\Bigl( \sqrt{p \over m+n}\Bigr)$ which is highly dependent on the dimension of the domain $\X\subseteq\R^p.$ 
Comparing with their bounds, we can get excess risk bounds which is {\em dimension independent}.   Dimension-independent generalization bounds were also provided in \cite{hu2020sample} which requires the H\"{o}lder error bound condition (e.g., strong convexity). The proof there heavily depends on the property of the ERM minimizer of the SCO problem and does not apply to SCGD and SCSC.  
 \end{remark}
\begin{remark}
    Theorem \ref{thm:gen_convex} shows that the generalization error for  SCGD can be achieved the rate $\mathcal{O}\bigl(1/\sqrt{n}+1/\sqrt{m}\bigr)$ in the convex case after   selecting appropriately the iteration number $T$ and step sizes $\eta$ and $\beta$. Recall that  
 in the non-compositional setting, \cite{hardt2016train,lei2020fine} established generalization error bounds $\mathcal{O}\bigl(1/\sqrt{n}\bigr)$ by choosing  $T\asymp n$ for SGD in the convex and smooth case. To achieve a similar rate, our results indicate that SCGD and SCSC need more iterations to do that.  The reason may be due to the usage of   the moving-average sequence $y_{t+1}$ to track $g_S(x_t)$ and the (empirical) variance term for the inner function $g_\omega$ as mentioned in Remark \ref{rem:stability}. 
\end{remark}

\begin{remark}
    Note that in Theorem \ref{thm:cov_sta} we present the stability result of the last iterate \(A(S)= x_T\). While in Theorem \ref{thm:gen_convex} we present the generalization bound of \(A(S)= \frac{1}{T}\sum_{t= 1}^T x_t\), which is the average of the intermediate iterates \(x_1, \ldots, x_T\). This stems from the fact that generalization is a combination of stability and optimization, and the main focus of optimization is the average of intermediate iterates in the convex setting (see e.g. \citep{wang2017stochastic}).
\end{remark}

\subsection{Strongly Convex Setting}
\textbf{Stability Results.} The following theorem  establishes  the \textit{compositional uniform Stability} (See Definition \ref{def:stability}) for SCGD and SCSC in the strongly convex setting.   

\begin{theorem}[Stability, Strongly Convex] \label{thm:stab_sconvex}
Suppose that Assumption \ref{assum:1} and \ref{assum:2} hold true and  $f_{\nu}(g_S(\cdot))$ is $\sigma$-strongly convex. Consider Algorithm \ref{alg:1} with  $\eta_t = \eta \leq1/\bigl(2L+2\sigma\bigr)$ and $\beta_t = \beta \in (0,1)$ for $t\in [0,T-1]$ and the output  $A(S)  =x_{T}$. Then,  
SCGD and SCSC are compositional uniform stable with 
\vspace{-1mm}
\begin{align*}\label{thm:str_conv_sta}
&\epsilon_{\nu}+\epsilon_{\omega}\\&= \mathcal{O}\Big(\frac{L_gL_f(L+\sigma)}{\sigma L m}+\frac{L_gL_f(L+\sigma)}{\sigma Ln}+\frac{L_f\sqrt{C_g(L+\sigma)\eta}}{\sqrt{\sigma L}}\\
&+C_fL_g\eta\sup_{S}\bigl\{ \sum_{j=0}^{T-1}(1-\eta\frac{L\sigma}{L+\sigma})^{T-j-1}\\ &\quad \qquad\qquad \qquad \times\bigl(\mathbb{E}_{A}[\| y_{j+1}-g_S(x_j)\|^2 ] \bigr)^{1\over 2}\Bigr) \bigr\}. \numberthis
\end{align*}\vspace{-6mm}
\end{theorem}
The proof for Theorem \ref{thm:stab_sconvex} is given in Appendix \ref{pf_strcon_stab}. 
\begin{remark}
  The stability for SGD with $\sigma$-strongly convex and smooth losses is of the order $\mathcal{O}(\frac{1}{\sigma n})$ which was established in \cite{hardt2016train}. Comparing the result  of SGD with our SCGD and SCSC, we have two extra terms if $n=m$, i.e., $\eta\sup_{S}\sum_{j=0}^{T-1}(1-\eta\frac{L\sigma}{L+\sigma})^{T-j-1}\bigl(\mathbb{E}_{A}[\| y_{j+1}-g_S(x_j)\|^2 ] \bigr)^{1\over 2}$ and $\frac{L_f\sqrt{C_g(L+\sigma)\eta}}{\sqrt{\sigma L}}$, where $C_g$ is the (empirical)  variance of the gradient of inner function, i.e.$\sup_{x\in \X}{1\over m}\sum_{j=1}^m\|\nabla g_{\omega_{j}}(x)-\nabla g_S(x)\|^2\leq C_g$ . We can see that if $g_{\omega}(x)=x$, then $g_S(x)=g_{\omega}(x)$ for any $\omega$ and $S$. In this case, $\mathbb{E}_{A}[\| y_{j+1}-g_S(x_j)\|^2 ]$ and $C_g$ will be zeros. Therefore, our stability results in Theorem \ref{thm:stab_sconvex} match that of SGD in the non-compositional setting \cite{hardt2016train}.

\end{remark}

Combining Theorem \ref{thm:stab_sconvex} with the estimation for  $\mathbb{E}_{A}[\| y_{j+1}-g_S(x_j)\|^2]$  in Lemma \ref{Mengdi-lemma} and using the Lemma \ref{lem:weighted_avg} which is given in Appendix \ref{tech_lemma},  we can derive the explicit  stability bounds in the following corollary. Its detailed proof is given at the end of Section  \ref{pf_strcon_stab} in the appendix.
\begin{corollary}\label{cor:2}Let Assumption \ref{assum:1} and \ref{assum:2} hold true and $f_{\nu}(g_S(\cdot))$ be $\sigma$-strongly convex. Consider Algorithm \ref{alg:1} with  $\eta_t = \eta \leq1/\bigl(2L+2\sigma\bigr)$ and $\beta_t = \beta \in (0,1)$ for $t\in [0,T-1]$ and the output  $A(S)  =x_{T}$. Let $c$ be an arbitrary constant. Then, we have the following results: \vspace{-3mm}
\begin{itemize}[leftmargin=3mm]
 \setlength\itemsep{0mm}
\item SCGD is compositional uniformly stable with\vspace{-2mm}
        \begin{equation*}
            \epsilon_\nu + \epsilon_\omega= \mathcal{O}\bigl(n^{-1}+m^{-1}+\eta^{\frac{1}{2}}+\eta\beta^{-1}+\beta^{\frac{1}{2}}+T^{-\frac{c}{2}}\beta^{-\frac{c}{2}}\bigr).
        \end{equation*}\vspace{-7mm}
\item SCSC is compositional uniformly stable with \vspace{-2mm}
        \begin{equation*}
            \epsilon_\nu + \epsilon_\omega=\mathcal{O}\bigl(n^{-1}+m^{-1}+\eta^{\frac{1}{2}}+\eta\beta^{-\frac{1}{2}}+\beta^{\frac{1}{2}}+T^{-\frac{c}{2}}\beta^{-\frac{c}{2}}\bigr).
        \end{equation*}
\end{itemize}\vspace{-3mm}
\end{corollary}
\textbf{Generalization results.} Using the error decomposition  \eqref{eq:error-decomp},  Corollary \ref{cor:2} and Theorem \ref{thm:stab_sconvex}, we can derive the excess risk rates. To this end, we need the following results to estimate the optimization error, i.e., $F_S(A(S))- F_S(x_*^S).$


\begin{theorem}[Optimization, Strongly Convex] \label{thm:opt_sconvex}
  Suppose Assumption \ref{assum:1} and \ref{assum:2} \ref{part1}, \ref{part2} hold for the empirical risk $F_S$, and $F_S$ is $\sigma$-strongly convex, and \(\eta, T\) is chosen such that \((\eta (T- 1))^{-1}\leq \frac{\sigma}{2}\). Let \(A(S)= \left(\sum_{t= 1}^T (1- \sigma\eta/2)^{T- t} x_t\right) / \left(\sum_{t= 1}^T (1- \sigma\eta/2)^{T- t}\right)\) be the solution produced by Algorithm \ref{alg:1} with SCGD or SCSC update and $\eta_t= \eta$ and $\beta_t= \beta$ for some $a,b\in \left(0,1\right]$.\vspace{-3mm}
  \begin{itemize}[leftmargin=*]
    \item For SCGD update,  there holds\vspace{-2mm}
        \begin{align*}
        \mathbb{E}_A[F_S(A(S))&- F_S(x_*^S)] \\
        = \mathcal{O}\bigl( D_x(\eta T)^{-c}&+ L_f^2L_g^2\eta+ \frac{C_f^2L_g^2D_y}{\sigma}(\beta T)^{-c}\\
        &+ \frac{C_f^2L_g^2V_g}{\sigma}\beta+ \frac{C_f^2L_f^2L_g^5}{\sigma}\eta^2\beta^{-2}\bigr).
        \end{align*}\vspace{-8mm}
    \item For SCSC update,  there holds\vspace{-2mm}
        \begin{align*}
        \mathbb{E}_A[F_S(A(S))&- F_S(x_*^S)] \\
        = \mathcal{O}\bigl( D_x(\eta T)^{-c}&+ L_f^2L_g^2\eta+ \frac{C_f^2L_g^2D_y}{\sigma}(\beta T)^{-c}\\
        &+ \frac{C_f^2L_g^2V_g}{\sigma}\beta+ \frac{C_f^2L_f^2L_g^5}{\sigma}\eta^2\beta^{-1}\bigr).
        \end{align*}
  \end{itemize}
\end{theorem}


\begin{theorem}[Excess Risk Bound, Strongly Convex] \label{thm:gen_sconvex}
  Suppose Assumption \ref{assum:1} and \ref{assum:2} hold true,  $f_{\nu}(g_S(\cdot))$ is $\sigma$-strongly convex, and \(\eta, T\) is chosen such that \((\eta (T- 1))^{-1}\leq \frac{\sigma}{2}\). Denote \(D_x:= \mathbb{E}_A [F_S(x_0)- F_S(x_*^S)]\) and \(D_y:= \mathbb{E}_A[\|y_1- g_S(x_0)\|^2]\) . Let \(A(S)= \left(\sum_{t= 1}^T (1- \sigma\eta/2)^{T- t} x_t\right) / \left(\sum_{t= 1}^T (1- \sigma\eta/2)^{T- t}\right)\) be a solution produced by Algorithm \ref{alg:1} with SCGD or SCSC update and $\eta=T^{-a}$ and $\beta=T^{-b}$ for some $a,b\in \left(0,1\right]$.  \vspace{-3mm}
\begin{itemize}[leftmargin=3mm]
\setlength\itemsep{0.5mm}
\item     If we select $T \asymp \max(n^{\frac{5}{3}},m^{\frac{5}{3}})$, $\eta=T^{-\frac{9}{10}}$ and $\beta=T^{-\frac{3}{5}}$, then, for the SCGD update, we have that  $ 
      \mathbb{E}_{S,A}\big[F(A(S)) - F(x_*) \big]= \mathcal{O}\bigl(\frac{1}{\sqrt{n}}+\frac{1}{\sqrt{m}}\bigr).$ 
\item If we select $T \asymp \max(n^{\frac{7}{6}},m^{\frac{7}{6}})$, $\eta= \beta=T^{-\frac{6}{7}}$, then, for the SCSC update, there holds $
      \mathbb{E}_{S,A}\Big[F(A(S)) - F(x_*) \Big]= \mathcal{O}(\frac{1}{\sqrt{n}}+\frac{1}{\sqrt{m}}).$
      \end{itemize}
\end{theorem}

\begin{remark}
    Theorem \ref{thm:gen_sconvex} shows that the generalization error for SCGD can be achieved the rate $\mathcal{O}(1/\sqrt{n}+1/\sqrt{m})$ in the strongly convex case after carefully selecting the iteration number $T$ and constant stepsize $\eta$ and $\beta$. It is worthy of noting that, for achieving the rate $\mathcal{O}(1/\sqrt{n}+1/\sqrt{m})$,  SCGD needs   iteration $T\asymp \max(n^{5/3},m^{5/3})$ in the strongly convex case while Theorem \ref{thm:gen_convex} shows that it needs more iterations, i.e.,   $T\asymp \max(n^{3.5},m^{3.5})$ in the convex case.   SCSC fruther improves the results as it only  needs iteration $T\asymp \max(n^{7/6},m^{7/6})$ in the strongly convex case.
\end{remark}

\section{Conclusion}\label{sec:conclusion}
In this paper, we conduct a comprehensive study on the stability and generalization analysis of stochastic compositional optimization (SCO) algorithms. We introduce the concept of compositional uniform stability to handle the function composition structure inherent in SCO problems. By establishing the connection between stability and generalization error, we provide stability bounds for two popular SCO algorithms: SCGD and SCSC. In the convex case with standard smooth assumptions, we demonstrate that both SCGD and SCSC achieve an excess generalization error rate of $\mathcal{O}(1/\sqrt{n}+1/\sqrt{m})$, with SCSC requiring fewer iterations than SCGD. Furthermore, we extend our analysis to the strongly convex case, where we show that SCGD and SCSC achieve the same rate of $\mathcal{O}(1/\sqrt{n}+1/\sqrt{m})$ with even fewer iterations than in the convex case. 

There are several directions for future research. Firstly, while our analysis only considers the convex and smooth cases, an interesting avenue for future research is to consider the case where the inner function and/or outer function are non-smooth and non-convex, e.g., neural networks with   Rectified Linear Unit (ReLU) activation function. Secondly, it would be interesting to get optimal  excess risk rates $\mathcal{O}(1/\sqrt{n}+1/\sqrt{m})$ with linear time complexity $T = \mathcal{O}\bigl(\max(n, m)\bigr)$ for SCGD and SCSC. 




%% file: appendix_lemma.tex
\section{Technical Lemmas}\label{tech_lemma}
 
\begin{table}[!htbp]
	\centering
 \caption{Notations}
 \label{notations}
\begin{tabular}{|c|c|c|}
	\hline
	  notations&meaning&mathematical language\\
	\hline
     $L_f$&$L_f$-Lipschitz continuous of $f_\nu(\cdot)$&$\sup_\nu\|f_\nu(y)-f_\nu(\hat{y})\|\leq L_f\|y-\hat{y}\|, \forall y,\hat{y}\in \mathbb{R^d}$\\
     \hline
     $C_f$&$C_f$-Lipschitz continuous  of $\nabla f_{\nu}(\cdot)$&$\sup_\nu\|\nabla f_{\nu}(y)- \nabla f_{\nu}(\bar{y})\|\leq C_f\|y- \bar{y}\|, \forall y,\bar{y}\in \mathbb{R^d}$\\
     \hline
     $L_g$&$L_g$-Lipschitz continuous of $g_\omega(\cdot)$&$\sup_\omega\left\|g_\omega\left(x\right)-g_\omega\left(\hat{x}\right)\right\|\leq L_g\left\|x-\hat{x}\right\|, \forall x,\hat{x}\in \mathcal{X}$\\
     \hline
      $V_g$&the empirical variance of the $g(\cdot)$& $\sup_{x\in \X}{1\over m}\sum_{j=1}^m\|g_{\omega_{j}}(x)-g_S(x)\|^2\leq V_g$\\
      \hline
      $C_g$&the empirical variance of the $\nabla g(\cdot)$& $\sup_{x\in \X}{1\over m}\sum_{j=1}^m\|\nabla g_{\omega_{j}}(x)-\nabla g_S(x)\|^2\leq C_g$\\
      \hline
       $L$&$L$ smooth of $f_\nu(g_S(\cdot))$&$
        \left\|g_S(u)\nabla f_\nu(g_S(u))-g_S(v)\nabla f_\nu(g_S(v))\right\|\leq L \left\|u-v\right\|$\\
       \hline
        $\epsilon_{\nu}, \epsilon_{\omega}$&$\left(\epsilon_{\nu}, \epsilon_{\omega}\right)$-uniform stability&\\
       \hline
       $n,m$&$n,m$: the numbers of  $S_\nu$ and $S_\omega$, respectively &\\
       \hline
\end{tabular}
\end{table}
First, we list some signal notations in Table \ref{notations} for our paper setting. To derive the stability and generalization bounds, we give the following lemmas.  

The following lemma is directly adapted from \cite{wang2017stochastic,chen2021solving} where both the population distribution for the random variables $\nu$ and $\omega$ are the uniform distributions over $S_\nu = \{\nu_1,\ldots, \nu_n\}$ and $S_\omega = \{\omega_1, \ldots,\omega_m\}.$  It  states that $y_{t+1}$ behaves similarly to $g_S(x_t) $

\begin{lemma}\label{Mengdi-lemma}
    Let Assumption \ref{assum:1} and \ref{assum:2} \ref{part1} hold and ${(x_t, y_t)}$ be generated by Algorithm \ref{alg:1}. Let $\eta_t= \eta$, and $\beta_t= \beta$ for \(\eta, \beta> 0\). Let \(c> 0\) be an arbitrary constant. 
  \begin{itemize}[leftmargin=*]
    \item With SCGD update, we have
      \begin{align*}
       &\mathbb{E}_A\bigl[\|y_{t+1}-g_S\big(x_t\big)\|^2\bigr] 
    \leq \left(\frac{c}{e}\right)^c (t\beta)^{-c}\mathbb{E}_A[\| y_1- g_S(x_0)\|^2]+  L_f^2L_g^3\frac{\eta^2}{\beta^2}+2V_g\beta.
      \end{align*}
    \item With SCSC update we have
      \begin{align*}
       &\mathbb{E}_A\bigl[\|y_{t+1}-g_S\big(x_t\big)\|^2\bigr] 
    \leq\left(\frac{c}{e}\right)^c (t\beta)^{-c}\mathbb{E}_A\bigl[\|y_{1}-g_S\big(x_0\big)\|^2\bigr]+L_f^2L_g^3\frac{\eta^2}{\beta}+ 2V_g\beta.
      \end{align*}
  \end{itemize}
\end{lemma}

The next lemma was established in \cite{schmidt2011convergence} and this lemma was used in \cite{wang2022stability}.
\begin{lemma} \label{recursion lemma}
    Assume that the non-negative sequence ${u_t: t\in\mathbb{N} }$ satisfies the following recursive inequality for all $t \in \mathbb{N}$,
    \begin{align*}
        u_t^2 \leq S_t+\sum_{\tau=1}^{t-1} \alpha_\tau u_\tau.
    \end{align*}
    where $\{S_\tau: \tau \in \mathbb{N}\}$ is an increasing sequence, $S_0 \geq u_0^2$ and $\alpha_\tau$ for any $\tau \in \mathbb{N}. $ Then, the following inequality holds true:
    \begin{align*}
        u_t \leq \sqrt{S_t}+\sum_{\tau=1}^{t-1} \alpha_\tau.
    \end{align*}
\end{lemma}

\begin{lemma} \label{lem:sum_prod}
  For any $\nu, c>0$, we have
    \begin{equation}
      e^{-\nu x} \le  \bigl(\frac{c}{\nu e} \bigr)^{c} x^{-c}
    \end{equation}
\end{lemma}

\begin{lemma}   \label{lem:weighted_avg}
  Let \(\{a_i\}_{i= 1}^T, \{b_i\}_{i= 1}^T\) be two sequences of positive real numbers such that \(a_i\leq a_{i+ 1}\) and \(b_i\geq b_{i+ 1}\) for all \(i\). Then we have
  \begin{equation}  \label{eq:opt_sconvex:15}
      \frac{\sum_{i= 1}^T a_ib_i}{\sum_{i= 1}^T a_i}\leq \frac{\sum_{i= 1}^T b_i}{T}.
  \end{equation}
\end{lemma}
\begin{proof}
  To show \eqref{eq:opt_sconvex:15}, it suffices to show
  \begin{equation*}
    \sum_{i= 1}^T a_ib_i \sum_{j= 1}^T 1\leq \sum_{j= 1}^T a_j\sum_{i= 1}^T b_i.
  \end{equation*}
  Rearranging the summation, it suffices to show
  \begin{equation*}
    \sum_{i= 1}^T \sum_{j= 1}^T a_ib_i- \sum_{i= 1}^T \sum_{j= 1}^T a_j b_i\leq 0.
  \end{equation*}
  The above inequality can be rewritten as
  \begin{equation*}
      0\geq \sum_{i= 1}^T \sum_{j= 1}^T (a_i- a_j)b_i= \sum_{i= 1}^T \sum_{j= i+ 1}^T (a_i- a_j)(b_i- b_j),
  \end{equation*}
  where the last equality holds due to the symmetry between \(i\) and \(j\). Since for \(i< j\) we have \(a_i\leq a_j\) and \(b_i\geq b_j\), we know the above inequality holds, and thus \eqref{eq:opt_sconvex:15} holds. Then we complete the proof.
\end{proof}

\subsection{Proof of Lemma \ref{Mengdi-lemma}}

The proof of Lemma \ref{Mengdi-lemma} leverages the following results.

\begin{lemma}[Lemma 2 in \cite{wang2017stochastic}]
  Suppose Assumption \ref{assum:1} \ref{assum:1b} and \ref{assum:2} \ref{part1} hold for the empirical risk $F_S$. By running Algorithm \ref{alg:1} with SCGD update, we have
  \begin{equation}  \label{eq:opt_convex:4}
    \mathbb{E}_A[\| y_{t+ 1}- g_S(x_t)\|^2| \mathcal{F}_t]\leq (1- \beta_t)\| y_t- g_S(x_{t- 1})\|^2+ \frac{L_g^2}{\beta_t}\| x_t- x_{t- 1}\|^2+ 2V_g\beta_t^2
  \end{equation}
\end{lemma}
\begin{lemma}[Lemma 1 in \cite{chen2021solving}]
  Suppose Assumption \ref{assum:1} \ref{assum:1b} and \ref{assum:2} \ref{part1} hold for the empirical risk $F_S$. By running Algorithm \ref{alg:1} with SCSC update, we have
  \begin{equation}  \label{eq:opt_convex:5}
    \mathbb{E}_A[\| y_{t+ 1}- g_S(x_t)\|^2| \mathcal{F}_t]\leq (1- \beta_t)\| y_t- g_S(x_{t- 1})\|^2+ L_g^2\| x_t- x_{t- 1}\|^2+ 2V_g\beta_t^2
  \end{equation}
\end{lemma}

Now we are ready to prove Lemma \ref{Mengdi-lemma}.
\begin{proof}[Proof of Lemma \ref{Mengdi-lemma}]
  We first present the proof for the SCGD update. Taking the expectation with respect to the internal randomness of the algorithm over \eqref{eq:opt_convex:4} and noting that \(\mathbb{E}_A[\|x_t- x_{t- 1}\|^2]\leq L_f^2L_g^2\eta_{t- 1}^2\), we get
  \begin{equation*}
    \mathbb{E}_A[\| y_{t+ 1}- g_S(x_t)\|^2]\leq (1- \beta_t)\EX_{A}[\| y_t- g_S(x_{t- 1})\|^2]+ \frac{L_f^2L_g^3\eta_{t- 1}^2}{\beta_t}+ 2V_g\beta_t^2.
  \end{equation*}
  Telescoping the above inequality from \(1\) to \(t\) yields
  \begin{align*}
    &\mathbb{E}_A[\| y_{t+ 1}- g_S(x_t)\|^2] \\
    \leq& \prod_{i= 1}^{t}(1- \beta_i)\mathbb{E}_A[\| y_1- g_S(x_0)\|^2]+ L_f^2L_g^3 \sum_{i= 1}^{t}\prod_{j= i+ 1}^{t} (1- \beta_j)\frac{\eta_{i- 1}^2}{\beta_i}+ 2V_g \sum_{i= 1}^{t}\prod_{j= i+ 1}^{t} (1- \beta_j)\beta_i^2.
  \end{align*}
  Note that \(\prod_{i= K}^N (1- \beta_i)\leq \exp(- \sum_{i= K}^{N} \beta_i)\) for all \(K\leq N\) and \(\beta_i> 0\), then setting \(\eta_t= \eta, \beta_t= \beta\), thus we have
  \begin{align*}
    &\mathbb{E}_A[\| y_{t+ 1}- g_S(x_t)\|^2] 
    \leq \exp(-\beta t)\mathbb{E}_A[\| y_1- g_S(x_0)\|^2]+ \sum_{i= 1}^{t} (1-\beta)^{t-i}(L_g^3L_f^2\frac{\eta^2}{\beta}+2V_g\beta^2).
  \end{align*}
  Using  Lemma \ref{lem:sum_prod} with \(\nu= 1\), we get
  \begin{align*}
    \mathbb{E}_A[\| y_{t+ 1}- g_S(x_t)\|^2] 
    \leq \left(\frac{c}{e}\right)^c (t\beta)^{-c}\mathbb{E}_A[\| y_1- g_S(x_0)\|^2]+  L_g^3L_f^2\frac{\eta^2}{\beta^2}+2V_g\beta,
  \end{align*}
  where the inequality holds for $\sum_{i= 1}^{t} (1-\beta)^{t-i}\leq\frac{1}{\beta}$. Then we get the desired result for the SCGD update. Next we present the proof for the SCSC update. Taking the total expectation with respect to the internal randomness of the algorithm over \eqref{eq:opt_convex:5} and noting that \(\mathbb{E}_A[\|x_t- x_{t- 1}\|^2]\leq L_fL_g\eta_{t- 1}^2\), we get
  \begin{equation*}
    \mathbb{E}_A[\| y_{t+ 1}- g_S(x_t)\|^2]\leq (1- \beta_t)\mathbb{E}_A[\| y_t- g_S(x_{t- 1})\|^2]+ L_f^2L_g^3\eta_{t- 1}^2+ 2V_g\beta_t^2.
  \end{equation*}
  Telescoping the above inequality from \(1\) to \(t\) yields
  \begin{align*}
    &\mathbb{E}_A[\| y_{t+ 1}- g_S(x_t)\|^2] \\
    \leq& \prod_{i= 1}^{t}(1- \beta_i)\mathbb{E}_A[\| y_1- g_S(x_0)\|^2]+ L_f^2L_g^3 \sum_{i= 1}^{t}\prod_{j= i+ 1}^{t} (1- \beta_j)\eta_{i- 1}^2+ 2V_g \sum_{i= 1}^{t}\prod_{j= i+ 1}^{t} (1- \beta_j)\beta_i^2.
  \end{align*}
  Note that \(\prod_{i= K}^N (1- \beta_i)\leq \exp(- \sum_{i= K}^{N} \beta_i)\) for all \(K\leq N\) and \(\beta_i> 0\), then setting \(\eta_t= \eta, \beta_t= \beta\), thus we have
  \begin{align*}
    &\mathbb{E}_A[\| y_{t+ 1}- g_S(x_t)\|^2] \\
    \leq& \exp(-t \beta)\mathbb{E}_A[\| y_1- g_S(x_0)\|^2]+ \sum_{i= 1}^{t} (1-\beta)^{t-i}(L_g^3L_f^2\eta^2+2V_g\beta^2).
  \end{align*}
   Using  Lemma \ref{lem:sum_prod} with \(\nu= 1\), we get
  \begin{align*}
    \mathbb{E}_A[\| y_{t+ 1}- g_S(x_t)\|^2] 
    \leq \left(\frac{c}{e}\right)^c (t\beta)^{-c}\mathbb{E}_A[\| y_1- g_S(x_0)\|^2]+  L_g^3L_f^2\frac{\eta^2}{\beta}+2V_g\beta,
  \end{align*}
  where the inequality holds for $\sum_{i= 1}^{t} (1-\beta)^{t-i}\leq\frac{1}{\beta}$. 
 Then we get the desired result for the SCSC update. Then we complete the proof.
\end{proof}

%% file: appendix_problem_formulation.tex
\section{Proof for Section \ref{sec:problem}}

\begin{proof}[Proof of Theorem \ref{thm:1}] Write  
    \begin{align*} & \EX_{S,A}\Big[  F(A(S))  - F_S(A(S)) \Big] =  \EX_{S,A}\Big[  \EX_\nu [f_\nu \bigl( g(\bx) \bigr)] - \frac{1}{n}\sum_{i=1}^n f_{\nu_i}\bigl( \frac{1}{m} \sum_{j=1}^m g_{\omega_j}(\bx) \bigr)\Big]\\
    & = \EX_{S,A}\Big[ \EX_\nu [f_\nu \bigl( g(A(S))\bigr)]  - \frac{1}{n}\sum_{i=1}^n f_{\nu_i}\bigl( g(A(S)) \bigr) \Big] \\ & +  \EX_{S,A}\Big[ \frac{1}{n}\sum_{i=1}^n f_{\nu_i}\bigl( g(A(S)) \bigr) - \frac{1}{n}\sum_{i=1}^n f_{\nu_i}\bigl( \frac{1}{m} \sum_{j=1}^m g_{\omega_j}(A(S)) \bigr)\Big] \\
    & \le \EX_{S,A}\Big[ \EX_\nu [f_\nu \bigl( g(A(S))\bigr)]  - \frac{1}{n}\sum_{i=1}^n f_{\nu_i}\bigl( g(A(S)) \bigr) \Big] \\ & +  \EX_{S,A}\Big[ \frac{1}{n}\sum_{i=1}^n \Big( f_{\nu_i}\bigl( g(A(S)) \bigr) - f_{\nu_i}\bigl( \frac{1}{m} \sum_{j=1}^m g_{\omega_j}(A(S)) \bigr) \Big)\Big].  
     \numberthis \label{eq:inter1} 
    \end{align*}
    Now we estimate the two terms on the right-hand side of \eqref{eq:inter1}. Define $S^{ \prime,\nu}=\{\nu_1^{\prime},\nu_2^{\prime},...,\nu_n^{\prime},\omega_1,\omega_2,...,\omega_m\}$. In particular, we have that 
    \begin{align*} 
    & \EX_{S,A}\big[ \EX_\nu [f_\nu ( g(A(S)))]  - \frac{1}{n}\sum_{i=1}^n f_{\nu_i}( g(A(S)) ) \big] \\ & = 
    \EX_{S,A, S^{\prime,\nu}}\big[ \frac{1}{n}\sum_{i=1}^n f_{\nu_i}( g(A(S^{i,\nu})))  - \frac{1}{n}\sum_{i=1}^n f_{\nu_i}( g(A(S)) ) \big]  \\ 
    & = \EX_{S,A, S^{\prime,\nu}}\big[ \frac{1}{n}\sum_{i=1}^n \big( f_{\nu_i}( g(A(S^{i,\nu}))  -  f_{\nu_i}( g(A(S)) \big) \big] \\
    & \le L_f \|g(A(S^{i,\nu})) - g(A(S))  \| \le L_f L_g \|A(S^{i,\nu}) -A(S) \|. \numberthis \label{eq:fnu}
    \end{align*}
    Furthermore, 
    \begin{align*}
    & \EX_{S,A}\big[ \frac{1}{n}\sum_{i=1}^n \big( f_{\nu_i}\bigl( g(A(S)) \bigr) - f_{\nu_i}\bigl( \frac{1}{m} \sum_{j=1}^m g_{\omega_j}(A(S)) \bigr) \big)\big] \\
     & \le L_f \EX_{S,A}\big[\big\|g(A(S))  - \frac{1}{m} \sum_{j=1}^m g_{\omega_j}(A(S)) \big \|\big]. \numberthis\label{secdterm}
     \end{align*}
    Now it is sufficient to estimate the term $\EX_{S,A}\big[\big\|g(A(S))  - \frac{1}{m} \sum_{j=1}^m g_{\omega_j}(A(S)) \big \|\big].$
    Note that, in general, $g$ is a mapping from $\R^p$ to $\R^d$. To this end, we will use some ideas from \cite{bousquet2020sharper}. To this end, we write
    \begin{align*}
    &g(A(S))-\frac{1}{m} \sum_{j=1}^m g_{\omega_j}(A(S))\\
    &= \frac{1}{m} \sum_{j=1}^m \mathbb{E}_{\omega,\omega_{j}^{\prime}}\left[g_{\omega}(A(S))-g_{\omega}\left(A\left(S^{j,\omega}\right)\right)\right]
    +\frac{1}{m}\sum_{j=1}^m \mathbb{E}_{\omega_j^{\prime}}[\mathbb{E}_{\omega}\left[g_{\omega}\left(A\left(S^{j,\omega}\right)\right)\right]-g_{\omega_j}\left(A\left(S^{j,\omega}\right)\right]\\
    &+\frac{1}{m} \sum_{j=1}^m \mathbb{E}_{\omega_{j}^{\prime}}\left[g_{\omega_{j} }\left(A\left(S^{j,\omega}\right)\right)- g_{\omega_j}(A(S))\right].
    \end{align*}
    It then follows that:
    \begin{align*}
    &\|g(A(S))-\frac{1}{m} \sum_{j=1}^m g_{\omega_j}(A(S))\|\leq \frac{1}{m} \sum_{j=1}^m \mathbb{E}_{\omega,\omega_{j}^{\prime}} \|g_{\omega}(A(S))-g_{\omega}(A(S^{j,\omega}))\|
    \\ & +\frac{1}{m}\|\sum_{j=1}^m \mathbb{E}_{\omega_j^{\prime}}[\mathbb{E}_{\omega}[g_{\omega}(A(S^{j,\omega}))]-g_{\omega_j}(A(S^{j,\omega}))]\|+\frac{1}{m} \sum_{j=1}^m \mathbb{E}_{\omega_{j}^{\prime}}\|g_{\omega_{j} }\left(A\left(S^{j,\omega}\right)\right)- g_{\omega_j}(A(S))\|.
    \end{align*}
    Note $S$ and $S^{j,\omega}$ differ by a single example. By the assumption on stability and Definition \ref{def:stability}, we further get 
        \begin{align*}
        &\mathbb{E}_{S, A}[\|g(A(S))-\frac{1}{m} \sum_{j=1}^m g_{\omega_j}(A(S))\|]\\
        &\leq \mathbb{E}_{S,A}[\frac{1}{m}\|\sum_{j=1}^m \mathbb{E}_{\omega_j^{\prime}}[\mathbb{E}_{\omega}[g_{\omega}(A(S^{j,\omega}))]-g_{\omega_j}(A(S^{j,\omega}))]\|]+2L_g\epsilon_{\omega}.
        \numberthis \label{eq:inter2} 
        \end{align*}
    Next step, we need to estimate $\|\sum_{j=1}^m \mathbb{E}_{\omega_j^{\prime}}[\mathbb{E}_{\omega}\left[g_{\omega}\left(A\left(S^{j,\omega}\right)\right)\right]-g_{\omega_j}\left(A\left(S^{j,\omega}\right)\right)]\|.$\\
    Using a similar proof technique in paper \cite{lei2022nonsmooth}, we can set $\xi_j(S)$ as a function of S as follows
        \begin{align*}
        \xi_j(S) =	\mathbb{E}_{\omega_j^{\prime}}[\mathbb{E}_{\omega}\left[g_{\omega}\left(A\left(S^{j,\omega}\right)\right)\right]-g_{\omega_j}\left(A\left(S^{j,\omega}\right)\right)].
        \end{align*}
    Notice that: 	
    \begin{align*}
    \mathbb{E}_{S, A}[\|\sum_{j=1}^m \xi_j(S)\|^2]= \mathbb{E}_{S,A}[\sum_{j=1}^m\|\xi_j(S)\|^2]+\sum_{j, i \in[m]: j \neq i} \mathbb{E}_{S,A}[\langle\xi_j(S), \xi_i(S)\rangle].
        \numberthis \label{eq:inter3} 
    \end{align*}
    According to the definition of $\xi_j(S)$ and the Cauchy-Schwarz inequality, we know
    \begin{align*}
& \mathbb{E}_{S,A}[\sum_{j=1}^m\|\xi_j(S)\|^2]
    =\sum_{j=1}^m \mathbb{E}_{S,A}[\| \mathbb{E}_{\omega_j^{\prime}}[\mathbb{E}_{\omega}[g_{\omega}(A(S^{j,\omega}))]-g_{\omega_j}(A(S^{j,\omega}))] \|^2]\\
    &\leq \sum_{j=1}^m \mathbb{E}_{S,A}[\| \mathbb{E}_{\omega}[g_{\omega}(A(S^{j,\omega}))]-g_{\omega_j}(A(S^{j,\omega})) \|^2] \\& =\sum_{j=1}^m \mathbb{E}_{S,A}[\| \mathbb{E}_{\omega}[g_{\omega}(A(S))]-g_{\omega_{j}^{\prime}}(A(S)) \|^2]=m\mathbb{E}_{S,A}\left[\text{Var}_\omega(g_\omega(A(S)))\right],\numberthis \label{eq:inter4} 
    \end{align*}
    where the variance term $\text{Var}_\omega(g_\omega(A(S)))=\mathbb{E}_\omega\left[\left\|g(A(S))-g_\omega(A(S))\right\|^2\right].$\\
    Next, we will estimate the second term on the right-hand side of \eqref{eq:inter3}. 
    To this end, we define
    \begin{align*}
    &S^{i,\omega}=\left\{\omega_1, \ldots, \omega_{i-1}, \omega_i^{\prime},\omega_{i+1}, \ldots, \omega_m, \nu_1,\ldots, \nu_n\right\};\\
    &S^{i,j,\omega}=\left\{\omega_1, \ldots ,\omega_{i-1}, \omega_i^{\prime},\omega_{i+1}, \ldots ,\omega_{j-1}, \omega_j^{\prime}, \omega_{j+1}, \ldots, \omega_m,\nu_1,\ldots, \nu_n\right\}.
    \end{align*}
    Due to the symmetry between $\omega$ and $\omega_{j}$, we can have
    \begin{align*}
    \mathbb{E}_{w_j}\left[\xi_j(S)\right]=0,\forall j \in[m]\numberthis \label{interse_0}
    \end{align*}
    If $j\neq i$,we have
        \begin{align*}
            \mathbb{E}_{S,A}\left[\left\langle\xi_j\left(S^{i,\omega}\right), \xi_i(S)\right\rangle\right]
            &=\mathbb{E}_{S,A}\mathbb{E}_{\omega_i}\left[\left\langle\xi_j\left(S^{i,\omega}\right), \xi_i(S)\right\rangle\right]\\
            &=\mathbb{E}_{S, A}\left[\left\langle\xi_j\left(S^{i,\omega}\right), \mathbb{E}_{\omega_i}\left[\xi_i(S)\right]\right\rangle\right]=0,
        \end{align*}
    where the second equality holds since the $\xi_j\left(S^{i,\omega}\right)$ is independent of $\omega_{i}$ and the last identity follows from $\mathbb{E}_{w_i}\left[\xi_i(S)\right]=0$ due to \eqref{interse_0} . In  a similar way, we can get the following equations for $j\neq i$
         \begin{align*}
                   \mathbb{E}_{S, A}\left[\left\langle\xi_j(S), \xi_i\left(S^{j,\omega}\right)\right\rangle\right]
             &=\mathbb{E}_{S,A}\mathbb{E}_{\omega_j}\left[\left\langle\xi_j(S), \xi_i\left(S^{j,\omega}\right)\right\rangle\right]\\
             &=\mathbb{E}_{S, A}\left[\left\langle \mathbb{E}_{\omega_j}\left[\xi_j(S)\right], \xi_i\left(S^{j,\omega}\right)\right\rangle\right]=0,
         \end{align*}
     and 
          \begin{align*}
                \mathbb{E}_{S, A}\left[\left\langle\xi_j\left(S^{i,\omega}\right), \xi_i\left(S^{j,\omega}\right)\right\rangle\right]
                &=\mathbb{E}_{S,A}\mathbb{E}_{\omega_j}\left[\left\langle\xi_j\left(S^{i,\omega}\right), \xi_i\left(S^{j,\omega}\right)\right\rangle\right]\\
                &=\mathbb{E}_{S, A}\left[\left\langle \mathbb{E}_{\omega_j}\left[\xi_j\left(S^{i,\omega}\right)\right], \xi_i\left(S^{j,\omega}\right)\right\rangle\right]=0.
          \end{align*}
    Combining the above identities, we have $j\neq i$
          \begin{align*}
               & \mathbb{E}_{S,A}\left[\left\langle\xi_j(S), \xi_i(S)\right\rangle\right]=\mathbb{E}_{S,A}\left[\left\langle\xi_j(S)-\xi_j\left(S^{i,\omega}\right), \xi_i(S)-\xi_i\left(S^{j,\omega}\right)\right\rangle\right]\\
               &\leq\mathbb{E}_{S , A}\left[\left\|\xi_j(S)-\xi_j\left(S^{i,\omega}\right)\right\| \cdot\left\|\xi_i(S)-\xi_i\left(S^{j,\omega}\right)\right\|\right]\\
               &\leq\frac{1}{2} \mathbb{E}_{S,A}\left[\left\|\xi_j(S)-\xi_j\left(S^{i,\omega}\right)\right\|^2\right]
               +\frac{1}{2} \mathbb{E}_{S, A}\left[\left\|\xi_i(S)-\xi_i\left(S^{j,\omega}\right)\right\|^2 \right],
               \numberthis \label{eq:inter5} 
           \end{align*}
       where the third inequality use $a b \leq \frac{1}{2}\left(a^2+b^2\right)$. With the definition of $\xi_j(S)$, $S^{i,\omega}$ and $S^{i,j,\omega}$, we can have the following identity for $j\neq i$
        \begin{align*}
&\mathbb{E}_{S,A}\left[\left\|\xi_j(S)-\xi_j\left(S^{i,\omega}\right)\right\|^2\right]\\
           &=	\mathbb{E}_{S,A}\left[\|   \mathbb{E}_{\omega_j^{\prime}}[\mathbb{E}_{\omega}\left[g_{\omega}\left(A\left(S^{j,\omega}\right)\right)\right]-g_{\omega_j}\left(A\left(S^{j,\omega}\right)\right)]
-\mathbb{E}_{\omega_j^{\prime}}\left[\mathbb{E}_\omega\left[g_\omega\left(A\left(S^{i,j,\omega}\right)\right)\right]-g_{\omega_ j}\left(A\left(S^{i,j,\omega}\right)\right)\right]\|^2\right]\\
           &= \mathbb{E}_{S, A}\left[\|\mathbb{E}_{\omega_j^{\prime}} \mathbb{E}_\omega\left[g_\omega\left(A\left(S^{j,\omega}\right)\right)-g_\omega\left(A\left(S^{i,j,\omega}\right)\right)\right]+\mathbb{E}_{\omega_j^{\prime}}\left[g_{\omega_j}\left(A\left(S^{i,j,\omega}\right)\right)-g_{\omega_j}\left(A\left(S^{j,\omega}\right)\right)\right]\|^{2}\right].
        \end{align*}
     Then using the elementary inequality $(a+b)^2 \leq 2\left(a^2+b^2\right)$ and the Cauchy-Schwarz inequality, we get
      \begin{align*}
        &\mathbb{E}_{S,A}\left[\left\|\xi_j(S)-\xi_j\left(S^{i,\omega}\right)\right\|^2\right]\\
        &\leq 2\mathbb{E}_{S, A}\left[\|g_\omega\left(A\left(S^{j,\omega}\right)\right)-g_\omega\left(A\left(S^{i,j,\omega}\right)\right)\|^2\right]+2 \mathbb{E}_{S, A}\left[\|g_{\omega_j}\left(A\left(S^{i,j,\omega}\right)\right)-g_{\omega_j}\left(A\left(S^{j,\omega}\right)\right)\|^2\right]\\
        &\leq 2 \mathbb{E}_{S,A}\left[L_g^2\left\|A\left(S^{j,\omega}\right)-A\left(S^{i,j,\omega}\right)\right\|^2\right]+2 \mathbb{E}_{S,A}\left[L_g^2\left\|A\left(S^{i,j,\omega}\right)-A\left(S^{j,\omega}\right)\right\|^2\right].
      \end{align*}
    Since $S^{i,\omega}$ and $S^{i,j,\omega}$ differ by one example, it follows from the definition of stability, we can have
      \begin{align*}
         \mathbb{E}_{S,A}\left[\left\|\xi_j(S)-\xi_j\left(S^{i,\omega}\right)\right\|^2\right]\leq 4 L_g^2 \epsilon_{\omega}^2,   \forall j \neq i.
      \end{align*}
    In a similar way, we can have 
        \begin{align*}
          \mathbb{E}_{S, A}\left[\left\|\xi_i(S)-\xi_i\left(S^{j,\omega}\right)\right\|^2 \right]\leq 4 L_g^2 \epsilon_{\omega}^2,   \forall j \neq i.
        \end{align*}
    Combining above two inequalities into \eqref{eq:inter5}, we get 
        \begin{align*}
           \sum_{j, i \in[m]: j \neq i} \mathbb{E}_{S,A}\left[\left\langle\xi_j(S), \xi_i(S)\right\rangle\right]\leq 4m(m-1) L_g^2 \epsilon_{\omega}^2,   \forall j \neq i.
               \numberthis \label{eq:inter6} 
        \end{align*}
    Then combining the \eqref{eq:inter6} and \eqref{eq:inter4} into \eqref{eq:inter3}, we can have 
        \begin{align*}
         \mathbb{E}_{S, A}\bigl[\|\sum_{j=1}^m \xi_j(S)\|^2\bigr]=m\mathbb{E}_{S,A}\left[\text{Var}_\omega(g_\omega(A(S)))\right]
             +4m (m-1)L_g^2 \epsilon_{\omega}^2.
        \end{align*}
        Then we get
        \begin{align*}
             \mathbb{E}_{S, A}\bigl[\|\sum_{j=1}^m \xi_j(S)\|\bigr]\leq( \mathbb{E}_{S, A}\bigl[\|\sum_{j=1}^m \xi_j(S)\|^2\bigr])^{1/2}\leq \sqrt{m\mathbb{E}_{S,A}\left[\text{Var}_\omega(g_\omega(A(S)))\right]}+2m L_g\epsilon_\omega,
        \end{align*}
     plugging the above inequality back into \eqref{eq:inter2}, we get
        \begin{align*}
             \mathbb{E}_{S, A}\bigl[\|g(A(S))-\frac{1}{m} \sum_{j=1}^m g_{\omega_j}(A(S))\|\bigr]\leq\sqrt{m^{-1}\mathbb{E}_{S,A}\left[\text{Var}_\omega(g_\omega(A(S)))\right]}+4L_g \epsilon_\omega.\numberthis \label{g_gS}
           \end{align*}
       
       Using the result \eqref{g_gS} into \eqref{secdterm} and then combining with the result \eqref{eq:fnu} into \eqref{eq:inter1}, we get final result
       \begin{align*}
           \EX_{S,A}\Big[  F(A(S))  - F_S(A(S)) \Big] \leq L_f L_g\epsilon_\nu+4L_fL_g\epsilon_\omega+L_f\sqrt{m^{-1}\mathbb{E}_{S,A}\left[\text{Var}_\omega(g_\omega(A(S)))\right]}
       \end{align*}
        where $\text{Var}_\omega(g_\omega(A(S)))=\mathbb{E}_\omega\left[\left\|g(A(S))-g_\omega(A(S))\right\|^2\right]$.
\end{proof} 

%% file: appendix_convex.tex
\section{Proof for the Convex Setting}\label{pr_convex}

\subsection{Stability}\label{pr_stab_con}

 \begin{proof}[Proof of Theorem \ref{thm:cov_sta}]
 For any $k\in [n]$, define $S^{k,\nu}=\{\nu_1,...,\nu_{k-1}, \nu_{k}^{\prime},\nu_{k+1},...,\nu_{n},\omega_1,...,\omega_{m}\}$ as formed from $S_\nu$ by replacing the $k$-th element. For any $l\in [m]$, define $S^{l,\omega}=\{\nu_1,...,\nu_n,\omega_1,...,\omega_{l-1}, \omega_{l}^{\prime},\omega_{l+1},...,\omega_{m}\}$ as formed from $S_\omega$ by replacing the $l$-th element.  Let $\{x_{t+1}\}$ and $\{y_{t+1}\}$ be produced by Algorithm \ref{alg:1} based on $S$,  $\{x_{t+1}^{k,\nu}\}$ and $\{y_{t+1}^{k,\nu}\}$ be produced by Algorithm \ref{alg:1} based on $S^{k,\nu}$,  $\{x_{t+1}^{l,\omega}\}$ and $\{y_{t+1}^{l,\omega}\}$ be produced by Algorithm \ref{alg:1} based on $S^{l,\omega}$. Let $x_0=x_0^{k,\nu}$ and $x_0=x_0^{l,\omega}$ be 
 starting points in $\mathcal{X}$. Since changing one sample data can happen in either $S_\nu$ or $S_\omega$,  we   estimate $\mathbb{E}_{A}\bigl[\|x_{t+1}-x_{t+1}^{k,\nu}\|\bigr]$ and $\mathbb{E}_{A}\bigl[\|x_{t+1}-x_{t+1}^{l,\omega}\|\bigr]$ as follows.

\noindent{\bf Estimation of $\mathbb{E}_{A}\bigl[\|x_{t+1}-x_{t+1}^{k,\nu}\|\bigr]$}  
   
We begin with the estimation of the term $\mathbb{E}_{A}\bigl[\|x_{t+1}-x_{t+1}^{k,\nu}\|\bigr]$. For this purpose, we will consider two cases, i.e., $i_t \neq k $ and $i_t = k $. 
    
\textbf{\quad Case 1 ($i_t \neq k $). }      ~~If $i_t \neq k $, we have
\begin{align}
    & \|x_{t+1}-x_{t+1}^{k,\nu}\|^2 \leq \| x_t-\eta_t \nabla g_{\omega_{j_t}}\left(x_t\right) \nabla f_{\nu_{i_ t}}\left(y_{t+1}\right)-x_t^{k,\nu}+\eta_t \nabla g_{\omega_{j_t}}(x_t^{k,\nu}) \nabla f_{\nu_{i_t}}(y_{t+1}^{k,\nu})\|^2\notag	\\
    &= \|x_t-x_t^{k,\nu}\|^2-2 \eta_t\langle \nabla g_{\omega_{j_t}}\left(x_t\right) \nabla f_{\nu_{i_ t}}\left(y_{t+1}\right)-\nabla g_{\omega_{j_t}}(x_t^{k,\nu}) \nabla f_{\nu_{i_t}}(y_{t+1}^{k,\nu}),x_t-x_t^{k,\nu}\rangle\notag \\
    &+\eta_t^2\|\nabla g_{\omega_{j_t}}\left(x_t\right) \nabla f_{\nu_{i_ t}}\left(y_{t+1}\right)-\nabla g_{\omega_{j_t}}(x_t^{k,\nu}) \nabla f_{\nu_{i_t}}(y_{t+1}^{k,\nu})\|^2. \label{mid}
\end{align}
 Taking the expectation w.r.t $j_t$ on the both sides of \eqref{mid} implies that
 \begin{align*}
     & \EX_{j_t}\bigl[\|x_{t+1}-x_{t+1}^{k,\nu}\|^2 \bigl]\\
		&\leq \EX_{j_t}\bigl[\|x_t-x_t^{k,\nu}\|^2\bigr]-2 \eta_t\EX_{j_t}\bigl[\langle \nabla g_{\omega_{j_t}}(x_t) \nabla f_{\nu_{i_ t}}(y_{t+1})-\nabla g_{\omega_{j_t}}(x_t^{k,\nu}) \nabla f_{\nu_{i_t}}(y_{t+1}^{k,\nu}),x_t-x_t^{k,\nu}\rangle\bigr]\notag \\
		&+\eta_t^2\EX_{j_t}\bigl[\|\nabla g_{\omega_{j_t}}(x_t) \nabla f_{\nu_{i_ t}}(y_{t+1})-\nabla g_{\omega_{j_t}}(x_t^{k,\nu}) \nabla f_{\nu_{i_t}}(y_{t+1}^{k,\nu})\|^2\bigr].\numberthis \label{expectjtstab}
 \end{align*}
 We first estimate the second term on the right hand side of \eqref{expectjtstab}.  It can be decomposed as 
 \begin{align*}
		-&2\eta_t\EX_{j_t}\bigl[\langle \nabla g_{\omega_{j_t}}(x_t) \nabla f_{\nu_{i_ t}}(y_{t+1})-\nabla g_{\omega_{j_t}}(x_t^{k,\nu}) \nabla f_{\nu_{i_t}}(y_{t+1}^{k,\nu}),x_t-x_t^{k,\nu}\rangle \bigr] \\
		=&-2\eta_t\EX_{j_t}\bigl[\langle \nabla g_{\omega_{j_t}}(x_t) \nabla f_{\nu_{i_ t}}(y_{t+1})-\nabla g_{\omega_{j_t}}(x_t) \nabla f_{\nu_{i_ t}}(g_S(x_t)),x_t-x_t^{k,\nu}\rangle \bigr] \\
		&-2\eta_t\EX_{j_t}\bigl[\langle \nabla g_{\omega_{j_t}}(x_t) \nabla f_{\nu_{i_ t}}(g_S(x_t))-\nabla g_S(x_t) \nabla f_{\nu_{i_ t}}(g_S(x_t)),x_t-x_t^{k,\nu}\rangle \bigr] \\
		&-2\eta_t\EX_{j_t}\bigl[\langle\nabla g_S(x_t) \nabla f_{\nu_{i_ t}}(g_S(x_t))-\nabla g_S(x_t^{k,\nu}) \nabla f_{\nu_{i_t}}(g_S(x_t^{k,\nu})),x_t-x_t^{k,\nu} \rangle \bigr] \\
		&-2\eta_t\EX_{j_t}\bigl[\langle\nabla g_S(x_t^{k,\nu}) \nabla f_{\nu_{i_ t}}(g_S(x_t^{k,\nu}))-\nabla g_{\omega_{j_t}}(x_t^{k,\nu}) \nabla f_{\nu_{i_t}}(g_S(x_t^{k,\nu})),x_t-x_t^{k,\nu}  \rangle\bigr]\\
		&-2\eta_t\EX_{j_t}\bigl[\langle\nabla g_{\omega_{j_t}}(x_t^{k,\nu}) \nabla f_{\nu_{i_t}}(g_S(x_t^{k,\nu})) -\nabla g_{\omega_{j_t}}(x_t^{k,\nu}) \nabla f_{\nu_{i_t}}(y_{t+1}^{k,\nu}),x_t-x_t^{k,\nu} \rangle\bigr] .\numberthis \label{expandmid}
	\end{align*}
Now we estimate the terms on the right hand side of  \eqref{expandmid} one by one. To this end, noticing that $j_t$ is independent of $i_t$ and $x_t$, then   $\EX_{j_t}\bigl[\nabla g_{\omega_{j_t}}(x_t) \nabla f_{\nu_{i_ t}}(g_S(x_t))\bigr]=\nabla g_{S}(x_t) \nabla f_{\nu_{i_ t}}(g_S(x_t))$ holds true. Consequently,  
    \begin{align*} \label{sec_forthmid}
        &-2\eta_t\EX_{j_t}\bigl[\langle \nabla g_{\omega_{j_t}}(x_t) \nabla f_{\nu_{i_ t}}(g_S(x_t))-\nabla g_{S}(x_t) \nabla f_{\nu_{i_ t}}(g_S(x_t)),x_t-x_t^{k,\nu}\rangle \bigr]=0,\\
&-2\eta_t\EX_{j_t}\bigl[\langle\nabla g_{S}(x_t^{k,\nu}) \nabla f_{\nu_{i_t}}(g_S(x_t^{k,\nu})) -\nabla g_{\omega_{j_t}}(x_t^{k,\nu}) \nabla f_{\nu_{i_t}}(g_S(x_t^{k,\nu})),x_t-x_t^{k,\nu} \rangle\bigr]=0.
\numberthis
    \end{align*}
 Then by  Part \ref{part3} of Assumption \ref{assum:2}, we know $f_\nu(g_S(\cdot))$ is  $L$-smooth. Combining this with the convexity of $f_\nu(g_S(\cdot))$ and inequality  \eqref{co-coercive},  we get
 	\begin{align*}
		&\langle \nabla g_S\bigl(x_t\bigr) \nabla f_{\nu_{i_ t}}(g_S(x_t))-\nabla g_S(x_t^{k,\nu}) \nabla f_{\nu_{i_t}}(g_S(x_t^{k,\nu})),x_t-x_t^{k,\nu}\rangle\\
		&\geq \frac{1}{L}\|\nabla g_S(x_t) \nabla f_{\nu_{i_ t}}(g_S(x_t))-\nabla g_S(x_t^{k,\nu}) \nabla f_{\nu_{i_t}}(g_S(x_t^{k,\nu}))\|^2.\numberthis \label{cor_convex}
	\end{align*}
 	Furthermore, noticing that  $x_t$ is independent of $j_t$, we get
  \begin{align*}
       &-2\eta_t\EX_{j_t}\bigl[\langle \nabla g_{\omega_{j_t}}(x_t) \nabla f_{\nu_{i_ t}}(y_{t+1})-\nabla g_{\omega_{j_t}}(x_t) \nabla f_{\nu_{i_ t}}(g_S(x_t)),x_t-x_t^{k,\nu}\rangle \bigr]\\
          & \leq2\eta_t \EX_{j_t}\bigl[\bigl| \langle \nabla g_{\omega_{j_t}}(x_t) (\nabla f_{\nu_{i_ t}}(y_{t+1})-\nabla f_{\nu_{i_ t}}(g_S(x_t))),x_t-x_t^{k,\nu}\rangle \bigr|\bigr] \\
        &\leq2\eta_t\EX_{j_t}\bigl[\|\nabla g_{\omega_{j_t}}(x_t) (\nabla f_{\nu_{i_ t}}(y_{t+1})-\nabla f_{\nu_{i_ t}}(g_S(x_t)))\| \|x_t-x_t^{k,\nu}\|\bigr]\\
		&\leq2\eta_t\EX_{j_t} \bigl[\|\nabla g_{\omega_{j_t}}(x_t)\| \|\nabla f_{\nu_{i_ t}}(y_{t+1})-\nabla f_{\nu_{i_ t}}\bigl(g_S(x_t)\bigr)\| \|x_t-x_t^{k,\nu}\|\bigr]\\
        &\leq C_fL_g2\eta_t\EX_{j_t}\bigl[\|y_{t+1}-g_S(x_t)\|\bigr]\|x_t-x_t^{k,\nu}\|,\numberthis \label{firstmid}
  \end{align*}
  where the last inequality holds by $L_g$ Lipschitz continuity of $g_\omega$ in Assumption \ref{assum:1}\ref{assum:1b} and the $C_f$ Lipschitz continuous gradients of $f_\nu$ in Assumption \ref{assum:2}\ref{part2}.
  Analogous to \eqref{firstmid}, we get
  \begin{align*}
      &-2\eta_t\EX_{j_t}\bigl[\langle\nabla g_{\omega_{j_t}}(x_t^{k,\nu}) \nabla f_{\nu_{i_t}}(g_S(x_t^{k,\nu})) -\nabla g_{\omega_{j_t}}(x_t^{k,\nu}) \nabla f_{\nu_{i_t}}(y_{t+1}^{k,\nu}),x_t-x_t^{k,\nu} \rangle\bigr]\\
      &\leq 2C_fL_g\eta_t\EX_{j_t}\bigl[\|y_{t+1}^{k,\nu}-g_S(x_t^{k,\nu})\|\bigr]\|x_t-x_t^{k,\nu}\|.\numberthis\label{lastmid}
  \end{align*}
Putting \eqref{sec_forthmid}, \eqref{cor_convex}, \eqref{firstmid} and \eqref{lastmid} into \eqref{expandmid},  we get that
     \begin{align*}\label{final_midterm}
         &-2\eta_t\EX_{j_t}\bigl[\langle \nabla g_{\omega_{j_t}}(x_t) \nabla f_{\nu_{i_ t}}(y_{t+1})-\nabla g_{\omega_{j_t}}(x_t^{k,\nu}) \nabla f_{\nu_{i_t}}(y_{t+1}^{k,\nu}),x_t-x_t^{k,\nu}\rangle\bigr]\\         &\leq2C_fL_g\eta_t\EX_{j_t}\bigl[\|y_{t+1}-g_S(x_t)\|\bigr]\|x_t-x_t^{k,\nu}\|+2C_fL_g\eta_t\EX_{j_t}\bigl[\|y_{t+1}^{k,\nu}-g_S(x_t^{k,\nu})\|\bigr]\|x_t-x_t^{k,\nu}\|\\
            &-2\eta_t\frac{1}{L}\|\nabla g_S(x_t) \nabla f_{\nu_{i_ t}}(g_S(x_t))-\nabla g_S(x_t^{k,\nu}) \nabla f_{\nu_{i_t}}(g_S(x_{t}^{k,\nu}))\|^2. \numberthis
     \end{align*}   
We estimate the third term on the right hand side of \eqref{expectjtstab} as follows:
 \begin{align*}
       &\|\nabla g_{\omega_{j_t}}(x_t) \nabla f_{\nu_{i_ t}}(y_{t+1})-\nabla g_{\omega_{j_t}}(x_t^{k,\nu}) \nabla f_{\nu_{i_t}}(y_{t+1}^{k,\nu})\|\\
     &\leq\|\nabla g_{\omega_{j_t}}(x_t) \nabla f_{\nu_{i_ t}}(y_{t+1})-\nabla g_{\omega_{j_t}}(x_t) \nabla f_{\nu_{i_ t}}(g_S(x_t))\|\\
     &+\|\nabla g_{\omega_{j_t}}(x_t) \nabla f_{\nu_{i_ t}}(g_S(x_t))-\nabla g_S(x_t) \nabla f_{\nu_{i_ t}}(g_S(x_t))\|\\
     &+\|\nabla g_S(x_t) \nabla f_{\nu_{i_ t}}(g_S(x_t))-\nabla g_S(x_t^{k,\nu}) \nabla f_{\nu_{i_t}}(g_S(x_t^{k,\nu}))\|\\
     &+\|\nabla g_S(x_t^{k,\nu}) \nabla f_{\nu_{i_ t}}(g_S(x_t^{k,\nu}))-\nabla g_{\omega_{j_t}}(x_t^{k,\nu}) \nabla f_{\nu_{i_t}}(g_S(x_t^{k,\nu}))\|\\
     &+\|\nabla g_{\omega_{j_t}}(x_t^{k,\nu}) \nabla f_{\nu_{i_t}}(g_S(x_t^{k,\nu})) -\nabla g_{\omega_{j_t}}(x_t^{k,\nu}) \nabla f_{\nu_{i_t}}(y_{t+1}^{k,\nu})\|.
 \end{align*}
 Taking square on both sides of the above inequality, we have that  
 	\begin{align*}
		\eta_t^2&\|\nabla g_{\omega_{j_t}}(x_t) \nabla f_{\nu_{i_ t}}(y_{t+1})-\nabla g_{\omega_{j_t}}(x_t^{k,\nu}) \nabla f_{\nu_{i_t}}(y_{t+1}^{k,\nu})\|^2\\
		\leq& 4\eta_t^2 C_f^2\|\nabla g_{\omega_{j_t}}(x_t) (y_{t+1}-g_S(x_t))\|^2+4\eta_t^2 C_f^2\|\nabla g_{\omega_{j_t}}(x_t) (g_S(x_t^{k,\nu})-y_{t+1}^{k,\nu})\|^2\\
		&+8\eta_t^2\|(\nabla g_{\omega_{j_t}}(x_t)-\nabla g_S(x_t)) \nabla f_{\nu_{i_ t}}(g_S(x_t))\|^2\\
		&+8\eta_t^2\|(\nabla g_{\omega_{j_t}}(x_t^{k,\nu})-\nabla g_S(x_t^{k,\nu})) \nabla f_{\nu_{i_ t}}(g_S(x_t^{k,\nu}))\|^2\\
		&+4\eta_t^2\|\nabla g_S(x_t) \nabla f_{\nu_{i_ t}}(g_S(x_t))-\nabla g_S(x_t^{k,\nu}) \nabla f_{\nu_{i_t}}(g_S(x_t^{k,\nu}))\|^2,\numberthis \label{thir_term}
	\end{align*}
	where we have used the fact that $(\sum_{i=1}^5 a_i)^2\leq 4 a_1^2+4 a_2^2 + 4a_3^2 + 8a_4^2 + 8 a_5^2$ and part \ref{part2} of Assumption \ref{assum:2}, i.e.,  $C_f$-Lipschitz continuity of $\nabla f_\nu$. Taking the expectation w.r.t. $j_t$ on both sides of \eqref{thir_term}, there holds
	\begin{align*}
	&\mathbb{E}_{{j_t}}\bigl[\eta_t^2\|\nabla g_{\omega_{j_t}}(x_t) \nabla f_{\nu_{i_ t}}(y_{t+1})-\nabla g_{\omega_{j_t}}(x_t^{k,\nu}) \nabla f_{\nu_{i_t}}(y_{t+1}^{k,\nu})\|^2\bigr]\\
		\leq&4\eta_t^2 C_f^2\EX_{j_t}\bigl[\|\nabla g_{\omega_{j_t}}(x_t) \|^2\|y_{t+1}-g_S(x_t)\|^2\bigr]+4\eta_t^2 C_f^2\EX_{j_t}\bigl[\|\nabla g_{\omega_{j_t}}(x_t) \|^2\|g_S(x_t^{k,\nu})-y_{t+1}^{k,\nu}\|^2\bigr]
		\\
        &+8\eta_t^2\EX_{j_t}\bigl[\|\nabla g_{\omega_{j_t}}(x_t)-\nabla g_S(x_t)\|^2\| \nabla f_{\nu_{i_ t}}(g_S(x_t))\|^2\bigr]\\
		&+8\eta_t^2\EX_{j_t}\bigl[\|\nabla g_{\omega_{j_t}}(x_t^{k,\nu})-\nabla g_S(x_t^{k,\nu})\|^2\| \nabla f_{\nu_{i_ t}}(g_S(x_t^{k,\nu}))\|^2\bigr]\\
        &+4\eta_t^2\|\nabla g_S(x_t) \nabla f_{\nu_{i_ t}}(g_S(x_t))-\nabla g_S(x_t^{k,\nu}) \nabla f_{\nu_{i_t}}(g_S(x_t^{k,\nu}))\|\\
        \leq& 4\eta_t^2 C_f^2L_g^2  \EX_{j_t}\bigl[\|y_{t+1}-g_S(x_t)\|^2\bigr]+4\eta_t^2 C_f^2L_g^2\EX_{j_t}\bigl[\|y_{t+1}^{k,\nu}-g_S(x_t^{k,\nu})\|^2\bigr]+16\eta_t^2L_f^2 C_g\\
        &+4\eta_t^2\|\nabla g_S(x_t) \nabla f_{\nu_{i_ t}}(g_S(x_t))-\nabla g_S(x_t^{k,\nu}) \nabla f_{\nu_{i_t}}(g_S(x_t^{k,\nu}))\|^2, \numberthis\label{third_term}
	\end{align*}
	where the second inequality follows from the Lipschitz continuity of $f_\nu$ and $g_{\omega}$ according to Assumption \ref{assum:1} as well as part \ref{part1b}
 of Assumption \ref{assum:2}. 
 
Putting (\ref{final_midterm}) and (\ref{third_term}) back into (\ref{expectjtstab}) implies that 
	\begin{align*}
		&\mathbb{E}_{{j_t}}\bigl[\|x_{t+1}-x_{t+1}^{k,\nu}\|^2\bigr]\\
  &\leq \|x_t-x_t^{k,\nu}\|^2+2C_fL_g\eta_t\EX_{j_t}\bigl[\|y_{t+1}-g_S\|\bigr]\|x_t-x_t^{k,\nu}\|\\
  &+2C_fL_g\eta_t\EX_{j_t}\bigl[\|y_{t+1}^{k,\nu}-g_S(x_t^{k,\nu})\|\bigr]\|x_t-x_t^{k,\nu}\|\\
            &+(4\eta_t^2-2\eta_t\frac{1}{L})\|\nabla g_S(x_t) \nabla f_{\nu_{i_ t}}(g_S(x_t))-\nabla g_S(x_t^{k,\nu}) \nabla f_{\nu_{i_t}}(g_S(x_{t}^{k,\nu}))\|^2\\
		&+4\eta_t^2 C_f^2L_g^2  \EX_{j_t}\bigl[\left\|y_{t+1}-g_S\left(x_t\right)\right\|^2\bigr]+4\eta_t^2 C_f^2L_g^2\EX_{j_t}\bigl[\|y_{t+1}^{k,\nu}-g_S(x_t^{k,\nu})\|^2\bigr]+16\eta_t^2L_f^2 C_g\\
 & \leq\|x_t-x_t^{k,\nu}\|^2+2C_fL_g\eta_t\EX_{j_t}\bigl[\|y_{t+1}-g_S(x_t)\|\bigr]\|x_t-x_t^{k,\nu}\|\\
  &+2C_fL_g\eta_t\EX_{j_t}\bigl[\|y_{t+1}^{k,\nu}-g_S(x_t^{k,\nu})\|\bigr]\|x_t-x_t^{k,\nu}\|\\
		&+4\eta_t^2 C_f^2L_g ^2 \EX_{j_t}\bigl[\left\|y_{t+1}-g_S\left(x_t\right)\right\|^2\bigr]+4\eta_t^2 C_f^2L_g^2\EX_{j_t}\bigl[\|y_{t+1}^{k,\nu}-g_S(x_t^{k,\nu})\|^2\bigr]+16\eta_t^2L_f^2 C_g, 
	\end{align*}
 where in the second inequality we have used  the fact that $\eta_t\leq\frac{1}{2L}$.

 \medskip 
 	\textbf{\quad Case 2 ($i_t = k $).}~~  If $i_t = k $, we have	
 \begin{align*}
		&\|x_{t+1}-x_{t+1}^{k,\nu}\| =\| x_t-\eta_t \nabla g_{\omega_{j_t}}(x_t) \nabla f_{\nu_{i_ t}}(y_{t+1})-x_t^{k,\nu}+\eta_t \nabla g_{\omega_{j_t}}(x_t^{k,\nu}) \nabla f_{\nu_{i_t}^{\prime}}(y_{t+1}^{k,\nu})\| \\
		&\leq \|x_t-x_t^{k,\nu}\|+\eta_t\|\nabla g_{\omega_{j_t}}(x_t) \nabla f_{\nu_{i_ t}}(y_{t+1})-\nabla g_{\omega_{j_t}}(x_t^{k,\nu}) \nabla f_{\nu_{i_t}^{\prime}}(y_{t+1}^{k,\nu})\|\\
        &\leq\|x_t-x_t^{k,\nu}\|+\eta_t\|\nabla g_{\omega_{j_t}}(x_t)\|\| \nabla f_{\nu_{i_ t}}(y_{t+1})\|+\eta_t\|\nabla g_{\omega_{j_t}}(x_t^{k,\nu}) \|\|\nabla f_{\nu_{i_t}^{\prime}}(y_{t+1}^{k,\nu})\|\\
        &\leq \|x_t-x_t^{k,\nu}\|+2L_g L_f\eta_t,
	\end{align*}
 where in the third inequality we have used Assumption \ref{assum:1}, i.e., the  Lipschitz continuity of $f_\nu$ and $g_{\omega}$.  Taking the square of the terms on both sides of the above inequality and taking the expectation w.r.t.  $j_t$ yield that  
\begin{align*}
\EX_{j_t}\bigl[\|x_{t+1}-x_{t+1}^{k,\nu}\|^2\bigr]\leq\|x_t-x_t^{k,\nu}\|^2+4L_gL_f\eta_t\|x_t-x_t^{k,\nu}\|+4L_g^2 L_f^2\eta_t^2.\numberthis\label{conveitcase_2}
\end{align*}
\smallskip  
Combining {\bf Case 1} and {\bf Case 2} together, we have that 
\begin{align*}
	 &\mathbb{E}_{{j_t}}\bigl[\|x_{t+1}-x_{t+1}^{k,\nu}\|^2\bigr]\leq \|x_t-x_t^{k,\nu}\|^2+2C_fL_g\eta_t\EX_{j_t}\bigl[\|y_{t+1}-g_S(x_t)\|\bigr]\|x_t-x_t^{k,\nu}\|\\
  &+2C_fL_g\eta_t\EX_{j_t}\bigl[\|y_{t+1}^{k,\nu}-g_S(x_t^{k,\nu})\|\bigr]\|x_t-x_t^{k,\nu}\|\\
		&+4\eta_t^2 C_f^2L_g ^2 \EX_{j_t}\bigl[\|y_{t+1}-g_S(x_t)\|^2\bigr]+4\eta_t^2 C_f^2L_g^2\EX_{j_t}\bigl[\|y_{t+1}^{k,\nu}-g_S(x_t^{k,\nu})\|^2\bigr]+16\eta_t^2L_f^2C_g\\
  &+4L_gL_f\eta_t\|x_t-x_t^{k,\nu}\|\mathbb{I}_{[i_t=k]}+4L_g^2 L_f^2\eta_t^2\mathbb{I}_{[i_t=k]}.\numberthis\label{E_jtlast}
\end{align*}
Taking the expectation w.r.t. $A$ on both sides of \eqref{E_jtlast}, we get that   
\begin{align*}
	 &\mathbb{E}_{A}\bigl[\|x_{t+1}-x_{t+1}^{k,\nu}\|^2\bigr]\\
  \leq&\EX_{A}\bigl[ \|x_t-x_t^{k,\nu}\|^2\bigr]+2C_fL_g\eta_t\EX_{A}\bigl[\EX_{j_t}[\|y_{t+1}-g_S(x_t)\|]\|x_t-x_t^{k,\nu}\|\bigr]\\
  &+2C_fL_g\eta_t\EX_{A}\bigl[\EX_{j_t}[\|y_{t+1}^{k,\nu}-g_S(x_t^{k,\nu})\|]\|x_t-x_t^{k,\nu}\|\bigr]\\
		&+4\eta_t^2 C_f^2L_g^2  \EX_{A}\bigl[\|y_{t+1}-g_S(x_t)\|^2\bigr]+4\eta_t^2 C_f^2L_g^2\EX_{A}\bigl[\|y_{t+1}^{k,\nu}-g_S(x_t^{k,\nu})\|^2\bigr]+16\eta_t^2L_f^2 C_g\\
  &+4L_f L_g\eta_t\EX_A[\|x_t-x_t^{k,\nu}\|\mathbb{I}_{[i_t=k]}]+4L_g^2 L_f^2\eta_t^2\EX_A[\mathbb{I}_{[i_t=k]}]\\
  \leq&\EX_{A}\bigl[ \|x_t-x_t^{k,\nu}\|^2\bigr]+2C_fL_g\eta_t(\EX_{A}\bigl[\|y_{t+1}-g_S(x_t)\|^2\bigr])^{1/2}(\EX_{A}\bigl[\|x_t-x_t^{k,\nu}\|^2\bigr])^{1/2}\\
  &+2C_fL_g\eta_t(\EX_{A}\bigl[\|y_{t+1}^{k,\nu}-g_S(x_t^{k,\nu})\|^2\bigr])^{1/2}(\EX_{A}\bigl[\|x_t-x_t^{k,\nu}\|^2\bigr])^{1/2}\\
		&+4\eta_t^2 C_f^2L_g^2  \EX_{A}\bigl[\|y_{t+1}-g_S(x_t)\|^2\bigr]+4\eta_t^2 C_f^2L_g^2\EX_{A}\bigl[\|y_{t+1}^{k,\nu}-g_S(x_t^{k,\nu})\|^2\bigr]+16\eta_t^2L_f^2 C_g\\
  &+4L_fL_g\eta_t\EX_A[\|x_t-x_t^{k,\nu}\|\mathbb{I}_{[i_t=k]}]+4L_g^2 L_f^2\eta_t^2\EX_A[\mathbb{I}_{[i_t=k]}],\numberthis\label{E_Ajtconvex}
\end{align*}

where the second inequality holds by the Cauchy-Schwarz inequality. Observe  that
\begin{align*}
    \EX_A[\|x_t-x_t^{k,\nu}\|\mathbb{I}_{[i_t=k]}]=\EX_A[\|x_t-x_t^{k,\nu}\|\EX_{i_t}[\mathbb{I}_{[i_t=k]}]]=\frac{1}{n}\EX_A[\|x_t-x_t^{k,\nu}\|]\leq\frac{1}{n}(\EX_A[\|x_t-x_t^{k,\nu}\|^2])^{1/2}.
\end{align*}
Note that $\|x_0-x_0^{k,\nu}\|^2=0$. Combining above observation with \eqref{E_Ajtconvex} implies that 
\begin{align*}
	&\mathbb{E}_{A}\bigl[\|x_{t+1}-x_{t+1}^{k,\nu}\|^2\bigr]\\
 &\leq 
 2C_fL_g\sum_{j=1}^{t}\eta_j(\EX_{A}\bigl[\|y_{j+1}-g_S(x_j)\|^2\bigr])^{1/2}\bigl(\EX_A\bigl[\|x_j-x_j^{k,\nu}\|^2\bigr]\bigr)^{1/2}\\
         &+2C_fL_g\sum_{j=1}^{t}\eta_j(\EX_{A}\bigl[\|y_{j+1}^{k,\nu}-g_S(x_j^{k,\nu})\|^2\bigr])^{1/2}\bigl(\EX_A\bigl[\|x_j-x_j^{k,\nu}\|^2\bigr]\bigr)^{1/2}\\
		&+4C_f^2L_g^2 \sum_{j=0}^{t}\eta_j^2  \EX_{A}\bigl[\|y_{j+1}-g_S(x_j)\|^2\bigr]+4C_f^2L_g^2 \sum_{j=0}^{t}\eta_j^2 \EX_{A}\bigl[\|y_{j+1}^{k,\nu}-g_S(x_j^{k,\nu})\|^2\bigr]\\
        &+16L_f^2C_g\sum_{j=0}^{t}\eta_j^2+\frac{4L_gL_f}{n}\sum_{j=1}^{t}\eta_j(\EX_A\bigl[\|x_{j}-x_{j}^{k,\nu}\|^2\bigr])^{1/2}
        +\frac{4L_f^2 L_g^2}{n}\sum_{j=0}^{t}\eta_j^2. \numberthis \label{recursionjtconve}
\end{align*}
For notational convenience, we denote by $u_{t}=(\mathbb{E}_{A}\bigl[\|x_{t}-x_{t}^{k,\nu}\|^2\bigr])^{1/2}$. Using this notation,  from  \eqref{recursionjtconve} we get that 
\begin{align*}
	u_{t}^2& \leq 
 2C_fL_g\sum_{j=1}^{t-1}\eta_j(\EX_{A}\bigl[\|y_{j+1}-g_S(x_j)\|^2\bigr])^{1/2}u_{j}+2C_fL_g\sum_{j=1}^{t-1}\eta_j(\EX_{A}\bigl[\|y_{j+1}^{k,\nu}-g_S(x_j^{k,\nu})\|^2\bigr])^{1/2}u_j\\
		&+4C_f^2L_g^2 \sum_{j=0}^{t-1}\eta_j^2  \EX_{A}\bigl[\|y_{j+1}-g_S(x_j)\|^2\bigr]+4C_f^2L_g ^2\sum_{j=0}^{t-1}\eta_j^2 \EX_{A}\bigl[\|y_{j+1}^{k,\nu}-g_S(x_j^{k,\nu})\|^2\bigr]\\
        &+16L_f^2C_g\sum_{j=0}^{t-1}\eta_j^2+\frac{4L_gL_f}{n}\sum_{j=1}^{t-1}\eta_j u_j
        +\frac{4L_f^2 L_g^2}{n}\sum_{j=0}^{t-1}\eta_j^2.
\end{align*}
We will apply Lemma \ref{recursion lemma} to get the desired estimation from the above recursive inequality. To this end, we define
\begin{align*}
    S_{t}&=4C_f^2L_g^2 \sum_{j=0}^{t-1}\eta_j^2  \EX_{A}\bigl[\left\|y_{j+1}-g_S\left(x_j\right)\right\|^2\bigr]+4C_f^2L_g^2 \sum_{j=0}^{t-1}\eta_j^2 \EX_{A}\bigl[\|y_{j+1}^{k,\nu}-g_S(x_j^{k,\nu})\|^2\bigr]\\ &+\frac{4L_f^2 L_g^2}{n}\sum_{j=0}^{t-1}\eta_j^2+16L_f ^2C_g\sum_{j=0}^{t-1}\eta_j^2,\\
    \alpha_j&=2C_fL_g\eta_j(\EX_{A}\bigl[\|y_{j+1}-g_S(x_j)\|^2\bigr])^{1/2}+2C_fL_g\eta_j(\EX_{A}\bigl[\|y_{j+1}^{k,\nu}-g_S(x_j^{k,\nu})\|^2\bigr])^{1/2}+\frac{4L_gL_f}{n}\eta_j.
\end{align*}
Now applying Lemma \ref{recursion lemma} with $u_{t}$, $S_{t}$ and $\alpha_j$ defined above, we get    
    \begin{align*} \label{eq:stability_delta}
        u_{t}&\leq\sqrt{S_{t}}+\sum_{j=1}^{t-1}\alpha_j\\
        &\leq\bigl(4C_f^2L_g ^2\sum_{j=0}^{t-1}\eta_j^2  \EX_{A}\bigl[\left\|y_{j+1}-g_S\left(x_j\right)\right\|^2\bigr]\bigr)^{1/2}+\bigl(4C_f^2L_g^2 \sum_{j=0}^{t-1}\eta_j^2 \EX_{A}\bigl[\|y_{j+1}^{k,\nu}-g_S(x_j^{k,\nu})\|^2\bigr]\bigr)^{1/2} \\
        &+\bigl(\frac{4L_f^2 L_g^2}{n}\sum_{j=0}^{t-1}\eta_j^2\bigr)^{1/2}+(16L_f ^2C_g\sum_{j=0}^{t-1}\eta_j^2)^{1/2}+2C_fL_g\sum_{j=1}^{t-1}\eta_j(\EX_{A}\bigl[\|y_{j+1}-g_S(x_j)\|^2\bigr])^{1/2}\\
        &+2C_fL_g\sum_{j=1}^{t-1}\eta_j(\EX_{A}\bigl[\|y_{j+1}^{k,\nu}-g_S(x_j^{k,\nu})\|^2\bigr])^{1/2}+\frac{4L_fL_g}{n}\sum_{j=1}^{t-1}\eta_j \\
        &\leq 4C_fL_g\sum_{j=0}^{t-1}\eta_j(\EX_{A}\bigl[\|y_{j+1}-g_S(x_j)\|^2\bigr])^{1/2}
        +4C_fL_g\sum_{j=0}^{t-1}\eta_j(\EX_{A}\bigl[\|y_{j+1}^{k,\nu}-g_S(x_j^{k,\nu})\|^2\bigr])^{1/2}\\
        &+4L_f\sqrt{C_g}(\sum_{j=0}^{t-1}\eta_j^2)^{1/2}+\bigl(\frac{4L_f^2L_g^2}{n}\sum_{j=0}^{t-1}\eta_j^2\bigr)^{1/2}+\frac{4L_fL_g}{n}\sum_{j=0}^{t-1}\eta_j, \numberthis
    \end{align*}
where the second inequality uses the fact that $(\sum_{i=1}^{4}a_i)^{1/2}\leq \sum_{i=1}^{4}(a_i)^{1/2}$ and the last inequality holds by the fact that
\begin{align*}
    \bigl(4C_f^2L_g ^2\sum_{j=0}^{t-1}\eta_j^2  \EX_{A}\bigl[\left\|y_{j+1}-g_S\left(x_j\right)\right\|^2\bigr]\bigr)^{1/2}&\leq2C_fL_g\sum_{j=0}^{t-1}\eta_j(\EX_{A}\bigl[\|y_{j+1}-g_S(x_j)\|^2\bigr])^{1/2},\\
    \bigl(4C_f^2L_g ^2\sum_{j=0}^{t-1}\eta_j^2  \EX_{A}\bigl[\|y_{j+1}^{k,\nu}-g_S(x_j^{k,\nu})\|^2\bigr]\bigr)^{1/2}&\leq2C_fL_g\sum_{j=0}^{t-1}\eta_j(\EX_{A}\bigl[\|y_{j+1}^{k,\nu}-g_S(x_j^{k,\nu})\|^2\bigr])^{1/2}.
\end{align*}
Furthermore,  if $\eta_t=\eta$, it is easy to see that  $\sum_{j=0}^{T-1} (\mathbb{E}_{A}[\|y_{j+1}-g_S(x_j)\|^2])^{1/2} \le \sup_{S} \eta\sum_{j=0}^{T-1} (\mathbb{E}_{A}[\|y_{j+1}-g_S(x_j)\|^2])^{1/2}$ and $\sum_{j=0}^{T-1} (\mathbb{E}_{A}[\|y_{j+1}^{k,\nu}-g_S(x_j^{k,\nu})\|^2])^{1/2} \le \sup_{S} \eta\sum_{j=0}^{T-1} (\mathbb{E}_{A}[\|y_{j+1}-g_S(x_j)\|^2])^{1/2}.$ Consequently, with $T$ iterations,  we obtain that 
\begin{align*}
    u_{T}
    \leq &8C_fL_g\sup_{S} \eta\sum_{j=0}^{T-1} (\mathbb{E}_{A}[\|y_{j+1}-g_S(x_j)\|^2])^{1/2}+4L_f\sqrt{C_g}(\sum_{j=0}^{T-1}\eta^2)^{1/2}+\bigl(\frac{4L_f^2L_g^2}{n}\sum_{j=0}^{T-1}\eta^2\bigr)^{1/2}\\
    &+\frac{4L_gL_f}{n}\sum_{j=0}^{T-1}\eta\\
    \leq&8C_fL_g\sup_{S} \eta\sum_{j=0}^{T-1} (\mathbb{E}_{A}[\|y_{j+1}-g_S(x_j)\|^2])^{1/2}+4L_f\sqrt{C_g}\eta\sqrt{T}+\frac{6L_gL_f}{n}\eta T,
\end{align*}
where the last inequality holds by the fact that
$\bigl(\frac{4L_f^2L_g^2}{n}\sum_{j=0}^{T-1}\eta^2\bigr)^{1/2}= \frac{2L_gL_f}{\sqrt{n}}\eta\sqrt{T}\leq  \frac{2L_fL_g}{n}\eta T$ because often we have $T\ge n$. \\
Since
$\mathbb{E}_{A}\bigl[\|x_{T}-x_{T}^{k,\nu}\|\bigr]\leq u_{T}=(\mathbb{E}_{A}\bigl[\|x_{T}-x_{T}^{k,\nu}\|^2\bigr])^{1/2}$,    we further get 
\begin{align*}  \label{eq:stability_delta:3}
    \mathbb{E}_{A}\bigl[\|x_{T}-x_{T}^{k,\nu}\|\bigr] 
   \leq 8C_fL_g\sup_{S} \eta\sum_{j=0}^{T-1} (\mathbb{E}_{A}[\|y_{j+1}-g_S(x_j)\|^2])^{1/2}
        +4L_f\sqrt{C_g}\eta\sqrt{T}+\frac{6L_fL_g}{n}\eta T.
        \numberthis
\end{align*}
We got the following desired result for { Case 1}: 
\begin{align*}
\mathbb{E}_{A}\bigl[\|x_{T}-x_{T}^{k,\nu}\|\bigr]= \mathcal{O}\Bigl( L_fL_g\frac{T\eta}{n}  
+L_f\sqrt{C_g}\eta\sqrt{T}+ C_fL_g\sup_{S}\sum_{j=0}^{T-1} \eta \bigl(\mathbb{E}_{A}[\| y_{j+1}-g_S(x_j)\|^2 ] \bigr)^{1\over 2}\Bigr).
\end{align*}

Next we move on to the estimation of  $\mathbb{E}_{A}[\|x_{t+1}-x_{t+1}^{l,\omega}\|]$. 

\smallskip 
\noindent{\bf Estimation of $\mathbb{E}_{A}[\|x_{t+1}-x_{t+1}^{l,\omega}\|]$.}  We will estimate it by considering two cases, i.e., $j_t \neq l $ and $j_t =l $.

 \medskip 
\textbf{~ Case 1 ($j_t \neq l $). }~~ If $j_t \neq l $, we have
	\begin{align*}
		&\|x_{t+1}-x_{t+1}^{l,\omega}\|^2 \leq \| x_t-\eta_t \nabla g_{\omega_{j_t}}(x_t) \nabla f_{\nu_{i_ t}}(y_{t+1})-x_t^{l,\omega}+\eta_t \nabla g_{\omega_{j_t}}(x_t^{l,\omega}) \nabla f_{\nu_{i_t}}(y_{t+1}^{l,\omega})\|^2	\\
		&= \|x_t-x_t^{l,\omega}\|^2-2 \eta_t\langle \nabla g_{\omega_{j_t}}(x_t) \nabla f_{\nu_{i_ t}}(y_{t+1})-\nabla g_{\omega_{j_t}}(x_t^{l,\omega}) \nabla f_{\nu_{i_t}}(y_{t+1}^{l,\omega}),x_t-x_t^{l,\omega}\rangle \\
		&+\eta_t^2\|\nabla g_{\omega_{j_t}}(x_t) \nabla f_{\nu_{i_ t}}(y_{t+1})-\nabla g_{\omega_{j_t}}(x_t^{l,\omega}) \nabla f_{\nu_{i_t}}(y_{t+1}^{l,\omega})\|^2 .\numberthis \label{j_tneql}
	\end{align*}
We will estimate the second term and the third one on the right hand side of \eqref{j_tneql} as follows.  First, we estimate the second term. To this end, using similar arguments in \eqref{expandmid}, it can be decomposed as 
\begin{align*}
		-&2\eta_t\langle \nabla g_{\omega_{j_t}}(x_t) \nabla f_{\nu_{i_ t}}(y_{t+1})-\nabla g_{\omega_{j_t}}(x_t^{l,\omega}) \nabla f_{\nu_{i_t}}(y_{t+1}^{l,\omega}),x_t-x_t^{l,\omega}\rangle \notag \\
  =& -2\eta_t  \langle \nabla g_{\omega_{j_t}}(x_t) (\nabla f_{\nu_{i_ t}}(y_{t+1})-\nabla f_{\nu_{i_ t}}(g_S(x_t))),x_t-x_t^{l,\omega}\rangle  \\
            &-2\eta_t\langle \nabla g_{\omega_{j_t}}(x_t) \nabla f_{\nu_{i_ t}}(g_S(x_t))-\nabla g_S(x_t) \nabla f_{\nu_{i_ t}}(g_S(x_t)),x_t-x_t^{l,\omega}\rangle\\
		&-2\eta_t\langle\nabla g_S(x_t) \nabla f_{\nu_{i_ t}}(g_S(x_t))-\nabla g_S(x_t^{l,\omega}) \nabla f_{\nu_{i_t}}(g_S(x_{t}^{l,\omega})), x_t-x_t^{l,\omega}\rangle\\
            &-2\eta_t\langle\nabla g_S(x_t^{l,\omega}) \nabla f_{\nu_{i_ t}}(g_S(x_t^{l,\omega}))-\nabla g_{\omega_{j_t}}(x_t^{l,\omega}) \nabla f_{\nu_{i_t}}(g_S(x_t^{l,\omega})),x_t-x_t^{l,\omega} \notag  \rangle\\
		&-2\eta_t\langle\nabla g_{\omega_{j_t}}(x_t^{l,\omega}) (\nabla f_{\nu_{i_ t}}(g_S(x_t^{l,\omega}))-\nabla f_{\nu_{i_t}}(y_{t+1}^{l,\omega})),x_t-x_t^{l,\omega}\rangle
  \numberthis \label{jtmidterm}.
	\end{align*}
Using the convexity of $f_\nu(g_S(\cdot))$, part \ref{part3} of Assumption \ref{assum:2} and inequality \eqref{co-coercive}, we have
 	\begin{align*}
		&\langle \nabla g_S\bigl(x_t\bigr) \nabla f_{\nu_{i_t}}(g_S(x_t))-\nabla g_S(x_t^{l,\omega}) \nabla f_{\nu_{i_t}}(g_S(x_t^{l,\omega})),x_t-x_t^{l,\omega}\rangle\\
		&\geq \frac{1}{L}\|\nabla g_S(x_t) \nabla f_{\nu_{i_t}}(g_S(x_t))-\nabla g_S(x_t^{l,\omega}) \nabla f_{\nu_{i_t}}(g_S(x_t^{l,\omega}))\|^2.\numberthis\label{cor_convex_jt}
	\end{align*}
Furthermore, using part \ref{assum:1b}  of  Assumption \ref{assum:1}  and part \ref{part2} of Assumption \ref{assum:2}, we get
      \begin{align*}\label{jtcoti_prpety}
          & -2\eta_t  \langle \nabla g_{\omega_{j_t}}(x_t) (\nabla f_{\nu_{i_ t}}(y_{t+1})-\nabla f_{\nu_{i_ t}}(g_S(x_t))),x_t-x_t^{l,\omega}\rangle  \\
        &\leq2\eta_t\|\nabla g_{\omega_{j_t}}(x_t) \bigl(\nabla f_{\nu_{i_ t}}(y_{t+1})-\nabla f_{\nu_{i_ t}}\bigl(g_S(x_t)\bigr)\| \|x_t-x_t^{l,\omega}\|\\
		&\leq 2\eta_t\|\nabla g_{\omega_{j_t}}(x_t)\| \|\nabla f_{\nu_{i_ t}}(y_{t+1})-\nabla f_{\nu_{i_ t}}\bigl(g_S(x_t)\bigr)\| \|x_t-x_t^{l,\omega}\| \\
        &\leq 2\eta_t C_fL_g\|y_{t+1}-g_S(x_t)\|\|x_t-x_t^{l,\omega}\|. \numberthis
      \end{align*}
Likewise,  
  \begin{align*}
      &-2\eta_t\langle\nabla g_{\omega_{j_t}}(x_t^{l,\omega}) (\nabla f_{\nu_{i_ t}}(g_S(x_t^{l,\omega}))-\nabla f_{\nu_{i_t}}(y_{t+1}^{l,\omega})),x_t-x_t^{l,\omega}\rangle\\
        &\leq 2\eta_tC_fL_g\|y_{t+1}^{l,\omega}-g_S(x_t^{l,\omega})\|\|x_t-x_t^{l,\omega}\|.\numberthis\label{jtconti_b}
  \end{align*}
Putting  \eqref{cor_convex_jt}, \eqref{jtcoti_prpety} and \eqref{jtconti_b} into \eqref{jtmidterm} yields that 
	\begin{align*}
		-&2\eta_t\langle \nabla g_{\omega_{j_t}}(x_t) \nabla f_{\nu_{i_ t}}(y_{t+1})-\nabla g_{\omega_{j_t}}(x_t^{l,\omega}) \nabla f_{\nu_{i_t}}(y_{t+1}^{l,\omega}),x_t-x_t^{l,\omega}\rangle \notag \\
  &\leq 2\eta_t C_fL_g\|y_{t+1}-g_S(x_t)\|\|x_t-x_t^{l,\omega}\|+ 2\eta_t C_fL_g\|y_{t+1}^{l,\omega}-g_S(x_t^{l,\omega})\|\|x_t-x_t^{l,\omega}\| \\
            &-2\eta_t\langle \nabla g_{\omega_{j_t}}(x_t) \nabla f_{\nu_{i_ t}}(g_S(x_t))-\nabla g_S(x_t) \nabla f_{\nu_{i_ t}}(g_S(x_t)),x_t-x_t^{l,\omega}\rangle\\
		&-2\eta_t\frac{1}{L}\|\nabla g_S(x_t) \nabla f_{\nu_{i_ t}}(g_S(x_t))-\nabla g_S(x_t^{l,\omega}) \nabla f_{\nu_{i_t}}(g_S(x_{t}^{l,\omega}))\|^2\\
            &-2\eta_t\langle\nabla g_S(x_t^{l,\omega}) \nabla f_{\nu_{i_ t}}(g_S(x_t^{l,\omega}))-\nabla g_{\omega_{j_t}}(x_t^{l,\omega}) \nabla f_{\nu_{i_t}}(g_S(x_t^{l,\omega})),x_t-x_t^{l,\omega}   \rangle
	 \numberthis \label{jtfnlmidterm}.
	\end{align*}
 Next we  will estimate the third term on the right hand side of \eqref{j_tneql}. In analogy  to the argument in \eqref{thir_term}, one can show that
 \begin{align*}
		&\eta_t^2\|\nabla g_{\omega_{j_t}}(x_t) \nabla f_{\nu_{i_ t}}(y_{t+1})-\nabla g_{\omega_{j_t}}(x_t^{l,\omega}) \nabla f_{\nu_{i_t}}(y_{t+1}^{l,\omega})\|^2\\
		&\leq 4\eta_t^2 C_f^2\|\nabla g_{\omega_{j_t}}(x_t) (y_{t+1}-g_S(x_t))\|^2+4\eta_t^2 C_f^2\|\nabla g_{\omega_{j_t}}(x_t) (g_S(x_t^{l,\omega})-y_{t+1}^{l,\omega})\|^2\\
		&+8\eta_t^2\|(\nabla g_{\omega_{j_t}}(x_t)-\nabla g_S(x_t)) \nabla f_{\nu_{i_ t}}(g_S(x_t))\|^2\\
		&+8\eta_t^2\|(\nabla g_{\omega_{j_t}}(x_t^{l,\omega})-\nabla g_S(x_t^{l,\omega})) \nabla f_{\nu_{i_ t}}(g_S(x_t^{l,\omega}))\|^2\\
		&+4\eta_t^2\|\nabla g_S(x_t) \nabla f_{\nu_{i_ t}}(g_S(x_t))-\nabla g_S(x_t^{l,\omega}) \nabla f_{\nu_{i_t}}(g_S(x_t^{l,\omega}))\|^2\\
  &\leq 4\eta_t^2 C_f^2L_g^2\| y_{t+1}-g_S(x_t)\|^2+4\eta_t^2 C_f^2L_g^2\|\nabla g_S(x_t^{l,\omega})-y_{t+1}^{l,\omega}\|^2\\
		&+8L_f^2\eta_t^2\|\nabla g_{\omega_{j_t}}(x_t)-\nabla g_S(x_t) \|^2+8L_f^2\eta_t^2\|\nabla g_{\omega_{j_t}}(x_t^{l,\omega})-\nabla g_S(x_t^{l,\omega}) \|^2\\
		&+4\eta_t^2\|\nabla g_S(x_t) \nabla f_{\nu_{i_ t}}(g_S(x_t))-\nabla g_S(x_t^{l,\omega}) \nabla f_{\nu_{i_t}}(g_S(x_t^{l,\omega}))\|^2, \numberthis \label{jtthir_term}
	\end{align*}
 where, in the second inequality, we have used  Assumption \ref{assum:1}.

Putting the results \eqref{jtfnlmidterm} and \eqref{jtthir_term} into \eqref{j_tneql} implies that 
	\begin{align*}
		&\|x_{t+1}-x_{t+1}^{l,\omega}\|^2\\ 
		 &\leq \|x_{t}-x_{t}^{l,\omega}\|^2 +2\eta_t C_fL_g\|y_{t+1}-g_S(x_t)\|\|x_t-x_t^{l,\omega}\|+ 2\eta_t C_fL_g\|y_{t+1}^{l,\omega}-g_S(x_t^{l,\omega})\|\|x_t-x_t^{l,\omega}\| \\
            &-2\eta_t\langle \nabla g_{\omega_{j_t}}(x_t) \nabla f_{\nu_{i_ t}}(g_S(x_t))-\nabla g_S(x_t) \nabla f_{\nu_{i_ t}}(g_S(x_t)),x_t-x_t^{l,\omega}\rangle\\
            &-2\eta_t\langle\nabla g_S(x_t^{l,\omega}) \nabla f_{\nu_{i_ t}}(g_S(x_t^{l,\omega}))-\nabla g_{\omega_{j_t}}(x_t^{l,\omega}) \nabla f_{\nu_{i_t}}(g_S(x_t^{l,\omega})),x_t-x_t^{l,\omega} \notag  \rangle\\
            &+(4\eta_t^2-2\eta_t\frac{1}{L})\|\nabla g_S(x_t) \nabla f_{\nu_{i_ t}}(g_S(x_t))-\nabla g_S(x_t^{l,\omega}) \nabla f_{\nu_{i_t}}(g_S(x_{t}^{l,\omega}))\|^2\\
            &+4\eta_t^2 C_f^2L_g^2\| y_{t+1}-g_S(x_t)\|^2+4\eta_t^2 C_f^2L_g^2\|g_S(x_t^{l,\omega})-y_{t+1}^{l,\omega}\|^2\\
		&+8L_f^2\eta_t^2\|\nabla g_{\omega_{j_t}}(x_t)-\nabla g_S(x_t) \|^2+8L_f^2\eta_t^2\|\nabla g_{\omega_{j_t}}(x_t^{l,\omega})-\nabla g_S(x_t^{l,\omega}) \|^2\\
   &\leq \|x_{t}-x_{t}^{l,\omega}\|^2 +2\eta_t C_fL_g\|y_{t+1}-g_S(x_t)\|\|x_t-x_t^{l,\omega}\|+ 2\eta_t C_fL_g\|y_{t+1}^{l,\omega}-g_S(x_t^{l,\omega})\|\|x_t-x_t^{l,\omega}\| \\
            &-2\eta_t\langle \nabla g_{\omega_{j_t}}(x_t) \nabla f_{\nu_{i_ t}}(g_S(x_t))-\nabla g_S(x_t) \nabla f_{\nu_{i_ t}}(g_S(x_t)),x_t-x_t^{l,\omega}\rangle\\
            &-2\eta_t\langle\nabla g_S(x_t^{l,\omega}) \nabla f_{\nu_{i_ t}}(g_S(x_t^{l,\omega}))-\nabla g_{\omega_{j_t}}(x_t^{l,\omega}) \nabla f_{\nu_{i_t}}(g_S(x_t^{l,\omega})),x_t-x_t^{l,\omega} \notag  \rangle\\
             &+4\eta_t^2 C_f^2L_g^2\| y_{t+1}-g_S(x_t)\|^2+4\eta_t^2 C_f^2L_g^2\|g_S(x_t^{l,\omega})-y_{t+1}^{l,\omega}\|^2\\
		&+8L_f^2\eta_t^2\|\nabla g_{\omega_{j_t}}(x_t)-\nabla g_S(x_t) \|^2+8L_f^2\eta_t^2\|\nabla g_{\omega_{j_t}}(x_t^{l,\omega})-\nabla g_S(x_t^{l,\omega}) \|^2, 
		\numberthis \label{j_tfnlneql}
	\end{align*}
 where we have used the fact that  $\eta_t\leq\frac{1}{2L}$ in the second inequality.

  \smallskip 
  
 \textbf{\quad Case 2 ($j_t=l$).}~~ If $j_t=l$, from Assumption \ref{assum:1} we have that
  \begin{align*}
		&\|x_{t+1}-x_{t+1}^{l,\omega}\| =\| x_t-\eta_t \nabla g_{\omega_{j_t}}(x_t) \nabla f_{\nu_{i_ t}}(y_{t+1})-x_t^{l,\omega}+\eta_t \nabla g_{\omega_{j_t}^{\prime}}(x_t^{l,\omega}) \nabla f_{\nu_{i_t}}(y_{t+1}^{l,\omega})\| \\
		&\leq \|x_t-x_t^{l,\omega}\|+\|\nabla g_{\omega_{j_t}}(x_t)\nabla f_{\nu_{i_ t}}(y_{t+1})+\nabla g_{\omega_{j_t}^{\prime}}(x_t^{l,\omega}) \nabla f_{\nu_{i_t}}(y_{t+1}^{l,\omega})\|\\
		&\leq\|x_t-x_t^{l,\omega}\|+\eta_t\|\nabla g_{\omega_{j_t}}(x_t)\|\| \nabla f_{\nu_{i_ t}}(y_{t+1})\|+\eta_t\|\nabla g_{\omega_{j_t}^{\prime}}(x_t^{l,\omega})\|\| \nabla f_{\nu_{i_ t}}(g_S(x_t^{l,\omega}))\| \\
  &\leq\|x_t-x_t^{l,\omega}\|+2\eta_t L_g L_f.
	\end{align*}
Therefore,
 \begin{align*}
 &\|x_{t+1}-x_{t+1}^{l,\omega}\|^2\leq \|x_{t}-x_{t}^{l,\omega}\|^2+4L_gL_f\eta_t\|x_{t}-x_{t}^{l,\omega}\|+4\eta_t^2L_g^2L_f^2.\numberthis \label{i_tequl}
	\end{align*}
 Combining {\bf Case 1} and {\bf Case 2} together, we obtain
 \begin{align*}
      &\|x_{t+1}-x_{t+1}^{l,\omega}\|^2\\
        &\leq \|x_{t}-x_{t}^{l,\omega}\|^2+2C_fL_g\eta_t\left\|y_{t+1}-g_S\left(x_t\right)\right\|\|x_{t}-x_{t}^{l,\omega}\|
    +2C_fL_g\eta_t\|y_{t+1}^{l,\omega}-g_S(x_t^{l,\omega})\|\|x_{t}-x_{t}^{l,\omega}\|\\
        & -2\eta_t\langle \nabla g_{\omega_{j_t}}\left(x_t\right) \nabla f_{\nu_{i_ t}}\left(g_S\left(x_t\right)\right)-\nabla g_S\left(x_t\right) \nabla f_{\nu_{i_ t}}\left(g_S\left(x_t\right)\right),x_t-x_t^{l,\omega}\rangle\mathbb{I}_{[j_t\neq l]}\\
            &-2\eta_t\langle\nabla g_S(x_t^{l,\omega}) \nabla f_{\nu_{i_ t}}(g_S(x_t^{l,\omega}))-\nabla g_{\omega_{j_t}}(x_t^{l,\omega}) \nabla f_{\nu_{i_t}}(g_S(x_t^{l,\omega})),x_t-x_t^{l,\omega}  \rangle\mathbb{I}_{[j_t\neq l]}\\
            &+8L_f^2\eta_t^2\left\|\nabla g_{\omega_{j_t}}\left(x_t\right)-\nabla g_S\left(x_t\right) \right\|^2+8L_f^2\eta_t^2\|\nabla g_{\omega_{j_t}}(x_t^{l,\omega})-\nabla g_S(x_t^{l,\omega}) \|^2\\
             &+4\eta_t^2 C_f^2L_g^2\| y_{t+1}-g_S(x_t)\|^2+4\eta_t^2 C_f^2L_g^2\|g_S(x_t^{l,\omega})-y_{t+1}^{l,\omega}\|^2\\
             &+4\eta_tL_gL_f\|x_{t}-x_{t}^{l,\omega}\|\mathbb{I}_{[j_t=l]}+4\eta_t^2L_g^2 L_f^2\mathbb{I}_{[j_t= l]} .\numberthis \label{comjt}
 \end{align*}
 Taking the expectation w.r.t. $A$ on both sides of \eqref{comjt} yields that 
      \begin{align*}
      &\EX_{A}\bigl[\|x_{t+1}-x_{t+1}^{l,\omega}\|^2\bigr]\\
        &\leq \EX_{A}\bigl[\|x_{t}-x_{t}^{l,\omega}\|^2\bigr]+2C_fL_g\eta_t\EX_{A}\bigl[\|y_{t+1}-g_S(x_t)\|\|x_{t}-x_{t}^{l,\omega}\|\bigr]
    \\ & +2C_fL_g\eta_t\EX_{A}\bigl[\|y_{t+1}^{l,\omega}-g_S(x_t^{l,\omega})\|\|x_{t}-x_{t}^{l,\omega}\|\bigr]\\
        & -2\eta_t\EX_{A}\bigl[\langle \nabla g_{\omega_{j_t}}\left(x_t\right) \nabla f_{\nu_{i_ t}}\left(g_S\left(x_t\right)\right)-\nabla g_S\left(x_t\right) \nabla f_{\nu_{i_ t}}\left(g_S\left(x_t\right)\right),x_t-x_t^{l,\omega}\rangle\mathbb{I}_{[j_t\neq l]}\bigr]\\
            &-2\eta_t\EX_{A}\bigl[\langle\nabla g_S(x_t^{l,\omega}) \nabla f_{\nu_{i_ t}}(g_S(x_t^{l,\omega}))-\nabla g_{\omega_{j_t}}(x_t^{l,\omega}) \nabla f_{\nu_{i_t}}(g_S(x_t^{l,\omega})),x_t-x_t^{l,\omega}  \rangle\mathbb{I}_{[j_t\neq l]}\bigr]\\
            &+8L_f^2\eta_t^2\EX_{A}\bigl[\|\nabla g_{\omega_{j_t}}(x_t)-\nabla g_S(x_t) \|^2]+8L_f^2\eta_t^2\EX_{A}\bigl[\|\nabla g_{\omega_{j_t}}(x_t^{l,\omega})-\nabla g_S(x_t^{l,\omega}) \|^2]\\
             &+4\eta_t^2 C_f^2L_g^2\EX_{A}\bigl[\| y_{t+1}-g_S(x_t)\|^2\bigr]+4\eta_t^2 C_f^2L_g^2\EX_{A}\bigl[\|g_S(x_t^{l,\omega})-y_{t+1}^{l,\omega}\|^2\bigr]\\
        &+4\eta_tL_fL_g\EX_{A}\bigl[\|x_{t}-x_{t}^{l,\omega}\|\mathbb{I}_{[j_t=l]}\bigr]+4\eta_t^2L_g^2 L_f^2\EX_{A}\bigl[\mathbb{I}_{[j_t= l]} \bigr].\numberthis \label{expejt}
 \end{align*}
We will estimate the terms on the right hand side of the above inequality.  To this end, denote 
  \begin{align*}
      T_1:&=\langle \nabla g_{\omega_{j_t}}\left(x_t\right) \nabla f_{\nu_{i_ t}}\left(g_S\left(x_t\right)\right)-\nabla g_S\left(x_t\right) \nabla f_{\nu_{i_ t}}\left(g_S\left(x_t\right)\right),x_t-x_t^{l,\omega}\rangle,\\
      T_2:&=\langle\nabla g_S(x_t^{l,\omega}) \nabla f_{\nu_{i_ t}}(g_S(x_t^{l,\omega}))-\nabla g_{\omega_{j_t}}(x_t^{l,\omega}) \nabla f_{\nu_{i_t}}(g_S(x_t^{l,\omega})),x_t-x_t^{l,\omega}  \rangle.
  \end{align*} 
  Taking the expectation w.r.t. $A$ on  both sides of the above identity, we have 
  \begin{align*}
      \EX_A[T_1]&=\EX_A[\langle \nabla g_{\omega_{j_t}}\left(x_t\right) \nabla f_{\nu_{i_ t}}\left(g_S\left(x_t\right)\right)-\nabla g_S\left(x_t\right) \nabla f_{\nu_{i_ t}}\left(g_S\left(x_t\right)\right),x_t-x_t^{l,\omega}\rangle]\\
      &=\EX_A[\langle \EX_{j_t}[\nabla g_{\omega_{j_t}}\left(x_t\right) \nabla f_{\nu_{i_ t}}\left(g_S\left(x_t\right)\right)]-\nabla g_S\left(x_t\right) \nabla f_{\nu_{i_ t}}\left(g_S\left(x_t\right)\right),x_t-x_t^{l,\omega}\rangle]\\
      &=\EX_A[\langle \nabla g_S\left(x_t\right) \nabla f_{\nu_{i_ t}}\left(g_S\left(x_t\right)\right)-\nabla g_S\left(x_t\right) \nabla f_{\nu_{i_ t}}\left(g_S\left(x_t\right)\right),x_t-x_t^{l,\omega}\rangle]\\
      &=0\numberthis\label{T_1},
  \end{align*}
  where the second identity holds true since $j_t$ is independent of $i_t$ and $x_t$. 
  Therefore,  
  \begin{align*}
      -2\eta_t\EX_A[T_1\mathbb{I}_{[j_t\neq l]} ]&= -2\eta_t\EX_A[T_1\mathbb{I}_{[j_t\neq l]} ]+ 2\eta_t\EX_A[T_1\mathbb{I}_{[j_t= l]} ] -2\eta_t\EX_A[T_1\mathbb{I}_{[j_t= l]} ]\\
      &=(-2\eta_t\EX_A[T_1\mathbb{I}_{[j_t\neq l]} ]-2\eta_t\EX_A[T_1\mathbb{I}_{[j_t= l]}] )+2\eta_t\EX_A[T_1\mathbb{I}_{[j_t= l]} ]\\
      &=-2\eta_t\EX_A[T_1 ]+2\eta_t\EX_A[T_1\mathbb{I}_{[j_t= l]} ]\\
      &=2\eta_t\EX_A[T_1\mathbb{I}_{[j_t= l]} ].\numberthis \label{E_T1}
  \end{align*}

  We further get the following estimation
  \begin{align*}
       &-2\eta_t\EX_{A}\bigl[\langle \nabla g_{\omega_{j_t}}\left(x_t\right) \nabla f_{\nu_{i_ t}}\left(g_S\left(x_t\right)\right)-\nabla g_S\left(x_t\right) \nabla f_{\nu_{i_ t}}\left(g_S\left(x_t\right)\right),x_t-x_t^{l,\omega}\rangle\mathbb{I}_{[j_t\neq l]}\bigr]\\
       &=2\eta_t\EX_{A}\bigl[\langle \nabla g_{\omega_{j_t}}\left(x_t\right) \nabla f_{\nu_{i_ t}}\left(g_S\left(x_t\right)\right)-\nabla g_S\left(x_t\right) \nabla f_{\nu_{i_ t}}\left(g_S\left(x_t\right)\right),x_t-x_t^{l,\omega}\rangle\mathbb{I}_{[j_t= l]}\bigr]\\
         &\leq 2\eta_t\EX_{A}\bigl[\| \nabla g_{\omega_{j_t}}\left(x_t\right) \nabla f_{\nu_{i_ t}}\left(g_S\left(x_t\right)\right)-\nabla g_S\left(x_t\right) \nabla f_{\nu_{i_ t}}\left(g_S\left(x_t\right)\right)\|\|x_t-x_t^{l,\omega}\|\mathbb{I}_{[j_t= l]}\bigr]\\
     &\leq2\eta_t\EX_{A}\bigl[(\| \nabla g_{\omega_{j_t}}\left(x_t\right) \|\|\nabla f_{\nu_{i_ t}}\left(g_S\left(x_t\right)\right)\|+\|\nabla g_S\left(x_t\right)\|\| \nabla f_{\nu_{i_ t}}\left(g_S\left(x_t\right)\right)\|)\|x_t-x_t^{l,\omega}\|\mathbb{I}_{[j_t= l]}\bigr]\\
     &\leq 4\eta_t L_g L_f\EX_{A}\bigl[\|x_t-x_t^{l,\omega}\|\mathbb{I}_{[j_t= l]}\bigr], \numberthis \label{forthE_A}
  \end{align*}
  where the last inequality holds true due to  Assumption \ref{assum:1}. Similar to estimations of \eqref{T_1} , \eqref{E_T1} and \eqref{forthE_A}, one can show that  
  \begin{align*}
   &-2\eta_t\EX_A[T_2\mathbb{I}_{[j_t\neq l]} ]\\ &=-2\eta_t\EX_A[\langle\nabla g_S(x_t^{l,\omega}) \nabla f_{\nu_{i_ t}}(g_S(x_t^{l,\omega}))-\nabla g_{\omega_{j_t}}(x_t^{l,\omega}) \nabla f_{\nu_{i_t}}(g_S(x_t^{l,\omega})),x_t-x_t^{l,\omega}  \rangle\mathbb{I}_{[j_t\neq l]}]\\
      &=2\eta_t\EX_{A}\bigl[\langle \nabla g_S(x_t^{l,\omega}) \nabla f_{\nu_{i_ t}}(g_S(x_t^{l,\omega}))-\nabla g_{\omega_{j_t}}(x_t^{l,\omega}) \nabla f_{\nu_{i_t}}(g_S(x_t^{l,\omega})),x_t-x_t^{l,\omega}\rangle\mathbb{I}_{[j_t= l]}\bigr]\\
      &\leq 4\eta_tL_g L_f\EX_{A}\bigl[\|x_t-x_t^{l,\omega}\|\mathbb{I}_{[j_t= l]}\bigr].\numberthis\label{fifthE_A}
  \end{align*}
Substituting \eqref{forthE_A}
 and \eqref{fifthE_A} into \eqref{expejt} and noting that $C_g$ represents the empirical variance associated with the gradient of the inner function as given in part \ref{part1b} of  Assumption \ref{assum:2}, we obtain 
  \begin{align*}
      \EX_{A}&\bigl[\|x_{t+1}-x_{t+1}^{l,\omega}\|^2\bigr]\\
        \leq &\EX_{A}\bigl[\|x_{t}-x_{t}^{l,\omega}\|^2\bigr]+2C_fL_g\eta_t\EX_{A}\bigl[\|y_{t+1}-g_S(x_t)\|\|x_{t}-x_{t}^{l,\omega}\|\bigr] \\ &+2C_fL_g\eta_t\EX_{A}\bigl[\|y_{t+1}^{l,\omega}-g_S(x_t^{l,\omega})\|\|x_{t}-x_{t}^{l,\omega}\|\bigr]\\
           &+4\eta_t^2 C_f^2L_g^2\EX_{A}\bigl[\|y_{t+1}-g_S\left(x_t\right)\|^2\bigr]+4\eta_t^2 C_f^2L_g^2\EX_{A}\bigl[\|y_{t+1}^{l,\omega}-g_S(x_t^{l,\omega})\|^2\bigr]\\
        &+16\eta_t^2L_f^2C_g+12\eta_tL_gL_f\EX_{A}\bigl[\|x_{t}-x_{t}^{l,\omega}\|\mathbb{I}_{[j_t=l]}\bigr]
        +4\eta_t^2L_g^2 L_f^2\EX_{A}\bigl[\mathbb{I}_{[j_t= l]} \bigr]\\
        \leq& \EX_{A}\bigl[\|x_{t}-x_{t}^{l,\omega}\|^2\bigr]+2C_fL_g\eta_t(\EX_{A}\bigl[\|y_{t+1}-g_S(x_t)\|^2\bigr])^{1/2}(\EX_{A}\bigl[\|x_{t}-x_{t}^{l,\omega}\|^2\bigr])^{1/2}\\
    &+2C_fL_g\eta_t(\EX_{A}\bigl[\|y_{t+1}^{l,\omega}-g_S(x_t^{l,\omega})\|^2\bigr])^{1/2}(\EX_{A}\bigl[\|x_{t}-x_{t}^{l,\omega}\|^2\bigr])^{1/2}\\
        &+4\eta_t^2 C_f^2L_g^2\EX_{A}\bigl[\|y_{t+1}-g_S\left(x_t\right)\|^2\bigr]+4\eta_t^2 C_f^2L_g^2\EX_{A}\bigl[\|y_{t+1}^{l,\omega}-g_S(x_t^{l,\omega})\|^2\bigr]\\
        &+16\eta_t^2L_f^2C_g+12\eta_t L_gL_f\EX_{A}\bigl[\|x_t-x_t^{l,\omega}\|\mathbb{I}_{[j_t= l]}\bigr]
        +4\eta_t^2L_g^2 L_f^2\EX_{A}\bigl[\mathbb{I}_{[j_t= l]} \bigr],\numberthis \label{aexpejt}     
 \end{align*}
 where the second inequality holds by the Cauchy-Schwarz inequality. Observe that 
 \begin{align*}
     \EX_{A}\bigl[\|x_t-x_t^{l,\omega}\|\mathbb{I}_{[j_t= l]}\bigr] & =\EX_{A}\bigl[\|x_t-x_t^{l,\omega}\|\EX_{j_t}[\mathbb{I}_{[j_t= l]}]\bigr]\\ 
     & =\frac{1}{m}\EX_{A}\bigl[\|x_t-x_t^{l,\omega}\|\bigr]\leq\frac{1}{m}(\EX_{A}\bigl[\|x_t-x_t^{l,\omega}\|^2\bigr])^{1/2}.
 \end{align*}
Note that $\|x_{0}-x_{0}^{l,\omega}\|^2=0$. Combining the above two estimations together implies that  
  \begin{align*}
      &\EX_{A}\bigl[\|x_{t+1}-x_{t+1}^{l,\omega}\|^2\bigr]\leq 2C_fL_g\sum_{i=1}^{t}\eta_i(\EX_{A}\bigl[\|y_{i+1}-g_S(x_i)\|^2\bigr])^{1/2}(\EX_{A}\bigl[\|x_{i}-x_{i}^{l,\omega}\|^2\bigr])^{1/2}\\
    &+2C_fL_g\sum_{i=1}^{t}\eta_i(\EX_{A}\bigl[\|y_{i+1}^{l,\omega}-g_S(x_i^{l,\omega})\|^2\bigr])^{1/2}(\EX_{A}\bigl[\|x_{i}-x_{i}^{l,\omega}\|^2\bigr])^{1/2}\\
        &+4\sum_{i=0}^{t}\eta_i^2 C_f^2L_g^2\EX_{A}\bigl[\|y_{i+1}-g_S\left(x_i\right)\|^2\bigr]+4\sum_{i=0}^{t}\eta_i^2 C_f^2L_g^2\EX_{A}\bigl[\|y_{i+1}^{l,\omega}-g_S\left(x_i^{l,\omega}\right)\|^2\bigr]\\
        &+16L_f^2C_g\sum_{i=0}^{t}\eta_i^2+\frac{12L_gL_f}{m}\sum_{i=1}^{t}\eta_i(\EX_{A}\bigl[\|x_{i}-x_{i}^{l,\omega}\|^2\bigr])^{1/2}
        +\frac{4L_g^2 L_f^2}{m}\sum_{i=0}^{t}\eta_i^2.  
 \end{align*}
Again, for notational convenience, let  $u_{t}=(\EX_{A}\bigl[\|x_{t}-x_{t}^{l,\omega}\|^2\bigr])^{1/2}$. The above estimation can be rewritten as 
 \begin{align*}\label{inter1-yy}
      u_{t}^2&\leq 2C_fL_g\sum_{i=1}^{t-1}\eta_i(\EX_{A}\bigl[\|y_{i+1}-g_S(x_i)\|^2\bigr])^{1/2}u_i+2C_fL_g\sum_{i=1}^{t-1}\eta_i(\EX_{A}\bigl[\|y_{i+1}^{l,\omega}-g_S(x_i^{l,\omega})\|^2\bigr])^{1/2}u_i\\
        &+4\sum_{i=0}^{t-1}\eta_i^2 C_f^2L_g^2\EX_{A}\bigl[\|y_{i+1}-g_S\left(x_i\right)\|^2\bigr]+4\sum_{i=0}^{t-1}\eta_i^2 C_f^2L_g^2\EX_{A}\bigl[\|y_{i+1}^{l,\omega}-g_S(x_i^{l,\omega})\|^2\bigr]\\
        &+16L_f^2C_g\sum_{i=0}^{t-1}\eta_i^2+\frac{12L_f L_g}{m}\sum_{i=1}^{t-1}\eta_i u_i
        +\frac{4L_g^2 L_f^2}{m}\sum_{i=0}^{t-1}\eta_i^2. \numberthis 
 \end{align*}
 We will use Lemma \ref{recursion lemma} to get the desired estimation.  For this purpose, define  
\begin{align*}
    S_{t}&=4\sum_{i=0}^{t-1}\eta_i^2 C_f^2L_g^2\EX_{A}\bigl[\|y_{i+1}-g_S\left(x_i\right)\|^2\bigr]+4\sum_{i=0}^{t-1}\eta_i^2 C_f^2L_g^2\EX_{A}\bigl[\|y_{i+1}^{l,\omega}-g_S(x_i^{l,\omega})\|^2\bigr]\\
           &+16L_f^2C_g\sum_{i=0}^{t-1}\eta_i^2+\frac{4L_g^2 L_f^2}{m}\sum_{i=0}^{t-1}\eta_i^2,\\
 \alpha_i&=2C_fL_g\eta_i(\EX_{A}\bigl[\|y_{i+1}-g_S(x_i)\|^2\bigr])^{1/2}+2C_fL_g\eta_i(\EX_{A}\bigl[\|y_{i+1}^{l,\omega}-g_S(x_i^{l,\omega})\|^2\bigr])^{1/2}+\frac{12L_gL_f}{m}\eta_i. 
\end{align*}
Now applying Lemma \ref{recursion lemma} with $u_{t}$, $S_{t}$ and $\alpha_i$ define as above to \eqref{inter1-yy}, we get
\begin{align*}
    u_{t}&\leq\sqrt{S_{t}}+\sum_{i=1}^{t-1}\alpha_i\\
    &\leq 2C_fL_g\sum_{i=1}^{t-1}\eta_i(\EX_{A}\bigl[\|y_{i+1}-g_S(x_i)\|^2\bigr])^{1/2}+2C_fL_g\sum_{i=1}^{t-1}\eta_i(\EX_{A}\bigl[\|y_{i+1}^{l,\omega}-g_S(x_i^{l,\omega})\|^2\bigr])^{1/2}\\
        &+(4\sum_{i=0}^{t-1}\eta_i^2 C_f^2L_g^2\EX_{A}\bigl[\|y_{i+1}-g_S\left(x_i\right)\|^2\bigr])^{1/2}+(4\sum_{i=0}^{t-1}\eta_i^2 C_f^2L_g^2\EX_{A}\bigl[\|y_{i+1}^{l,\omega}-g_S(x_i^{l,\omega})\|^2\bigr])^{1/2}\\
        &+(16L_f^2 C_g\sum_{i=0}^{t-1}\eta_i^2)^{1/2}+\frac{12L_fL_g}{m}\sum_{i=1}^{t-1}\eta_i
        +(\frac{4L_g L_f}{m}\sum_{i=0}^{t-1}\eta_i^2)^{1/2}\\
        &\leq 
       4C_fL_g\sum_{i=0}^{t-1}\eta_j(\EX_{A}\bigl[\|y_{i+1}-g_S(x_i)\|^2\bigr])^{1/2}
        +4C_fL_g\sum_{i=0}^{t-1}\eta_i(\EX_{A}\bigl[\|y_{i+1}^{l,\omega}-g_S(x_i^{l,\omega})\|^2\bigr])^{1/2}\\
        &+(16L_f^2C_g\sum_{i=0}^{t-1}\eta_i^2)^{1/2}+\frac{12L_gL_f}{m}\sum_{i=0}^{t-1}\eta_i
        +(\frac{4L_g^2 L_f^2}{m}\sum_{i=0}^{t-1}\eta_i^2)^{1/2},\numberthis \label{aeq:stability_delta}
\end{align*}
where the second inequality uses the fact that $(\sum_{i=1}^{4}a_i)^{1/2}\leq \sum_{i=1}^{4}(a_i)^{1/2}$ and the last inequality holds by the fact that $ 
    \bigl(4C_f^2L_g^2 \sum_{i=0}^{t-1}\eta_i^2  \EX_{A}\bigl[\left\|y_{i+1}-g_S\left(x_i\right)\right\|^2\bigr]\bigr)^{1/2}\leq2C_fL_g\sum_{i=0}^{t-1}\eta_i(\EX_{A}\bigl[\|y_{i+1}-g_S(x_i)\|^2\bigr])^{1/2}$  and $ 
    (4\sum_{i=0}^{t-1}\eta_i^2 C_f^2L_g^2\EX_{A}\bigl[\|y_{i+1}^{l,\omega}-g_S(x_i^{l,\omega})\|^2\bigr])^{1/2}\leq 2C_fL_g\sum_{i=0}^{t-1}\eta_i(\EX_{A}\bigl[\|y_{i+1}^{l,\omega}-g_S(x_i^{l,\omega})\|^2\bigr])^{1/2}.$ 
    
If $\eta_i =\eta$,  note that $\eta\sum_{i=0}^{T-1} (\mathbb{E}_{A}[\|y_{j+1}-g_S(x_i)\|^2])^{1/2}\leq\sup_{S} \eta\sum_{i=0}^{T-1} (\mathbb{E}_{A}[\|y_{i+1}-g_S(x_i)\|^2])^{1/2}$ and $\eta\sum_{i=0}^{T-1} (\mathbb{E}_{A}[\|y_{i+1}^{l,\omega}-g_S(x_i^{l,\omega})\|^2])^{1/2}\leq \sup_{S} \eta\sum_{i=0}^{T-1} (\mathbb{E}_{A}[\|y_{i+1}-g_S(x_i)\|^2])^{1/2}$. Consequently, with $T$ iterations, we further obtain that
\begin{align*}
  u_T\leq& 8C_fL_g\sup_{S} \eta\sum_{i=0}^{T-1} (\mathbb{E}_{A}[\|y_{i+1}-g_S(x_i)\|^2])^{1/2}+(16L_f^2C_g\sum_{i=0}^{T-1}\eta^2)^{1/2}+\frac{12L_fL_g}{m}\sum_{i=0}^{T-1}\eta\\
        &+(\frac{4L_g^2 L_f^2}{m}\sum_{i=0}^{T-1}\eta^2)^{1/2}\\
        \leq&8C_fL_g\sup_{S} \eta\sum_{i=0}^{T-1} (\mathbb{E}_{A}[\|y_{i+1}-g_S(x_i)\|^2])^{1/2}+4L_f\sqrt{C_g}\eta\sqrt{T}+\frac{14L_gL_f}{m}\eta T. 
\end{align*}
where the last inequality holds by the fact that 
$\bigl(\frac{4L_f^2L_g^2}{m}\sum_{i=0}^{T-1}\eta^2\bigr)^{1/2}= \frac{2L_fL_g}{\sqrt{m}}\eta\sqrt{T}\leq  \frac{2L_fL_g}{m}\eta T$ because often we have $T\ge m$. Noting that  
$\mathbb{E}_{A}\bigl[\|x_{T}-x_{T}^{l,\omega}\|\bigr]\leq u_T=(\mathbb{E}_{A}\bigl[\|x_{T}-x_{T}^{l,\omega}\|^2\bigr])^{1/2}$, we further get
\begin{align*}  \label{eq:stability_omega:3}
\mathbb{E}_{A}\bigl[\|x_{T}-x_{T}^{l,\omega}\|\bigr] 
   & \leq 8C_fL_g\sup_{S} \eta\sum_{i=0}^{T-1} (\mathbb{E}_{A}[\|y_{i+1}-g_S(x_i)\|^2])^{1/2}\\ &+4L_f\sqrt{C_g}\eta\sqrt{T}+\frac{14L_fL_g}{m}\eta T. 
   \numberthis
\end{align*}
Equivalently,  
\begin{align*}
\mathbb{E}_{A}\bigl[\|x_{T}-x_{T}^{l,\omega}\|\bigr] =  \mathcal{O}\Bigl(\frac{L_fL_g}{m} \eta T 
+L_f\sqrt{C_g}\eta\sqrt{T} + \sup_{S}C_fL_g\sum_{i=0}^{T-1} \eta \bigl(\mathbb{E}_{A}[\| y_{i+1}-g_S(x_i)\|^2 ] \bigr)^{1\over 2}\Bigr).
\end{align*}
Now we combine the above results for estimating
$\mathbb{E}_{A}\bigl[\|x_{T}-x_{T}^{k,\nu}\|\bigr]$ and $\mathbb{E}_{A}\bigl[\|x_{T}-x_{T}^{l,\omega}\|\bigr]$ and conclude that 
\begin{align}\label{convexorder}\displaystyle
    \epsilon_{\nu}+\epsilon_{\omega}=  \mathcal{O}\Bigl( \frac{L_fL_g}{n} \eta T  
+\frac{L_fL_g}{m} \eta T+L_f\sqrt{C_g}\eta\sqrt{T} +C_fL_g \sup_{S} \sum_{j=0}^{T-1} \eta \bigl(\mathbb{E}_{A}[\| y_{j+1}-g_S(x_j)\|^2 ] \bigr)^{1\over 2}\Bigr).
\end{align}
The proof is completed. 
      \end{proof}
Next we move on to the proof of Corollary \ref{cor:1}
\begin{proof}[Proof of Corollary \ref{cor:1}]
    Considering the constant step size $\eta_t=\eta$, and with the result of the SCGD update in Lemma \ref{Mengdi-lemma}, we have  
    \begin{align*}
        \epsilon_\nu+\epsilon_{\omega}&= \mathcal{O}(\eta T n^{-1}+\eta T m^{-1}+\eta T^{\frac{1}{2}}+\eta\sum_{j=1} ^{T-1}(j^{-c/2}\beta^{-c/2}+\eta/\beta+\beta^{1/2}))\\
        &=\mathcal{O}(\eta T n^{-1}+\eta T m^{-1}+\eta T^{\frac{1}{2}}+\eta T^{-c/2+1}\beta^{-c/2}+\eta^2 \beta^{-1}T+\eta\beta^{1/2}T).
    \end{align*}
    
    With the result of the SCSC update in Lemma \ref{Mengdi-lemma}, we have 
    \begin{align*}
        \epsilon_\nu+\epsilon_{\omega}&= \mathcal{O}(\eta T n^{-1}+\eta T m^{-1}+\eta T^{\frac{1}{2}}+\eta\sum_{j=1} ^{T-1}(j^{-c/2}\beta^{-c/2}+\eta\beta^{-\frac{1}{2}}+\beta^{1/2}))\\
        &=\mathcal{O}(\eta T n^{-1}+\eta T m^{-1}+\eta T^{\frac{1}{2}}+\eta T^{-c/2+1}\beta^{-c/2}+\eta^2 \beta^{-\frac{1}{2}}T+\eta\beta^{1/2}T).
    \end{align*}
    
\end{proof}
\subsection{Optimization}\label{pf_opt_cv}

\begin{lemma} \label{lem:opt_convex:1}
  Suppose Assumptions \ref{assum:1} and \ref{assum:2} \ref{part2} holds for the empirical risk $F_S$, By running Algorithm \ref{alg:1}, we have for any \(\gamma_t> 0\)
  \begin{align*}  \label{eq:opt_convex:1}
      \EX_A[\|x_{t+1}- x_*^S\|^2| \mathcal{F}_t] 
      \leq& \bigl(1+ \frac{C_fL_g^2\eta_t}{\gamma_t}\bigr)\|x_t- x_*^S\|^2+ L_f^2L_g^2\eta_t^2 - 2\eta_t(F_S(x_t)- F_S(x_*^S)) \\ & + \gamma_t C_f\eta_t\EX_A[\|g_S(x_t)- y_{t+1}\|^2| \mathcal{F}_t]. \numberthis 
  \end{align*}
  where \(\mathcal{F}_t\) is the \(\sigma\)-field generated by \(\{\omega_{j_0}, \ldots, \omega_{j_{t-1}}, \nu_{i_0}, \ldots, \nu_{i_{t- 1}}\}\).
\end{lemma}

The proof of Lemma \ref{lem:opt_convex:1} is deferred to the end of this subsection. Now we are ready to prove the convergence of Algorithm \ref{alg:1} for the convex case.
\begin{proof}[Proof of Theorem \ref{thm:opt_convex}]
  We first present the proof for the SCGD update. Taking the total expectation with respect to the internal randomness of \(A\) on both sides of \eqref{eq:opt_convex:1}, we get
    \begin{align*} \label{eq:opt_convex:2}
        \mathbb{E}_A[\|x_{t+1}- x_*^S\|^2]\leq& \mathbb{E}_A[\|x_t- x_*^S\|^2]+ L_f^2L_g^2\eta_t^2- 2\eta_t\mathbb{E}_A[F_S(x_t)- F_S(x_*^S)] \\
        &+ \gamma_t C_f\eta_t\mathbb{E}_A[\|g_S(x_t)- y_{t+1}\|^2]+ \frac{C_fL_g^2\eta_t}{\gamma_t}\mathbb{E}_A[\|x_t- x_*^S\|^2].
        \numberthis
    \end{align*}
  Setting \(\eta_t= \eta, \beta_t= \beta\) and $\gamma_t= \frac{\beta_t}{\eta_t}= \frac{\beta}{\eta}$, plugging Lemma \ref{Mengdi-lemma} into \eqref{eq:opt_convex:2}, we have
  \begin{align*}
    \mathbb{E}_A[\|x_{t+1}- x_*^S\|^2]\leq& \mathbb{E}_A[\|x_t- x_*^S\|^2]+ L_f^2L_g^2\eta^2- 2\eta \mathbb{E}_A[F_S(x_t)- F_S(x_*^S)] \\
    &+ C_f\beta\left( (\frac{c}{e})^cD_y (t\beta)^{-c} +L_g^3 L_f^2\frac{\eta^2}{\beta^2}+2V_g\beta\right)+ C_fL_g^2\mathbb{E}_A[\|x_t- x_*^S\|^2]\frac{\eta^2}{\beta}.
  \end{align*}
  Setting \(\eta= T^{-a}, \beta= T^{-b}\), telescoping the above inequality for $t= 1, \cdots, T$, and noting that \(\mathbb{E}_A[\|x_t- x_*^S\|^2]\) is bounded by \(D_x\), we get
  \begin{align*}   \label{eq:opt_convex:11}
    2 \eta\sum_{t= 1}^{T} \mathbb{E}_A[F_S(x_t)- F_S(x_*^S)]\leq& D_x+ L_f^2L_g^2 \eta^2 T+ (\frac{c}{e})^cC_fD_y \beta^{1- c} \sum_{t= 1}^T t^{-c} + 2C_fV_g \beta^2 T \\
    &+ C_fL_f^2L_g^3 \eta^2\beta^{-1} T+ C_fL_g^2D_x \eta^2\beta^{-1} T.
    \numberthis
  \end{align*}
  From the choice of $A(S)$ and the convexity of $F_S$, noting that \(\sum_{t= 1}^{T} t^{-z}= \mathcal{O}(T^{1- z})\) for \(z\in (0, 1)\cup (1, \infty)\) and \(\sum_{t= 1}^{T} t^{-1}= \mathcal{O}(\log T)\), as long as \(c\neq 1\) we get
  \begin{align*}
    &\mathbb{E}_A[F_S(A(S))- F_S(x_*^S)] \\
    =& \mathcal{O}\Bigl(D_x(\eta T)^{-1}+ L_f^2L_g^2\eta+ C_fD_y(\beta T)^{1-c}(\eta T)^{-1}+ C_fV_g\beta^2 \eta^{-1}+ C_fL_f^2L_g^3D_x\eta \beta^{-1}\Bigr).
  \end{align*}
  Then we get the desired result for the SCGD update. Next we present the proof for the SCSC update. Setting \(\eta_t= \eta, \beta_t= \beta\) and $\gamma_t= \frac{1}{\sqrt{\beta_t}}= \frac{1}{\sqrt{\beta}}$, plugging Lemma \ref{Mengdi-lemma} into \eqref{eq:opt_convex:2}, we have
  \begin{align*}
    \mathbb{E}_A[\|x_{t+1}- x_*^S\|^2]\leq& \mathbb{E}_A[\|x_t- x_*^S\|^2]+ L_f^2L_g^2\eta^2- 2\eta \mathbb{E}_A[F_S(x_t)- F_S(x_*^S)] \\
    &+ C_f\frac{\eta}{\sqrt{\beta}}\left( (\frac{c}{e})^cD_y (t\beta)^{-c}+L_g^3 L_f^2\frac{\eta^2}{\beta}+2V_g\beta\right)+ C_fL_g^2\mathbb{E}_A[\|x_t- x_*^S\|^2]\eta\sqrt{\beta}.
  \end{align*}
  Setting \(\eta= T^{-a}, \beta= T^{-b}\), telescoping the above inequality for $t= 1, \cdots, T$, and noting that \(\mathbb{E}_A[\|x_t- x_*^S\|^2]\) is bounded by \(D_x\), we get
  \begin{align*}   \label{eq:opt_convex:12}
    2 \eta\sum_{t= 1}^{T} \mathbb{E}_A[F_S(x_t)- F_S(x_*^S)]\leq& D_x+ L_f^2L_g^2 \eta^2 T+ (\frac{c}{e})^cC_fD_y \eta\beta^{-\frac{1}{2}- c} \sum_{t= 1}^T t^{-c} + 2C_fV_g \eta\beta^{\frac{1}{2}} T \\
    &+ C_fL_f^2L_g^3 \eta^3\beta^{-\frac{3}{2}} T+ C_fL_g^2D_x \eta\beta^{\frac{1}{2}} T.
    \numberthis
  \end{align*}
  From the choice of $A(S)$ and the convexity of $F_S$, noting that \(\sum_{t= 1}^{T} t^{-z}= \mathcal{O}(T^{1- z})\) for \(z\in (0, 1)\cup (1, \infty)\) and \(\sum_{t= 1}^{T} t^{-1}= \mathcal{O}(\log T)\), as long as \(c> 2\) we get
  \begin{align*}
    &\mathbb{E}_A[F_S(A(S))- F_S(x_*^S)] \\
    =& \mathcal{O}\Bigl(D_x(\eta T)^{-1}+ L_f^2L_g^2\eta+ C_fD_y(\beta T)^{-c}\beta^{-\frac{1}{2}}+ C_fV_g\beta^{\frac{1}{2}}+ C_fL_f^2L_g^3\eta^2 \beta^{-\frac{3}{2}}+ C_fL_g^2D_x \beta^{\frac{1}{2}}\Bigr).
  \end{align*}
We have completed the proof.
\end{proof}

\begin{proof}[Proof of Lemma \ref{lem:opt_convex:1}]
  From Algorithm \ref{alg:1} we have
    \begin{align*}
      &\|x_{t+1}- x_*^S\|^2 \\
      \leq& \|x_t- \eta_t\nabla g_{\omega_{j_t}}(x_t)\nabla f_{\nu_{i_t}}(y_{t+1})- x_*^S\|^2 \\
      =& \|x_t- x_*^S\|^2+ \eta_t^2\|\nabla g_{\omega_{j_t}}(x_t)\nabla f_{\nu_{i_t}}(y_{t+1})\|^2- 2\eta_t\langle x_t- x_*^S, \nabla g_{\omega_{j_t}}(x_t)\nabla f_{\nu_{i_t}}(y_{t+1})\rangle \\
      =& \|x_t- x_*^S\|^2+ \eta_t^2\|\nabla g_{\omega_{j_t}}(x_t)\nabla f_{\nu_{i_t}}(y_{t+1})\|^2- 2\eta_t\langle x_t- x_*^S, \nabla g_{\omega_{j_t}}(x_t)\nabla f_{\nu_{i_t}}(g_S(x_t))\rangle + u_t, \\
    \end{align*}
  where
  \begin{equation*}
    u_t:= 2\eta_t\langle x_t- x_*^S, \nabla g_{\omega_{j_t}}(x_t)\nabla f_{\nu_{i_t}}(g_S(x_t))- \nabla g_{\omega_{j_t}}(x_t)\nabla f_{\nu_{i_t}}(y_{t+1})\rangle.
  \end{equation*}
  Let \(\mathcal{F}_t\) be the \(\sigma\)-field generated by \(\{\omega_{j_0}, \ldots, \omega_{j_{t-1}}, \nu_{i_0}, \ldots, \nu_{i_{t- 1}}\}\). Taking the expectation with respect to the internal randomness of the algorithm and using Assumption \ref{assum:1}, we have
    \begin{align*} \label{eq:x_{t+1}-x_*2}
      &\EX_A[\|x_{t+1}- x_*^S\|^2| \mathcal{F}_t] \\
      \leq& \|x_t- x_*^S\|^2+ L_f^2L_g^2\eta_t^2- 2\eta_t\EX_A[\langle x_t- x_*^S, \nabla g_{\omega_{j_t}}(x_t)\nabla f_{\nu_{i_t}}(g_S(x_t))\rangle|\mathcal{F}_t]+ \EX_A[u_t| \mathcal{F}_t] \\
      =& \|x_t- x_*^S\|^2+ L_f^2L_g^2\eta_t^2- 2\eta_t\langle x_t- x_*^S, \nabla F_S(x_t)\rangle+ \EX_A[u_t| \mathcal{F}_t] \\
      \leq& \|x_t- x_*^S\|^2+ L_f^2L_g^2\eta_t^2- 2\eta_t(F_S(x_t)- F_S(x_*^S))+ \EX_A[u_t| \mathcal{F}_t], \\
      \numberthis
    \end{align*}
  where the last inequality comes from the convexity of $F_S$. From the Cauchy-Schwarz inequality, Young's inequality, Assumption \ref{assum:1} \ref{assum:1b} and \ref{assum:2} \ref{part2} we have,   for all $\gamma_t> 0$, that   
    \begin{align*}\label{eq:u_t2}
      &2\eta_t\langle x_t- x_*^S, \nabla g_{\omega_{j_t}}(x_t)\nabla f_{\nu_{i_t}}(g_S(x_t))- \nabla g_{\omega_{j_t}}(x_t)\nabla f_{\nu_{i_t}}(y_{t+1})\rangle \\
      \leq& 2\eta_t\|x_t- x_*^S\|\|\nabla g_{\omega_{j_t}}(x_t)\|\|\nabla f_{\nu_{i_t}}(g_S(x_t))- \nabla f_{\nu_{i_t}}(y_{t+1})\| \\
      \leq& 2C_f\eta_t\|x_t- x_*^S\|\|\nabla g_{\omega_{j_t}}(x_t)\|\|g_S(x_t)- y_{t+1}\| \\
      \leq& 2C_f\eta_t\Bigl(\frac{\|x_t- x_*^S\|^2\|\nabla g_{\omega_{j_t}}(x_t)\|^2}{2\gamma_t}+ \frac{\gamma_t}{2}\|g_S(x_t)- y_{t+1}\|^2\Bigr) \\
      \leq& \frac{C_fL_g^2\eta_t}{\gamma_t}\|x_t- x_*^S\|^2+ \gamma_t C_f\eta_t\|g_S(x_t)- y_{t+1}\|^2. \numberthis
    \end{align*}
  Substituting (\ref{eq:u_t2}) into (\ref{eq:x_{t+1}-x_*2}), we get
    \begin{align*}
      \EX_A[\|x_{t+1}- x_*^S\|^2| \mathcal{F}_t] 
      \leq& \Bigl(1+ \frac{C_fL_g^2\eta_t}{\gamma_t}\Bigr)\|x_t- x_*^S\|^2+ L_f^2L_g^2\eta_t^2- 2\eta_t(F_S(x_t)- F_S(x_*^S)) \\ & +  \gamma_t C_f\eta_t\EX_A[\|g_S(x_t)- y_{t+1}\|^2| \mathcal{F}_t]. \numberthis
    \end{align*}
The proof is completed.
\end{proof}

\subsection{Excess Generalization}\label{gene_cv}

\begin{proof}[Proof of Theorem \ref{thm:gen_convex}]
    We first present the proof for the SCGD update. Setting \(\eta_t= \eta, \beta_t= \beta\) for \(\eta, \beta> 0\), from \eqref{eq:stability_delta:3} and \eqref{eq:stability_omega:3} we get for all \(t\)
    \begin{align*}\label{eq:stability_delta:2}
       & \EX_A[\|x_{t}- x_{t}^{k, \nu}\|] 
       \leq 8C_fL_g\sup_{S} \eta\sum_{j=0}^{t-1} (\mathbb{E}_{A}[\|y_{j+1}-g_S(x_j)\|^2])^{1/2}
        +4L_f\sqrt{C_g}\eta\sqrt{t}+\frac{6L_fL_g}{n}\eta t.
        \numberthis 
    \end{align*}
    and 
    \begin{align*}\label{eq:stability_omega:2}
       &\EX_A[\|x_{t}- x_{t}^{l, \omega}\|]  
       \leq 8C_fL_g\sup_{S} \eta\sum_{i=0}^{t-1} (\mathbb{E}_{A}[\|y_{i+1}-g_S(x_i)\|^2])^{1/2}+4L_f\sqrt{C_g}\eta\sqrt{t}+\frac{14L_fL_g}{m}\eta t. \numberthis
    \end{align*}
    Plugging Lemma \ref{Mengdi-lemma} with SCGD update into \eqref{eq:stability_delta:2} and \eqref{eq:stability_omega:2}, then we have
    \begin{align*}
       \EX_A[\|x_{t}- x_{t}^{k, \nu}\|]\leq& 8C_fL_g \eta\sum_{j= 1}^{t- 1} \sqrt{\left(\frac{c}{e}\right)^cD_y (j\beta)^{-c}+  L_fL_g^2\frac{\eta^2}{\beta^2}+2V_g\beta} \\
       &+ 4L_f\sqrt{C_g} \eta \sqrt{t}+ \frac{6L_fL_g}{n} \eta t+ 8C_fL_gD_y\eta.
    \end{align*}
    and
    \begin{align*}
       \EX_A[\|x_{t}- x_{t}^{l, \omega}\|]\leq& 8C_fL_g \eta\sum_{j= 1}^{t- 1} \sqrt{\left(\frac{c}{e}\right)^cD_y (j\beta)^{-c}+  L_fL_g^2\frac{\eta^2}{\beta^2}+2V_g\beta} \\
       &+ 4L_f\sqrt{C_g} \eta \sqrt{t}+ \frac{14L_fL_g}{m} \eta t+ 8C_fL_gD_y\eta.
    \end{align*}
    From the fact that \(\sqrt{a+ b}\leq \sqrt{a}+ \sqrt{b}\) we get
    \begin{align*}
       \EX_A[\|x_{t}- x_{t}^{k, \nu}\|]\leq& 8C_fL_g\sqrt{\left(\frac{c}{e}\right)^cD_y} \eta\beta^{-\frac{c}{2}}\sum_{j= 1}^{t- 1} j^{-\frac{c}{2}}+ 8C_f\sqrt{L_f} L_g^2\frac{\eta^2}{\beta} t+ 8C_fL_g\sqrt{2V_g}\eta\sqrt{\beta} t \\
       &+ 4L_f\sqrt{C_g} \eta \sqrt{t}+ \frac{6L_fL_g}{n} \eta t+ 8C_fL_gD_y\eta.
    \end{align*}
    and
    \begin{align*}
       \EX_A[\|x_{t}- x_{t}^{l, \omega}\|]\leq& 8C_fL_g\sqrt{\left(\frac{c}{e}\right)^cD_y} \eta\beta^{-\frac{c}{2}}\sum_{j= 1}^{t- 1} j^{-\frac{c}{2}}+ 8C_f\sqrt{L_f} L_g^2\frac{\eta^2}{\beta} t+ 8C_fL_g\sqrt{2V_g}\eta\sqrt{\beta} t \\
       &+ 4L_g\sqrt{C_g} \eta \sqrt{t}+ \frac{14L_fL_g}{m} \eta t+ 8C_fL_gD_y\eta.
    \end{align*}
    Thus we get
    \begin{align*}
      &\EX_A[\|x_{t}- x_{t}^{k, \nu}\|]+ 4\EX_A[\|x_{t}- x_{t}^{l, \omega}\|] \\
      \leq& 40C_fL_g\sqrt{\left(\frac{c}{e}\right)^cD_y} \eta\beta^{-\frac{c}{2}}\sum_{j= 1}^t j^{-\frac{c}{2}}+ 40C_f\sqrt{L_f} L_g^2\frac{\eta^2}{\beta} t+ 40C_fL_g\sqrt{2V_g}\eta\sqrt{\beta} t \\
      &+ 20L_f\sqrt{C_g} \eta \sqrt{t}+ \frac{6L_fL_g}{n} \eta t+ \frac{56L_fL_g}{m} \eta t+ 40C_fL_gD_y \eta.
    \end{align*}
    Using Theorem \ref{thm:1}, we have
    \begin{align*} \label{eq:opt_convex:9}
        &\EX_{S, A} \left[F(x_t)- F_S(x_t)\right] \\
        \leq& 40C_fL_g\sqrt{\left(\frac{c}{e}\right)^cD_y}L_fL_g \eta\beta^{-\frac{c}{2}}\sum_{j= 1}^t j^{-\frac{c}{2}}+ 40C_f\sqrt{L_f} L_fL_g^3\frac{\eta^2}{\beta} t+ 40C_f\sqrt{2V_g}L_fL_g^2\eta\sqrt{\beta} t \\
        &+ 20\sqrt{C_g}L_f^2L_g \eta \sqrt{t}+ \frac{6L_fL_g}{n} \eta t+ \frac{56L_fL_g}{m} \eta t+ 40C_fL_gD_y \eta+ L_f\sqrt{\frac{\EX_{S, A}[\var_\omega(g_\omega(x_t))]}{m}}.
        \numberthis
    \end{align*}
    From \eqref{eq:opt_convex:11} we get
      \begin{align*}   \label{eq:opt_convex:10}
        \sum_{t= 1}^{T} \mathbb{E}_{S, A}[F_S(x_t)- F_S(x_*^S)]\leq& D_x \eta^{-1}+ L_fL_g  \eta T+ (\frac{c}{e})^cC_fD_y \eta^{-1}\beta^{1- c} \sum_{t= 1}^T t^{-c} + 2C_fV_g \eta^{-1}\beta^2 T \\
        &+ C_fL_fL_g^2 \eta\beta^{-1}T+ C_fL_gD_x \eta\beta^{-1}T.
        \numberthis
      \end{align*}
    Setting \(\eta= T^{-a}\) and \(\beta= T^{-b}\) in \eqref{eq:opt_convex:9} with \(a, b\in (0, 1]\) and telescoping from \(t= 1, \ldots, T\), then adding the result with \eqref{eq:opt_convex:10}, and using the fact $F_S(x_*^S)\leq F_S(x_*)$, we get
      \begin{align*}   \label{eq:opt_convex:3}
          &\sum_{t= 1}^{T} \mathbb{E}_{S, A}[F(x_t)- F(x_*)] \\
        \leq& 40C_fL_g\sqrt{\left(\frac{c}{e}\right)^cD_y}L_fL_g T^{-a+ \frac{bc}{2}} \sum_{t= 1}^T\sum_{j= 1}^t j^{-\frac{c}{2}}+ 40C_f\sqrt{L_f} L_fL_g^3 T^{b- 2a} \sum_{t= 1}^T t \\
        &+ 40C_f\sqrt{2V_g}L_fL_g^2 T^{-a- \frac{b}{2}} \sum_{t= 1}^T t+ 20\sqrt{C_g}L_f^2L_g T^{-a} \sum_{t= 1}^T \sqrt{t}+ \frac{6L_f^2L_g^2}{n} T^{-a} \sum_{t= 1}^T t \\
        &+ \frac{56L_f^2L_g^2}{m} T^{-a} \sum_{t= 1}^T t+ 40C_fL_gD_y T^{1- a}+ L_f\sum_{t= 1}^T\sqrt{\frac{\EX_{S, A}[\var_\omega(g_\omega(x_t))]}{m}} \\
        &+ D_xT^a+ L_fL_g T^{1- a}+ (\frac{c}{e})^cC_fD_y T^{-b(1- c)+ a} \sum_{t= 1}^T t^{-c} \\
        &+ 2C_fV_g T^{1- 2b+ a}+ C_fL_fL_g^2 T^{1+ b- a}+ C_fL_gD_x T^{1+ b- a}.
        \numberthis
      \end{align*}
    Noting that \(\sum_{t= 1}^{T} t^{-z}= \mathcal{O}(T^{1- z})\) for \(z\in (-1, 0)\cup (-\infty, -1)\) and \(\sum_{t= 1}^{T} t^{-1}= \mathcal{O}(\log T)\), we have
    \begin{align*}
      \sum_{t= 1}^T \sum_{j= 1}^T j^{-\frac{c}{2}}= \mathcal{O}\left(\sum_{t= 1}^T t^{1- \frac{c}{2}} (\log t)^{\mathbb{I}_{c= 2}}\right)= \mathcal{O}\left(T^{2- \frac{c}{2}} (\log T)^{\mathbb{I}_{c= 2}}\right).
    \end{align*}
    With the same derivation we can get the bounds on other terms on the right hand side of \eqref{eq:opt_convex:3}. Then we get
    \begin{align*}
        &\sum_{t= 1}^{T} \mathbb{E}_{S, A}[F(x_t)- F(x_*)] \\
        =& \mathcal{O}\left(T^{2- a- \frac{c(1- b)}{2}} (\log T)^{\mathbb{I}_{c= 2}}+ T^{2+ b- 2a}+ T^{2- a- \frac{b}{2}}+ T^{\frac{3}{2}- a}\right. \\
        &\left.+ n^{-1}T^{2- a}+ m^{-1}T^{2- a}+ T^{1- a}+ m^{-\frac{1}{2}}T+ T^a+ T^{1- a}+ T^{(1- b)(1- c)+ a} (\log T)^{\mathbb{I}_{c= 1}}\right. \\
        &\left.+ T^{1- 2b+ a}+ T^{1+ b- a}\right).
    \end{align*}
    Dividing both sides of \eqref{eq:opt_convex:3} with \(T\), then from the choice of \(A(S)\) we get
    \begin{align*}
        &\EX_{S,A}\Big[F(A(S)) - F(x_*) \Big] \\
        =& \mathcal{O}\left(T^{1- a- \frac{c(1- b)}{2}} (\log T)^{\mathbb{I}_{c= 2}}+ T^{1+ b- 2a}+ T^{1- a- \frac{b}{2}}+ T^{\frac{1}{2}- a}\right. \\
        &\left.+ n^{-1}T^{1- a}+ m^{-1}T^{1- a}+ T^{-a}+ m^{-\frac{1}{2}}+ T^{a- 1}+ T^{- a}+ T^{(1- b)(1- c)+ a- 1} (\log T)^{\mathbb{I}_{c= 1}}\right. \\
        &\left.+ T^{- 2b+ a}+ T^{b- a}\right).
    \end{align*}
    Since \(a, b\in (0, 1]\), as long as we have \(c> 2\), the dominating terms are
    \begin{equation*}
        \mathcal{O}(T^{1- a- \frac{b}{2}}), \quad \mathcal{O}(T^{1+ b- 2a}), \quad \mathcal{O}(n^{-1}T^{1- a}), \quad \mathcal{O}(m^{-1}T^{1- a}), \quad \mathcal{O}(T^{a-1}), \quad \mathcal{O}(T^{a-2b}).
    \end{equation*}
    Setting $a= \frac{6}{7}$ and $b= \frac{4}{7}$ yields
    \begin{equation*}
        \EX_{S,A}\Big[F(A(S)) - F(x_*) \Big]= \mathcal{O}(T^{-\frac{1}{7}}+ \frac{T^{\frac{1}{7}}}{n}+ \frac{T^{\frac{1}{7}}}{m}+ \frac{1}{\sqrt{m}}).
    \end{equation*}
    Setting $T= \mathcal{O}(\max\{n^{3.5}, m^{3.5}\})$ yields the following bound
    \begin{equation*}
        \EX_{S,A}\Big[F(A(S)) - F(x_*) \Big]= \mathcal{O}(\frac{1}{\sqrt{n}}+ \frac{1}{\sqrt{m}}).
    \end{equation*}
    Then we get the desired result for the SCGD update. Next we present the proof for the SCSC update. Plugging Lemma \ref{Mengdi-lemma} with SCSC update into \eqref{eq:stability_delta:2} and \eqref{eq:stability_omega:2}, then we have
    \begin{align*}
       \EX_A[\|x_{t}- x_{t}^{k, \nu}\|]\leq& 8C_fL_g \eta\sum_{j= 1}^{t- 1} \sqrt{\left(\frac{c}{e}\right)^cD_y (j\beta)^{-c}+  L_fL_g^2\frac{\eta^2}{\beta}+2V_g\beta} \\
       &+ 4L_f\sqrt{C_g} \eta \sqrt{t}+ \frac{6L_fL_g}{n} \eta t+ 8C_fL_gD_y\eta.
    \end{align*}
    and
    \begin{align*}
       \EX_A[\|x_{t}- x_{t}^{l, \omega}\|]\leq& 8C_fL_g \eta\sum_{j= 1}^{t- 1} \sqrt{\left(\frac{c}{e}\right)^cD_y (j\beta)^{-c}+  L_fL_g^2\frac{\eta^2}{\beta}+2V_g\beta} \\
       &+ 4L_f\sqrt{C_g} \eta \sqrt{t}+ \frac{14L_fL_g}{m} \eta t+ 8C_fL_gD_y\eta.
    \end{align*}
    From the fact that \(\sqrt{a+ b}\leq \sqrt{a}+ \sqrt{b}\) we get
    \begin{align*}
       \EX_A[\|x_{t}- x_{t}^{k, \nu}\|]\leq& 8C_fL_g\sqrt{\left(\frac{c}{e}\right)^cD_y} \eta\beta^{-\frac{c}{2}}\sum_{j= 1}^{t- 1} j^{-\frac{c}{2}}+ 8C_f\sqrt{L_f} L_g^2\frac{\eta^2}{\sqrt{\beta}} t+ 8C_fL_g\sqrt{2V_g}\eta\sqrt{\beta} t \\
       &+ 4L_f\sqrt{C_g} \eta \sqrt{t}+ \frac{6L_fL_g}{n} \eta t+ 8C_fL_gD_y\eta.
    \end{align*}
    and
    \begin{align*}
       \EX_A[\|x_{t}- x_{t}^{l, \omega}\|]\leq& 8C_fL_g\sqrt{\left(\frac{c}{e}\right)^cD_y} \eta\beta^{-\frac{c}{2}}\sum_{j= 1}^{t- 1} j^{-\frac{c}{2}}+ 8C_f\sqrt{L_f} L_g^2\frac{\eta^2}{\sqrt{\beta}} t+ 8C_fL_g\sqrt{2V_g}\eta\sqrt{\beta} t \\
       &+ 4L_g\sqrt{C_g} \eta \sqrt{t}+ \frac{14L_fL_g}{m} \eta t+ 8C_fL_gD_y\eta.
    \end{align*}
    Thus we get
    \begin{align*}
      &\EX_A[\|x_{t}- x_{t}^{k, \nu}\|]+ 4\EX_A[\|x_{t}- x_{t}^{l, \omega}\|] \\
      \leq& 40C_fL_g\sqrt{\left(\frac{c}{e}\right)^cD_y} \eta\beta^{-\frac{c}{2}}\sum_{j= 1}^t j^{-\frac{c}{2}}+ 40C_f\sqrt{L_f} L_g^2\frac{\eta^2}{\sqrt{\beta}} t+ 40C_fL_g\sqrt{2V_g}\eta\sqrt{\beta} t \\
      &+ 20L_f\sqrt{C_g} \eta \sqrt{t}+ \frac{6L_fL_g}{n} \eta t+ \frac{56L_fL_g}{m} \eta t+ 40C_fL_gD_y \eta.
    \end{align*}
    Using Theorem \ref{thm:1}, we have
    \begin{align*} \label{eq:opt_convex:13}
        &\EX_{S, A} \left[F(x_t)- F_S(x_t)\right] \\
        \leq& 40C_fL_g\sqrt{\left(\frac{c}{e}\right)^cD_y}L_fL_g \eta\beta^{-\frac{c}{2}}\sum_{j= 1}^t j^{-\frac{c}{2}}+ 40C_f\sqrt{L_f} L_fL_g^3\frac{\eta^2}{\sqrt{\beta}} t+ 40C_f\sqrt{2V_g}L_fL_g^2\eta\sqrt{\beta} t \\
        &+ 20\sqrt{C_g}L_f^2L_g \eta \sqrt{t}+ \frac{6L_fL_g}{n} \eta t+ \frac{56L_fL_g}{m} \eta t+ 40C_fL_gD_y \eta+ L_f\sqrt{\frac{\EX_{S, A}[\var_\omega(g_\omega(x_t))]}{m}}.
        \numberthis
    \end{align*}
    From \eqref{eq:opt_convex:12} we get
  \begin{align*}   \label{eq:opt_convex:14}
    \sum_{t= 1}^{T} \mathbb{E}_{S, A}[F_S(x_t)- F_S(x_*^S)]\leq& D_x \eta^{-1}+ L_fL_g \eta T+ (\frac{c}{e})^cC_fD_y \beta^{-\frac{1}{2}- c} \sum_{t= 1}^T t^{-c} + 2C_fV_g \beta^{\frac{1}{2}}T \\
    &+ C_fL_fL_g^2 \eta^2\beta^{-\frac{3}{2}} T+ C_fL_gD_x \beta^{\frac{1}{2}} T.
    \numberthis
  \end{align*}
    Setting \(\eta= T^{-a}\) and \(\beta= T^{-b}\) in \eqref{eq:opt_convex:9} with \(a, b\in (0, 1]\) and telescoping from \(t= 1, \ldots, T\), then adding the result with \eqref{eq:opt_convex:10}, and using the fact $F_S(x_*^S)\leq F_S(x_*)$, we get
  \begin{align*}   \label{eq:opt_convex:15}
      &\sum_{t= 1}^{T} \mathbb{E}_{S, A}[F(x_t)- F(x_*)] 
    \leq  40C_fL_g\sqrt{\left(\frac{c}{e}\right)^cD_y}L_fL_g T^{-a+ \frac{bc}{2}} \sum_{t= 1}^T\sum_{j= 1}^t j^{-\frac{c}{2}} \\ & + 40C_f\sqrt{L_f} L_fL_g^3 T^{\frac{b}{2}- 2a} \sum_{t= 1}^T t + 40C_f\sqrt{2V_g}L_fL_g^2 T^{-a- \frac{b}{2}} \sum_{t= 1}^T t \\ & + 20\sqrt{C_g}L_f^2L_g T^{-a} \sum_{t= 1}^T \sqrt{t} 
    + \frac{6L_f^2L_g^2}{n} T^{-a} \sum_{t= 1}^T t+ \frac{56L_f^2L_g^2}{m} T^{-a} \sum_{t= 1}^T t+ 40C_fL_gD_y T^{1- a} \\
    &+ L_f\sum_{t= 1}^T\sqrt{\frac{\EX_{S, A}[\var_\omega(g_\omega(x_t))]}{m}}+ D_xT^a+ L_fL_g T^{1- a}+ (\frac{c}{e})^cC_fD_y T^{b(\frac{1}{2}+ c)} \sum_{t= 1}^T t^{-c} \\
    &+ 2C_fV_g T^{1- \frac{b}{2}}+ C_fL_fL_g^2 T^{1+ \frac{3}{2}b- 2a}+ C_fL_gD_x T^{1- \frac{b}{2}}.
    \numberthis
  \end{align*}
    Noting that \(\sum_{t= 1}^{T} t^{-z}= \mathcal{O}(T^{1- z})\) for \(z\in (-1, 0)\cup (-\infty, -1)\) and \(\sum_{t= 1}^{T} t^{-1}= \mathcal{O}(\log T)\), we have
    \begin{align*}
      \sum_{t= 1}^T \sum_{j= 1}^T j^{-\frac{c}{2}}= \mathcal{O}\left(\sum_{t= 1}^T t^{1- \frac{c}{2}} (\log t)^{\mathbb{I}_{c= 2}}\right)= \mathcal{O}\left(T^{2- \frac{c}{2}} (\log T)^{\mathbb{I}_{c= 2}}\right).
    \end{align*}
    With the same derivation for estimating other terms on the right hand side of \eqref{eq:opt_convex:15}, we get
    \begin{align*}
        &\sum_{t= 1}^{T} \mathbb{E}_{S, A}[F(x_t)- F(x_*)] 
        = \mathcal{O}\left(T^{2- a- \frac{c(1- b)}{2}} (\log T)^{\mathbb{I}_{c= 2}}+ T^{2+ \frac{b}{2}- 2a}+ T^{2- a- \frac{b}{2}}+ T^{\frac{3}{2}- a}\right. \\
        &\left.+ n^{-1}T^{2- a}+ m^{-1}T^{2- a}+ T^{1- a}+ m^{-\frac{1}{2}}T+ T^a+ T^{1- a}+ T^{1- (1- b)c+ \frac{b}{2}} (\log T)^{\mathbb{I}_{c= 1}} \right. \\
        &\left.+ T^{1- \frac{b}{2}}+ T^{1+ \frac{3}{2}b- 2a}+ T^{1- \frac{b}{2}}\right).
    \end{align*}
    Dividing both sides of \eqref{eq:opt_convex:15} with \(T\), then from the choice of \(A(S)\) we get
    \begin{align*}
        &\EX_{S,A}\Big[F(A(S)) - F(x_*) \Big] 
        = \mathcal{O}\left(T^{1- a- \frac{c(1- b)}{2}} (\log T)^{\mathbb{I}_{c= 2}}+ T^{1+ \frac{b}{2}- 2a}+ T^{1- a- \frac{b}{2}}+ T^{\frac{1}{2}- a}\right. \\
        &\left.+ n^{-1}T^{1- a}+ m^{-1}T^{1- a}+ T^{-a}+ m^{-\frac{1}{2}}+ T^{a- 1}+ T^{- a}+ T^{- (1- b)c+ \frac{b}{2}} (\log T)^{\mathbb{I}_{c= 1}} \right. \\
        &\left.+ T^{- \frac{b}{2}}+ T^{\frac{3}{2}b- 2a}+ T^{- \frac{b}{2}}\right).
    \end{align*}
    Since \(a, b\in (0, 1]\), as long as we have \(c> 4\), the dominating terms are
        $ \mathcal{O}(T^{1- a- \frac{b}{2}}), \quad \mathcal{O}(T^{1+ \frac{b}{2}- 2a}), \quad \mathcal{O}(n^{-1}T^{1- a}), \quad \mathcal{O}(m^{-1}T^{1- a}), \quad  \mathcal{O}(T^{a-1})$,  and $ \mathcal{O}(T^{\frac{3}{2}b- 2a}).$ 
    Setting $a= b= \frac{4}{5}$ yields
    \begin{equation*}
        \EX_{S,A}\Big[F(A(S)) - F(x_*) \Big]= \mathcal{O}(T^{-\frac{1}{5}}+ \frac{T^{\frac{1}{5}}}{n}+ \frac{T^{\frac{1}{5}}}{m}+ \frac{1}{\sqrt{m}}).
    \end{equation*}
   Choosing  $T= \mathcal{O}(\max\{n^{2.5}, m^{2.5}\})$ yields the following bound
    \begin{equation*}
        \EX_{S,A}\Big[F(A(S)) - F(x_*) \Big]= \mathcal{O}(\frac{1}{\sqrt{n}}+ \frac{1}{\sqrt{m}}).
    \end{equation*}
Therefore,  we get the desired result for the SCSC update. The proof is complete. 
\end{proof}

%% file: appendix_strongly_convex.tex
\section{Proof for the Strongly Convex Setting} 

\subsection{Stability}\label{pf_strcon_stab}

\begin{proof}[Proof of Theorem \ref{thm:stab_sconvex}]
    The proof is analogous to the convex case.  For any $k\in [n]$, define $S^{k,\nu}=\{\nu_1,...,\nu_{k-1}, \nu_{k}^{\prime},\nu_{k+1},...,\nu_{n},\omega_1,...,\omega_{m}\}$ as formed from $S_\nu$ by replacing the $k$-th element. For any $l\in [m]$, define $S^{l,\omega}=\{\nu_1,...,\nu_n,\omega_1,...,\omega_{l-1}, \omega_{l}^{\prime},\omega_{l+1},...,\omega_{m}\}$ as formed from $S_\omega$ by replacing the $l$-th element.  Let $\{x_{t+1}\}$ and $\{y_{t+1}\}$ be produced by Algorithm \ref{alg:1} based on $S$,  $\{x_{t+1}^{k,\nu}\}$ and $\{y_{t+1}^{k,\nu}\}$ be produced by Algorithm \ref{alg:1} based on $S^{k,\nu}$,  $\{x_{t+1}^{l,\omega}\}$ and $\{y_{t+1}^{l,\omega}\}$ be produced by Algorithm \ref{alg:1} based on $S^{l,\omega}$. Let $x_0=x_0^{k,\nu}$ and $x_0=x_0^{l,\omega}$ be starting points in $\mathcal{X}$. Since changing one sample data can happen in either $S_\nu$ or $S_\omega$, we need to consider the $\mathbb{E}_{A}\bigl[\|x_{t+1}-x_{t+1}^{k,\nu}\|\bigr]$ and $\mathbb{E}_{A}\bigl[\|x_{t+1}-x_{t+1}^{l,\omega}\|\bigr]$. 
    
    \noindent{\bf Estimation of $\mathbb{E}_{A}\bigl[\|x_{t+1}-x_{t+1}^{k,\nu}\|\bigr]$}
    
   We begin with the estimation of the term $\mathbb{E}_{A}\bigl[\|x_{t+1}-x_{t+1}^{k,\nu}\|\bigr]$.  For this purpose, we will consider two cases, i.e., $i_t\neq k$ and $i_t=k$. 
   
    \smallskip 
    \textbf{\quad Case 1 ($i_t\neq k$). } 
    ~~If $i_t \neq k $, we have
    	\begin{align}
		\|x_{t+1}-x_{t+1}^{k,\nu}\|^2 &\leq \| x_t-\eta_t \nabla g_{\omega_{j_t}}(x_t) \nabla f_{\nu_{i_ t}}(y_{t+1})-x_t^{k,\nu}+\eta_t \nabla g_{\omega_{j_t}}(x_t^{k,\nu}) \nabla f_{\nu_{i_t}}(y_{t+1}^{k,\nu})\|^2\notag	\\
		&= \|x_t-x_t^{k,\nu}\|^2-2 \eta_t\langle \nabla g_{\omega_{j_t}}(x_t) \nabla f_{\nu_{i_ t}}(y_{t+1})-\nabla g_{\omega_{j_t}}(x_t^{k,\nu}) \nabla f_{\nu_{i_t}}(y_{t+1}^{k,\nu}),x_t-x_t^{k,\nu}\rangle\notag \\
		&+\eta_t^2\|\nabla g_{\omega_{j_t}}(x_t) \nabla f_{\nu_{i_ t}}(y_{t+1})-\nabla g_{\omega_{j_t}}(x_t^{k,\nu}) \nabla f_{\nu_{i_t}}(y_{t+1}^{k,\nu})\|^2. \label{str_mid}
	\end{align}
  Taking the expectation w.r.t. $j_t$ on the both sides of \eqref{str_mid} implies that 
  \begin{align*}
     & \EX_{j_t}\bigl[\|x_{t+1}-x_{t+1}^{k,\nu}\|^2 \bigl]\\
		&\leq \EX_{j_t}\bigl[\|x_t-x_t^{k,\nu}\|^2\bigr]-2 \eta_t\EX_{j_t}\bigl[\langle \nabla g_{\omega_{j_t}}(x_t) \nabla f_{\nu_{i_ t}}(y_{t+1})-\nabla g_{\omega_{j_t}}(x_t^{k,\nu}) \nabla f_{\nu_{i_t}}(y_{t+1}^{k,\nu}),x_t-x_t^{k,\nu}\rangle\bigr]\notag \\
		&+\eta_t^2\EX_{j_t}\bigl[\|\nabla g_{\omega_{j_t}}(x_t) \nabla f_{\nu_{i_ t}}(y_{t+1})-\nabla g_{\omega_{j_t}}(x_t^{k,\nu}) \nabla f_{\nu_{i_t}}(y_{t+1}^{k,\nu})\|\bigr].\numberthis \label{strexpectjtstab}
 \end{align*}
 We first estimate the second term on the right hand side of \eqref{strexpectjtstab}. It can be decomposed as
 \begin{align*}
		-&2\eta_t\EX_{j_t}\bigl[\langle \nabla g_{\omega_{j_t}}(x_t) \nabla f_{\nu_{i_ t}}(y_{t+1})-\nabla g_{\omega_{j_t}}(x_t^{k,\nu}) \nabla f_{\nu_{i_t}}(y_{t+1}^{k,\nu}),x_t-x_t^{k,\nu}\rangle \bigr] \\
		=&-2\eta_t\EX_{j_t}\bigl[\langle \nabla g_{\omega_{j_t}}(x_t) \nabla f_{\nu_{i_ t}}(y_{t+1})-\nabla g_{\omega_{j_t}}(x_t) \nabla f_{\nu_{i_ t}}(g_S(x_t)),x_t-x_t^{k,\nu}\rangle \bigr] \\
		&-2\eta_t\EX_{j_t}\bigl[\langle \nabla g_{\omega_{j_t}}(x_t) \nabla f_{\nu_{i_ t}}(g_S(x_t))-\nabla g_S(x_t) \nabla f_{\nu_{i_ t}}(g_S(x_t)),x_t-x_t^{k,\nu}\rangle \bigr] \\
		&-2\eta_t\EX_{j_t}\bigl[\langle\nabla g_S(x_t) \nabla f_{\nu_{i_ t}}(g_S(x_t))-\nabla g_S(x_t^{k,\nu}) \nabla f_{\nu_{i_t}}(g_S(x_t^{k,\nu})),x_t-x_t^{k,\nu} \rangle \bigr] \\
		&-2\eta_t\EX_{j_t}\bigl[\langle\nabla g_S(x_t^{k,\nu}) \nabla f_{\nu_{i_ t}}(g_S(x_t^{k,\nu}))-\nabla g_{\omega_{j_t}}(x_t^{k,\nu}) \nabla f_{\nu_{i_t}}(g_S(x_t^{k,\nu})),x_t-x_t^{k,\nu}  \rangle\bigr]\\
		&-2\eta_t\EX_{j_t}\bigl[\langle\nabla g_{\omega_{j_t}}(x_t^{k,\nu}) \nabla f_{\nu_{i_t}}(g_S(x_t^{k,\nu})) -\nabla g_{\omega_{j_t}}(x_t^{k,\nu}) \nabla f_{\nu_{i_t}}(y_{t+1}^{k,\nu}),x_t-x_t^{k,\nu} \rangle\bigr] .\numberthis \label{strexpandmid}
	\end{align*}
We will estimate the terms on the right hand side of the above equality.     Indeed, from part \ref{part3} of Assumption \ref{assum:2}, we know that $f_\nu(g_S(\cdot))$ is $L$-smooth. This combined with the strongly convexity of $f_\nu(g_S(\cdot))$ and inequality \eqref{sm+str} implied that 
    \begin{align*}
	&\langle \nabla g_S(x_t) \nabla f_{\nu_{i_ t}}(g_S(x_t))-\nabla g_S(x_t^{k,\nu}) \nabla f_{\nu_{i_t}}(g_S(x_t^{k,\nu}), x_t-x_t^{k,\nu}\rangle \\
	&\geq \frac{L\sigma}{L+\sigma}\|x_t-x_t^{k,\nu}\|^2+\frac{1}{L+\sigma}\|\nabla g_S(x_t) \nabla f_{\nu_{i_ t}}(g_S(x_t))-\nabla g_S(x_t^{k,\nu}) \nabla f_{\nu_{i_t}}(g_S(x_t^{k,\nu}))\|^2.\numberthis \label{corstrconvex}
    \end{align*}
Substituting \eqref{sec_forthmid}, \eqref{firstmid}, \eqref{lastmid} and \eqref{corstrconvex} into \eqref{strexpandmid}, we get that  
\begin{align*}\label{strfinal_midterm}
         &-2\eta_t\EX_{j_t}\bigl[\langle \nabla g_{\omega_{j_t}}\left(x_t\right) \nabla f_{\nu_{i_ t}}\left(y_{t+1}\right)-\nabla g_{\omega_{j_t}}(x_t^{k,\nu}) \nabla f_{\nu_{i_t}}(y_{t+1}^{k,\nu}),x_t-x_t^{k,\nu}\rangle\bigr]\\
         &\leq2C_fL_g\eta_t\EX_{j_t}\bigl[\|y_{t+1}-g_S(x_t)\|\bigr]\|x_t-x_t^{k,\nu}\|-\frac{2L\eta_t\sigma}{L+\sigma}\|x_t-x_t^{k,\nu}\|^2\\
            &-2\eta_t\frac{1}{L+\sigma}\|\nabla g_S\left(x_t\right) \nabla f_{\nu_{i_ t}}\left(g_S\left(x_t\right)\right)-\nabla g_S(x_t^{k,\nu}) \nabla f_{\nu_{i_t}}(g_S(x_{t}^{k,\nu}))\|^2\\
         &+2C_fL_g\eta_t\EX_{j_t}\bigl[\|y_{t+1}^{k,\nu}-g_S(x_t^{k,\nu})\|\bigr]\|x_t-x_t^{k,\nu}\|. \numberthis
     \end{align*}
    Furthermore, similar to the argument for \eqref{third_term}, we take the expectation w.r.t. $j_t$ of the third term on the right hand side of \eqref{str_mid} and then obtain that 
    		\begin{align*}
    	&\mathbb{E}_{{j_t}}[\eta_t^2\|\nabla g_{\omega_{j_t}}\left(x_t\right) \nabla f_{\nu_{i_ t}}\left(y_{t+1}\right)-\nabla g_{\omega_{j_t}}(x_t^{k,\nu}) \nabla f_{\nu_{i_t}}(y_{t+1}^{k,\nu})\|^2]\\
     &\leq4\eta_t^2 C_f^2L_g^2  \EX_{j_t}\bigl[\left\|y_{t+1}-g_S\left(x_t\right)\right\|^2\bigr]+4\eta_t^2 C_f^2L_g^2\EX_{j_t}\bigl[\|y_{t+1}^{k,\nu}-g_S(x_t^{k,\nu})\|^2\bigr]+16\eta_t^2L_f^2 C_g\\
        &+4\eta_t^2\|\nabla g_S\left(x_t\right) \nabla f_{\nu_{i_ t}}\left(g_S\left(x_t\right)\right)-\nabla g_S(x_t^{k,\nu}) \nabla f_{\nu_{i_t}}(g_S(x_t^{k,\nu}))\|^2. \numberthis\label{str_third_term}
    		\end{align*}
    Putting  \eqref{strfinal_midterm} and \eqref{str_third_term} back into \eqref{strexpectjtstab} implies that
 \begin{align*}
		&\mathbb{E}_{{j_t}}\bigl[\|x_{t+1}-x_{t+1}^{k,\nu}\|^2\bigr]\\
  &\leq (1-\frac{2L\sigma\eta_t}{L+\sigma}) \|x_t-x_t^{k,\nu}\|^2+2C_fL_g\eta_t\EX_{j_t}\bigl[\|y_{t+1}-g_S(x_t)\|\bigr]\|x_t-x_t^{k,\nu}\|\\
         &+2C_fL_g\eta_t\EX_{j_t}\bigl[\|y_{t+1}^{k,\nu}-g_S(x_t^{k,\nu})\|\bigr]\|x_t-x_t^{k,\nu}\|\\
         &+(4\eta_t^2-2\eta\frac{1}{L+\sigma})\|\nabla g_S\left(x_t\right) \nabla f_{\nu_{i_ t}}\left(g_S\left(x_t\right)\right)-\nabla g_S(x_t^{k,\nu}) \nabla f_{\nu_{i_t}}(g_S(x_t^{k,\nu}))\|^2\\
		&+4\eta_t^2 C_f^2L_g ^2 \EX_{j_t}\bigl[\left\|y_{t+1}-g_S\left(x_t\right)\right\|^2\bigr]+4\eta_t^2 C_f^2L_g^2\EX_{j_t}\bigl[\|y_{t+1}^{k,\nu}-g_S(x_t^{k,\nu})\|^2\bigr]+16\eta_t^2L_f ^2C_g\\
  &\leq (1-\frac{2L\sigma\eta_t}{L+\sigma}) \|x_t-x_t^{k,\nu}\|^2+2C_fL_g\eta_t\EX_{j_t}\bigl[\|y_{t+1}-g_S(x_t)\|\bigr]\|x_t-x_t^{k,\nu}\|\\
       &+2C_fL_g\eta_t\EX_{j_t}\bigl[\|y_{t+1}^{k,\nu}-g_S(x_t^{k,\nu})\|\bigr]\|x_t-x_t^{k,\nu}\|\\
		&+4\eta_t^2 C_f^2L_g^2  \EX_{j_t}\bigl[\left\|y_{t+1}-g_S\left(x_t\right)\right\|^2\bigr]+4\eta_t^2 C_f^2L_g^2\EX_{j_t}\bigl[\|y_{t+1}^{k,\nu}-g_S(x_t^{k,\nu})\|^2\bigr]+16\eta_t^2L_f^2 C_g,
	\end{align*}
where in the second inequality we have used the fact that  $\eta_t\leq\frac{1}{2\left(L+\sigma\right)}$. 

   \textbf{Case 2 ($i_t = k $). }      ~~If $i_t = k $,  in analogy to the argument in \eqref{conveitcase_2}, we have
 \begin{align*}
\EX_{j_t}\bigl[\|x_{t+1}-x_{t+1}^{k,\nu}\|^2\bigr]\leq\|x_t-x_t^{k,\nu}\|^2+4L_gL_f\eta_t\|x_t-x_t^{k,\nu}\|+4L_g^2L_f^2\eta_t^2.\numberthis\label{strconveitcase_2}
\end{align*}
    Combining the results of  {\bf Case 1} and {\bf Case 2} and taking the expectation w.r.t. $A$, we have that
    \begin{align*}
	&\mathbb{E}_{A}[\|x_{t+1}-x_{t+1}^{k,\nu}\|^2]\\
       &\leq\bigl(1-2\eta_t\frac{L\sigma}{L+\sigma}+\frac{2\eta_tL\sigma}{n(L+\sigma)}\bigr)\EX_A\bigl[\|x_t-x_t^{k,\nu}\|^2\bigr] \\ &+2C_fL_g\eta_t(\EX_{A}\bigl[\|y_{t+1}-g_S(x_t)\|^2\bigr])^{1/2}\bigl(\EX_A\bigl[\|x_t-x_t^{k,\nu}\|^2\bigr]\bigr)^{1/2}\\
         &+2C_fL_g\eta_t(\EX_{A}\bigl[\|y_{t+1}^{k,\nu}-g_S(x_t^{k,\nu})\|^2\bigr])^{1/2}\bigl(\EX_A\bigl[\|x_t-x_t^{k,\nu}\|^2\bigr]\bigr)^{1/2}\\
		&+4\eta_t^2 C_f^2L_g^2  \EX_{A}\bigl[\left\|y_{t+1}-g_S\left(x_t\right)\right\|^2\bigr]+4\eta_t^2 C_f^2L_g^2\EX_{A}\bigl[\|y_{t+1}^{k,\nu}-g_S(x_t^{k,\nu})\|^2\bigr]\\
        &+16\eta_t^2L_f^2 C_g+4\eta_tL_gL_f\EX_A\bigl[\|x_{t}-x_{t}^{k,\nu}\|\mathbb{ I } _{[i_t= k]}\bigr]
        +4\eta_t^2L_f^2 L_g^2\EX_A\bigl[\mathbb{ I } _{[i_t= k]}\bigr].\numberthis\label{strknu}
\end{align*}
  Note that  $\eta_t\frac{L\sigma}{L+\sigma}\ge \frac{2\eta_tL\sigma}{n(L+\sigma)}$ as $n \ge 2$.  We further  get that $1-2\eta_t\frac{L\sigma}{L+\sigma}+\frac{2\eta_tL\sigma}{n(L+\sigma)}\leq 1-\eta_t\frac{L\sigma}{L+\sigma}$. Observe that $
        \EX_A\bigl[\|x_{t}-x_{t}^{k,\nu}\|\mathbb{ I } _{[i_t= k]}\bigr]=\frac{1}{n}\EX_A\bigl[\|x_{t}-x_{t}^{k,\nu}\|\bigr]\leq\frac{1}{n}(\EX_A\bigl[\|x_{t}-x_{t}^{k,\nu}\|^2\bigr])^{1/2}.$ 
 If $\eta_t=\eta$, combining the above observations with \eqref{strknu} implies that
    \begin{align*}
     &\mathbb{E}_{A}[\|x_{t+1}-x_{t+1}^{k,\nu}\|^2]\\
     &\leq 
 2C_fL_g\sum_{j=1}^{t}(1-\eta\frac{L\sigma}{L+\sigma})^{t-j}\eta(\EX_{A}\bigl[\|y_{j+1}-g_S(x_j)\|^2\bigr])^{1/2}\bigl(\EX_A\bigl[\|x_j-x_j^{k,\nu}\|^2\bigr]\bigr)^{1/2}\\
         &+2C_fL_g\sum_{j=1}^{t}(1-\eta\frac{L\sigma}{L+\sigma})^{t-j}\eta(\EX_{A}\bigl[\|y_{j+1}^{k,\nu}-g_S(x_j^{k,\nu})\|^2\bigr])^{1/2}\bigl(\EX_A\bigl[\|x_j-x_j^{k,\nu}\|^2\bigr]\bigr)^{1/2}  \\
		&+4C_f^2L_g^2 \sum_{j=0}^{t}(1-\eta\frac{L\sigma}{L+\sigma})^{t-j}\eta^2  \EX_{A}\bigl[\|y_{j+1}-g_S(x_j)\|^2\bigr] \\ &+4C_f^2L_g ^2\sum_{j=0}^{t}(1-\eta\frac{L\sigma}{L+\sigma})^{t-j}\eta^2 \EX_{A}\bigl[\|y_{j+1}^{k,\nu}-g_S(x_j^{k,\nu})\|^2\bigr]  \\
        &+16L_f^2 C_g\sum_{j=0}^{t}(1-\eta\frac{L\sigma}{L+\sigma})^{t-j}\eta^2+\frac{4L_gL_f}{n}\sum_{j=1}^{t}(1-\eta\frac{L\sigma}{L+\sigma})^{t-j}\eta(\EX_A\bigl[\|x_{j}-x_{j}^{k,\nu}\|^2\bigr])^{1/2}\\
        &+\frac{4L_f^2 L_g^2}{n}\sum_{j=0}^{t}(1-\eta\frac{L\sigma}{L+\sigma})^{t-j}\eta^2 . \numberthis \label{recursionjtstconve}
    \end{align*}  
   Again, for notatioanl convenience, let $u_{t}=( \mathbb{E}_{A}[\|x_{t}-x_{t}^{k,\nu}\|^2])^{1/2}$. The above estimation can be equivalently rewritten as
  \begin{align*}
     &u_{t}^2\leq 
 2C_fL_g\sum_{j=1}^{t-1}(1-\eta\frac{L\sigma}{L+\sigma})^{t-j-1}\eta(\EX_{A}\bigl[\|y_{j+1}-g_S(x_j)\|^2\bigr])^{1/2}u_j\\
 &+2C_fL_g\sum_{j=1}^{t-1}(1-\eta\frac{L\sigma}{L+\sigma})^{t-j-1}\eta(\EX_{A}\bigl[\|y_{j+1}^{k,\nu}-g_S(x_j^{k,\nu})\|^2\bigr])^{1/2}u_j \\
		&+4C_f^2L_g^2 \sum_{j=0}^{t-1}(1-\eta\frac{L\sigma}{L+\sigma})^{t-j-1}\eta^2  \EX_{A}\bigl[\|y_{j+1}-g_S(x_j)\|^2\bigr] \\ & +4C_f^2L_g ^2\sum_{j=0}^{t-1}(1-\eta\frac{L\sigma}{L+\sigma})^{t-j-1}\eta^2\EX_{A}\bigl[\|y_{j+1}^{k,\nu}-g_S(x_j^{k,\nu})\|^2\bigr] +16L_f^2 C_g\eta^2\sum_{j=0}^{t-1}(1-\eta\frac{L\sigma}{L+\sigma})^{t-j-1} \\ &+\frac{4L_gL_f}{n}\eta\sum_{j=1}^{t-1}(1-\eta\frac{L\sigma}{L+\sigma})^{t-j-1}u_j +\frac{4L_f^2 L_g^2}{n}\eta^2\sum_{j=0}^{t-1}(1-\eta\frac{L\sigma}{L+\sigma})^{t-j-1}.
         \numberthis \label{brecursionjtstconve}
    \end{align*}
  Note that
  $ 
      16L_f^2 C_g\eta^2\sum_{j=0}^{t-1}(1-\eta\frac{L\sigma}{L+\sigma})^{t-j-1} \leq16L_f^2 C_g\eta^2\frac{L+\sigma}{L\eta\sigma}=16L_f^2 C_g\frac{{L+\sigma}}{L\sigma}\eta$ and $ 
      \frac{4L_f^2 L_g^2}{n}\eta^2\sum_{j=0}^{t-1}(1-\eta\frac{L\sigma}{L+\sigma})^{t-j-1}\leq \frac{4L_f^2 L_g^2}{n}\eta^2\frac{L+\sigma}{L\eta\sigma}=\frac{4L_f^2 L_g^2}{n}\frac{L+\sigma}{L\sigma}\eta.$ 
Furthermore, define 
  \begin{align*}
      S_{t}&=4C_f^2L_g^2 \sum_{j=0}^{t-1}(1-\eta\frac{L\sigma}{L+\sigma})^{t-j-1}\eta^2  \EX_{A}\bigl[\|y_{j+1}-g_S(x_j)\|^2\bigr]\\
      &+4C_f^2L_g ^2\sum_{j=0}^{t-1}(1-\eta\frac{L\sigma}{L+\sigma})^{t-j-1}\eta^2\EX_{A}\bigl[\|y_{j+1}^{k,\nu}-g_S(x_j^{k,\nu})\|^2\bigr]+16L_f^2 C_g\frac{{L+\sigma}}{L\sigma}\eta+\frac{4L_f^2 L_g^2}{n}\frac{L+\sigma}{L\sigma}\eta,\\
        \alpha_j&=2C_fL_g(1-\eta\frac{L\sigma}{L+\sigma})^{t-j-1}\eta(\EX_{A}\bigl[\|y_{j+1}-g_S(x_j)\|^2\bigr])^{1/2}\\
        &+2C_fL_g(1-\eta\frac{L\sigma}{L+\sigma})^{t-j-1}\eta(\EX_{A}\bigl[\|y_{j+1}^{k,\nu}-g_S(x_j^{k,\nu})\|^2\bigr])^{1/2}+\frac{4L_gL_f}{n}(1-\eta\frac{L\sigma}{L+\sigma})^{t-j}\eta . 
  \end{align*}
   Now applying Lemma \ref{recursion lemma} with $u_{t}$, $S_t$ and $\alpha_j$ defined above to \eqref{brecursionjtstconve}, we get
  \begin{align*}
      &u_{t}\leq\sqrt{S_{t}}+\sum_{j=1}^{t-1}\alpha_{j}\\
    &\leq2C_fL_g(\sum_{j=0}^{t-1}  (1-\eta\frac{L\sigma}{L+\sigma})^{t-j-1}\eta^2\EX_{A}\bigl[\left\|y_{j+1}-g_S\left(x_j\right)\right\|^2\bigr] )^{1/2}\\
    &+2C_fL_g(\sum_{j=0}^{t-1}(1-\eta\frac{L\sigma}{L+\sigma})^{t-j-1}\eta^2\EX_{A}\bigl[\|y_{j+1}^{k,\nu}-g_S(x_j^{k,\nu})\|^2\bigr])^{1/2}\\
    &+2C_fL_g\sum_{j=1}^{t-1}(1-\eta\frac{L\sigma}{L+\sigma})^{t-j-1}\eta(\EX_{A}\bigl[\|y_{j+1}-g_S(x_j)\|^2\bigr])^{1/2} \\
    &+2C_fL_g\sum_{j=1}^{t-1}(1-\eta\frac{L\sigma}{L+\sigma})^{t-j-1}\eta(\EX_{A}\bigl[\|y_{j+1}^{k,\nu}-g_S(x_j^{k,\nu})\|^2\bigr])^{1/2} +4L_f\sqrt{ C_g\frac{L+\sigma}{L \sigma}}\sqrt{\eta}\\
    &+2L_gL_f\sqrt{\frac{L+\sigma}{L\sigma}}\sqrt{\frac{\eta}{n}}
    +\frac{4L_gL_f(L+\sigma)}{nL\sigma}\\
   \end{align*}
  where the last inequality uses the fact that $(\sum_{i=1}^{4}a_i)^{1/2}\leq \sum_{i=1}^{4}(a_i)^{1/2}$ and we use the fact that $\frac{4L_gL_f}{n}\eta\sum_{j=1}^{t-1}(1-\eta\frac{L\sigma}{L+\sigma})^{t-j-1}\leq \frac{4L_gL_f(L+\sigma)}{nL\sigma} $. 
   Note that $\EX_{A}\bigl[\|y_{j+1}-g_S(x_j)\|^2\bigr]\leq \sup_{S}  \mathbb{E}_{A}[\|y_{j+1}-g_S(x_j)\|^2] $ and $\EX_{A}\bigl[\|y_{j+1}^{k,\nu}-g_S(x_j^{k,\nu})\|^2\bigr]\leq \sup_{S}  \mathbb{E}_{A}[\|y_{j+1}-g_S(x_j)\|^2]$. 
Consequently, with $T$ iterations, since $\mathbb{E}_{A}\bigl[\|x_{T}-x_{T}^{k,\nu}\|\bigr]\leq u_T=(\mathbb{E}_{A}\bigl[\|x_{T}-x_{T}^{k,\nu}\|^2\bigr])^{1/2}$, we further obtain 
\begin{align*}
    &\mathbb{E}_{A}\bigl[\|x_{T}-x_{T}^{k,\nu}\|\bigr]\leq  4C_fL_g\eta\sup_S(\sum_{j=0}^{T-1} (1-\eta\frac{L\sigma}{L+\sigma})^{T-j-1}\EX_{A}\bigl[\left\|y_{j+1}-g_S\left(x_j\right)\right\|^2\bigr] )^{1/2}\\
    &+4C_fL_g\eta\sup_{S}\sum_{j=0}^{T-1}(1-\eta\frac{L\sigma}{L+\sigma})^{T-j-1}\bigl(\mathbb{E}_{A}[\| y_{j+1}-g_S(x_j)\|^2 ] \bigr)^{1/ 2}+4L_f\sqrt{ C_g\frac{L+\sigma}{L \sigma}}\sqrt{\eta}\\
    &+2L_fL_g\sqrt{\frac{L+\sigma}{L\sigma}}\sqrt{\frac{\eta}{n}}
    +\frac{4L_gL_f(L+\sigma)}{ nL\sigma}.\numberthis \label{eq:stabilitysrtconv}
\end{align*}
  
    \noindent{\bf Estimation of  $\mathbb{E}_{A }[\|x_{t+1}-x_{t+1}^{l,\omega}\|]$}

    Likewise, we will estimate $\mathbb{E}_{A }[\|x_{t+1}-x_{t+1}^{l,\omega}\|]$ by considering two cases, i.e., $j_t\neq l$ and $j_t=l$.

     \textbf{~Case 1 ($j_t\neq l$).} ~~ If $j_t\neq l$, we have
       \begin{align*}
		&\|x_{t+1}-x_{t+1}^{l,\omega}\|^2 \leq \| x_t-\eta_t \nabla g_{\omega_{j_t}}(x_t) \nabla f_{\nu_{i_ t}}(y_{t+1})-x_t^{l,\omega}+\eta_t \nabla g_{\omega_{j_t}}(x_t^{l,\omega}) \nabla f_{\nu_{i_t}}(y_{t+1}^{l,\omega})\|^2	\\
		&= \|x_t-x_t^{l,\omega}\|^2-2 \eta_t\langle \nabla g_{\omega_{j_t}}\left(x_t\right) \nabla f_{\nu_{i_ t}}\left(y_{t+1}\right)-\nabla g_{\omega_{j_t}}(x_t^{l,\omega}) \nabla f_{\nu_{i_t}}(y_{t+1}^{l,\omega}),x_t-x_t^{l,\omega}\rangle \\
		&+\eta_t^2\|\nabla g_{\omega_{j_t}}\left(x_t\right) \nabla f_{\nu_{i_ t}}\left(y_{t+1}\right)-\nabla g_{\omega_{j_t}}(x_t^{l,\omega}) \nabla f_{\nu_{i_t}}(y_{t+1}^{l,\omega})\|^2 .\numberthis \label{strj_tneql}
	\end{align*}
 We first estimate the second term on the right hand side of \eqref{strj_tneql}. It can be decomposed as 
 \begin{align*}
		-&2\eta_t\langle \nabla g_{\omega_{j_t}}(x_t) \nabla f_{\nu_{i_ t}}(y_{t+1})-\nabla g_{\omega_{j_t}}(x_t^{l,\omega}) \nabla f_{\nu_{i_t}}(y_{t+1}^{l,\omega}),x_t-x_t^{l,\omega}\rangle \notag \\
  =& -2\eta_t  \langle \nabla g_{\omega_{j_t}}(x_t) (\nabla f_{\nu_{i_ t}}(y_{t+1})-\nabla f_{\nu_{i_ t}}(g_S(x_t))),x_t-x_t^{l,\omega}\rangle  \\
            &-2\eta_t\langle \nabla g_{\omega_{j_t}}(x_t) \nabla f_{\nu_{i_ t}}(g_S(x_t))-\nabla g_S(x_t) \nabla f_{\nu_{i_ t}}(g_S(x_t)),x_t-x_t^{l,\omega}\rangle\\
		&-2\eta_t\langle\nabla g_S(x_t) \nabla f_{\nu_{i_ t}}(g_S(x_t))-\nabla g_S(x_t^{l,\omega}) \nabla f_{\nu_{i_t}}(g_S(x_{t}^{l,\omega})), x_t-x_t^{l,\omega}\rangle\\
            &-2\eta_t\langle\nabla g_S(x_t^{l,\omega}) \nabla f_{\nu_{i_ t}}(g_S(x_t^{l,\omega}))-\nabla g_{\omega_{j_t}}(x_t^{l,\omega}) \nabla f_{\nu_{i_t}}(g_S(x_t^{l,\omega})),x_t-x_t^{l,\omega} \notag  \rangle\\
		&-2\eta_t\langle\nabla g_{\omega_{j_t}}(x_t^{l,\omega}) (\nabla f_{\nu_{i_ t}}(g_S(x_t^{l,\omega}))-\nabla f_{\nu_{i_t}}(y_{t+1}^{l,\omega})),x_t-x_t^{l,\omega}\rangle
  \numberthis \label{strjtmidterm}.
	\end{align*}
From the strongly convexity of $f_\nu(g_S(\cdot))$, part \ref{part3} of Assumption \ref{assum:2} and inequality \eqref{sm+str}, we have
 \begin{align}
	&\langle \nabla g_S\left(x_t\right) \nabla f_{\nu_{i_ t}}\left(g_S\left(x_t\right)\right)-\nabla g_S(x_t^{l,\omega}) \nabla f_{\nu_{i_t}}(g_S(x_t^{l,\omega})),x_t-x_t^{l,\omega}\rangle\notag \\
	&\geq \frac{L\sigma}{L+\sigma}\|x_t-x_t^{l,\omega}\|^2+\frac{1}{L+\sigma}\|\nabla g_S\left(x_t\right) \nabla f_{\nu_{i_ t}}\left(g_S\left(x_t\right)\right)-\nabla g_S(x_t^{l,\omega}) \nabla f_{\nu_{i_t}}(g_S(x_t^{l,\omega}))\|^2.\label{jt_str_corci}
    \end{align}
Plugging  \eqref{jtcoti_prpety}, \eqref{jtconti_b} and \eqref{jt_str_corci} into \eqref{strjtmidterm} implies that 
  \begin{align*}
		-&2\eta_t\langle \nabla g_{\omega_{j_t}}(x_t) \nabla f_{\nu_{i_ t}}(y_{t+1})-\nabla g_{\omega_{j_t}}(x_t^{l,\omega}) \nabla f_{\nu_{i_t}}(y_{t+1}^{l,\omega}),x_t-x_t^{l,\omega}\rangle \notag \\
  &\leq 2\eta_t C_fL_g\|y_{t+1}-g_S(x_t)\|\|x_t-x_t^{l,\omega}\|+ 2\eta_t C_fL_g\|y_{t+1}^{l,\omega}-g_S(x_t^{l,\omega})\|\|x_t-x_t^{l,\omega}\| \\
            &-2\eta_t\langle \nabla g_{\omega_{j_t}}(x_t) \nabla f_{\nu_{i_ t}}(g_S(x_t))-\nabla g_S(x_t) \nabla f_{\nu_{i_ t}}(g_S(x_t)),x_t-x_t^{l,\omega}\rangle\\
		&-2\eta_t\frac{L\sigma}{L+\sigma}\|x_t-x_t^{l,\omega}\|^2-2\eta_t\frac{1}{L}\|\nabla g_S(x_t) \nabla f_{\nu_{i_ t}}(g_S(x_t))-\nabla g_S(x_t^{l,\omega}) \nabla f_{\nu_{i_t}}(g_S(x_{t}^{l,\omega}))\|^2\\
            &-2\eta_t\langle\nabla g_S(x_t^{l,\omega}) \nabla f_{\nu_{i_ t}}(g_S(x_t^{l,\omega}))-\nabla g_{\omega_{j_t}}(x_t^{l,\omega}) \nabla f_{\nu_{i_t}}(g_S(x_t^{l,\omega})),x_t-x_t^{l,\omega} \notag  \rangle
	 \numberthis \label{strjtfnlmidterm}.
	\end{align*}
 Next we estimate the last term on the right hand side of \eqref{strj_tneql}. Using arguments similar to that for  \eqref{jtthir_term}, we have
 \begin{align*}
		&\eta_t^2\|\nabla g_{\omega_{j_t}}(x_t) \nabla f_{\nu_{i_ t}}(y_{t+1})-\nabla g_{\omega_{j_t}}(x_t^{l,\omega}) \nabla f_{\nu_{i_t}}(y_{t+1}^{l,\omega})\|^2\\
		&\leq 4\eta_t^2 C_f^2L_g^2\|y_{t+1}-g_S(x_t)\|^2+4L_g^2\eta_t^2 C_f^2\| g_S(x_t^{l,\omega})-y_{t+1}^{l,\omega}\|^2\\
		&+8L_f^2\eta_t^2\|\nabla g_{\omega_{j_t}}(x_t)-\nabla g_S(x_t) \|^2+8L_f^2\eta_t^2\|\nabla g_{\omega_{j_t}}(x_t^{l,\omega})-\nabla g_S(x_t^{l,\omega}) \|^2\\
		&+4\eta_t^2\|\nabla g_S(x_t) \nabla f_{\nu_{i_ t}}(g_S(x_t))-\nabla g_S(x_t^{l,\omega}) \nabla f_{\nu_{i_t}}(g_S(x_t^{l,\omega}))\|^2.\numberthis \label{strjtthir_term}
	\end{align*}
Putting \eqref{strjtfnlmidterm} and \eqref{strjtthir_term} into \eqref{strj_tneql} and noting that  $\eta_t\leq\frac{1}{2\left(L+\sigma\right)}$, we get 
    \begin{align*}
		&\|x_{t+1}-x_{t+1}^{l,\omega}\|^2\leq(1-\frac{2L\sigma\eta_t}{L+\sigma}) \|x_{t}-x_{t}^{l,\omega}\|^2 +2\eta_t C_fL_g\|y_{t+1}-g_S(x_t)\|\|x_t-x_t^{l,\omega}\| \\
            &+ 2\eta_t C_fL_g\|y_{t+1}^{l,\omega}-g_S(x_t^{l,\omega})\|\|x_t-x_t^{l,\omega}\|  \\ &  -2\eta_t\langle \nabla g_{\omega_{j_t}}\left(x_t\right) \nabla f_{\nu_{i_ t}}\left(g_S\left(x_t\right)\right)-\nabla g_S\left(x_t\right) \nabla f_{\nu_{i_ t}}\left(g_S\left(x_t\right)\right),x_t-x_t^{l,\omega}\rangle\\
            &-2\eta_t\langle\nabla g_S(x_t^{l,\omega}) \nabla f_{\nu_{i_ t}}(g_S(x_t^{l,\omega}))-\nabla g_{\omega_{j_t}}(x_t^{l,\omega}) \nabla f_{\nu_{i_t}}(g_S(x_t^{l,\omega})),x_t-x_t^{l,\omega} \notag  \rangle.\\
            &+4\eta_t^2 C_f^2L_g^2\left\| y_{t+1}-g_S\left(x_t\right)\right\|^2+4\eta_t^2 C_f^2L_g^2\|g_S(x_t^{l,\omega})-y_{t+1}^{l,\omega}\|^2\\
		&+8L_f^2\eta_t^2\|\nabla g_{\omega_{j_t}}(x_t)-\nabla g_S(x_t) \|^2+8L_f^2\eta_t^2\|\nabla g_{\omega_{j_t}}(x_t^{l,\omega})-\nabla g_S(x_t^{l,\omega}) \|^2.
		\numberthis \label{strj_tfnlneql}
	\end{align*}

  \smallskip 
     \textbf{~ Case 2 ($j_t= l$).}~~ If $j_t= l$, using the argument similar to \eqref{i_tequl}, it is easy to see that 
     \begin{align*}
 &\|x_{t+1}-x_{t+1}^{l,\omega}\|^2\leq \|x_{t}-x_{t}^{l,\omega}\|^2+4L_gL_f\eta_t\|x_{t}-x_{t}^{l,\omega}\|+4\eta_t^2L_g^2L_f^2.\numberthis \label{strj_tequl}
	\end{align*}
Combining \text{\bf Case 1} and \text{\bf Case 2}  and taking the expectation w.r.t. $A$ on both sides and together with part \ref{part1b} of Assumption \ref{assum:2} , we have 
 \begin{align*}
      &\EX_{A}\bigl[\|x_{t+1}-x_{t+1}^{l,\omega}\|^2\bigr]\\
        &\leq(1-\frac{2L\sigma\eta_t}{L+\sigma}+\frac{2\eta_tL\sigma}{m(L+\sigma)})\EX_{A}\bigl[\|x_{t}-x_{t}^{l,\omega}\|^2\bigr]+2C_fL_g\eta_t\EX_{A}\bigl[\|y_{t+1}-g_S(x_t)\|\|x_{t}-x_{t}^{l,\omega}\|\bigr]\\
    &+2C_fL_g\eta_t\EX_{A}\bigl[\|y_{t+1}^{l,\omega}-g_S(x_t^{l,\omega})\|\|x_{t}-x_{t}^{l,\omega}\|\bigr]+4\eta_t^2 C_f^2L_g^2\EX_{A}\bigl[\|y_{t+1}-g_S\left(x_t\right)\|^2\bigr]\\
    &+4\eta_t^2 C_f^2L_g^2\EX_{A}\bigl[\|y_{t+1}^{l,\omega}-g_S(x_t^{l,\omega})\|^2\bigr]+16C_gL_f^2\eta_t^2 \\
    &-2\eta_t\EX_{A}\bigl[\langle \nabla g_{\omega_{j_t}}\left(x_t\right) \nabla f_{\nu_{i_ t}}\left(g_S\left(x_t\right)\right)-\nabla g_S\left(x_t\right) \nabla f_{\nu_{i_ t}}\left(g_S\left(x_t\right)\right),x_t-x_t^{l,\omega}\rangle\mathbb{I}_{[j_t\neq l]}\bigr]\\
            &-2\eta_t\EX_{A}\bigl[\langle\nabla g_S(x_t^{l,\omega}) \nabla f_{\nu_{i_ t}}(g_S(x_t^{l,\omega}))-\nabla g_{\omega_{j_t}}(x_t^{l,\omega}) \nabla f_{\nu_{i_t}}(g_S(x_t^{l,\omega})),x_t-x_t^{l,\omega}  \rangle\mathbb{I}_{[j_t\neq l]}\bigr]\\
        &+4\eta_tL_fL_g\EX_{A}\bigl[\|x_{t}-x_{t}^{l,\omega}\|\mathbb{I}_{[j_t=l]}\bigr]+4\eta_t^2L_g^2L_f^2\EX_{A}\bigl[\mathbb{I}_{[j_t= l]} \bigr].\numberthis \label{strexpejt}
 \end{align*}
 Note that  $\eta_t\frac{L\sigma}{L+\sigma}\ge \frac{2\eta_tL\sigma}{m(L+\sigma)}$ as $m \ge 2$.  We further  get that $1-2\eta_t\frac{L\sigma}{L+\sigma}+\frac{2\eta_tL\sigma}{m(L+\sigma)}\leq 1-\eta_t\frac{L\sigma}{L+\sigma}$. Plugging   \eqref{forthE_A} and \eqref{fifthE_A} into \eqref{strexpejt} implies that 
\begin{align*}
     &\EX_{A}\bigl[\|x_{t+1}-x_{t+1}^{l,\omega}\|^2\bigr]\\
     &\leq(1-\frac{L\sigma\eta_t}{L+\sigma})\EX_{A}\bigl[\|x_{t}-x_{t}^{l,\omega}\|^2\bigr]+2C_fL_g\eta_t\EX_{A}\bigl[\|y_{t+1}-g_S(x_t)\|\|x_{t}-x_{t}^{l,\omega}\|\bigr]\\
    &+2C_fL_g\eta_t\EX_{A}\bigl[\|y_{t+1}^{l,\omega}-g_S(x_t^{l,\omega})\|\|x_{t}-x_{t}^{l,\omega}\|\bigr]+4\eta_t^2 C_f^2L_g^2\EX_{A}\bigl[\|y_{t+1}-g_S\left(x_t\right)\|^2\bigr]\\
    &+4\eta_t^2 C_f^2L_g^2\EX_{A}\bigl[\|y_{t+1}^{l,\omega}-g_S(x_t^{l,\omega})\|^2\bigr]+16C_gL_f^2\eta_t^2 \\      &+12\eta_tL_fL_g\EX_{A}\bigl[\|x_{t}-x_{t}^{l,\omega}\|\mathbb{I}_{[j_t=l]}\bigr]+4\eta_t^2L_g^2L_f^2\EX_{A}\bigl[\mathbb{I}_{[j_t= l]} \bigr]\\
      &\leq(1-\frac{L\sigma\eta_t}{L+\sigma}) \EX_{A}\bigl[\|x_{t}-x_{t}^{l,\omega}\|^2\bigr]+2C_fL_g\eta_t(\EX_{A}\bigl[\|y_{t+1}-g_S(x_t)\|^2\bigr])^{1/2}(\EX_{A}\bigl[\|x_{t}-x_{t}^{l,\omega}\|^2\bigr])^{1/2}\\
    &+2C_fL_g\eta_t(\EX_{A}\bigl[\|y_{t+1}^{l,\omega}-g_S(x_t^{l,\omega})\|^2\bigr])^{1/2}(\EX_{A}\bigl[\|x_{t}-x_{t}^{l,\omega}\|^2\bigr])^{1/2}\\
        &+4\eta_t^2 C_f^2L_g^2\EX_{A}\bigl[\|y_{t+1}-g_S\left(x_t\right)\|^2\bigr]+4\eta_t^2 C_f^2L_g^2\EX_{A}\bigl[\|y_{t+1}^{l,\omega}-g_S(x_t^{l,\omega})\|^2\bigr]+16C_g L_f^2\eta_t^2\\
        &+12\eta_tL_fL_g\EX_{A}\bigl[\|x_{t}-x_{t}^{l,\omega}\|\mathbb{I}_{[j_t=l]}\bigr]
        +4\eta_t^2L_g^2 L_f^2\EX_{A}\bigl[\mathbb{I}_{[j_t= l]} \bigr],
\end{align*}
where the second inequality holds by the Cauchy-Schwarz inequality.  
In addition, observe that   
\begin{align*}
    \EX_{A}\bigl[\|x_{t}-x_{t}^{l,\omega}\|\mathbb{I}_{[j_t=l]}\bigr]&=\EX_{A}\bigl[\|x_{t}-x_{t}^{l,\omega}\|\EX_{j_t}[\mathbb{I}_{[j_t=l]}]\bigr]\\&=\frac{1}{m}\EX_{A}\bigl[\|x_{t}-x_{t}^{l,\omega}\|\bigr]\leq\frac{1}{m}(\EX_{A}\bigl[\|x_{t}-x_{t}^{l,\omega}\|^2\bigr])^{1/2}. 
\end{align*}
If $\eta_t=\eta$, using the above observations, noting $\|x_{0}-x_{0}^{l,\omega}\|^2=0$, we can obtain 
  \begin{align*}
     &\mathbb{E}_{A}[\|x_{t+1}-x_{t+1}^{l,\omega}\|^2]\\
     &\leq 
 2C_fL_g\sum_{i=1}^{t}(1-\frac{L\sigma\eta}{L+\sigma})^{t-i}\eta(\EX_{A}\bigl[\|y_{i+1}-g_S(x_i)\|^2\bigr])^{1/2}\bigl(\EX_A\bigl[\|x_i-x_i^{l,\omega}\|^2\bigr]\bigr)^{1/2}\\
         &+2C_fL_g\sum_{i=1}^{t}(1-\frac{L\sigma\eta}{L+\sigma})^{t-i}\eta(\EX_{A}\bigl[\|y_{i+1}^{l,\omega}-g_S(x_i^{l,\omega})\|^2\bigr])^{1/2}\bigl(\EX_A\bigl[\|x_i-x_i^{l,\omega}\|^2\bigr]\bigr)^{1/2}  \\
		&+4C_f^2L_g^2 \sum_{i=0}^{t}(1-\frac{L\sigma\eta}{L+\sigma})^{t-i}\eta^2  \EX_{A}\bigl[\|y_{i+1}-g_S(x_i)\|^2\bigr]\\
  &+4C_f^2L_g^2 \sum_{i=0}^{t}(1-\frac{L\sigma\eta}{L+\sigma})^{t-i}\eta^2 \EX_{A}\bigl[\|y_{i+1}^{l,\omega}-g_S(x_i^{l,\omega})\|^2\bigr]  \\
        &+16L_f^2 C_g\sum_{i=0}^{t}(1-\frac{L\sigma\eta}{L+\sigma})^{t-i}\eta^2+\frac{12L_gL_f}{m}\sum_{i=1}^{t}(1-\frac{L\sigma\eta}{L+\sigma})^{t-i}\eta(\EX_A\bigl[\|x_{i}-x_{i}^{l,\omega}\|^2\bigr])^{1/2} \\
        &+\frac{4L_f^2 L_g^2}{m}\sum_{i=0}^{t}(1-\frac{L\sigma\eta}{L+\sigma})^{t-i}\eta^2 . \numberthis \label{arecursionitstconve}
    \end{align*}  
For notional convenience, let  $u_t=( \mathbb{E}_{A}[\|x_{t}-x_{t}^{l,\omega}\|^2])^{1/2}$.  Therefore, \eqref{arecursionitstconve} can be equivalently rewritten as 
      \begin{align*}\label{recursionitstconve}
       u_t^2&\leq2C_fL_g\sum_{i=1}^{t-1}(1-\frac{L\sigma\eta}{L+\sigma})^{t-i-1}\eta(\EX_{A}\bigl[\|y_{i+1}-g_S(x_i)\|^2\bigr])^{1/2}u_i\\
       &+2C_fL_g\sum_{i=1}^{t-1}(1-\frac{L\sigma\eta}{L+\sigma})^{t-i-1}\eta(\EX_{A}\bigl[\|y_{i+1}^{l,\omega}-g_S(x_i^{l,\omega})\|^2\bigr])^{1/2}u_i \\
       &+4C_f^2L_g^2 \sum_{i=0}^{t-1}(1-\frac{L\sigma\eta}{L+\sigma})^{t-i-1}\eta^2  \EX_{A}\bigl[\|y_{i+1}-g_S(x_i)\|^2\bigr] \\
        & +4C_f^2L_g^2 \sum_{i=0}^{t-1}(1-\frac{L\sigma\eta}{L+\sigma})^{t-i-1}\eta^2 \EX_{A}\bigl[\|y_{i+1}^{l,\omega}-g_S(x_i^{l,\omega})\|^2\bigr]   +16L_f^2 C_g\sum_{i=0}^{t-1}(1-\frac{L\sigma\eta}{L+\sigma})^{t-i-1}\eta^2 \\
        &+\frac{12L_gL_f}{m}\sum_{i=1}^{t-1}(1-\frac{L\sigma\eta}{L+\sigma})^{t-i-1}\eta u_i +\frac{4L_f^2 L_g^2}{m}\sum_{i=0}^{t-1}(1-\frac{L\sigma\eta}{L+\sigma})^{t-i-1}\eta^2 .\numberthis  
    \end{align*}
We will use Lemma \ref{recursion lemma} to get the desired result. To this end, notice that  
\begin{align*}
      16L_f^2 C_g\eta^2\sum_{i=0}^{t-1}(1-\eta\frac{L\sigma}{L+\sigma})^{t-i-1}&\leq16L_f^2 C_g\eta^2\frac{L+\sigma}{L\eta\sigma}=16L_f^2 C_g\frac{{L+\sigma}}{L\sigma}\eta, \\
      \frac{4L_f^2 L_g^2}{m}\eta^2\sum_{i=0}^{t-1}(1-\eta\frac{L\sigma}{L+\sigma})^{t-i-1}&\leq \frac{4L_f^2 L_g^2}{m}\eta^2\frac{L+\sigma}{L\eta\sigma}=\frac{4L_f^2 L_g^2}{m}\frac{L+\sigma}{L\sigma}\eta.
  \end{align*}
  Moreover, we define
    \begin{align*}
        S_t&=4C_f^2L_g^2 \sum_{i=0}^{t-1}(1-\frac{L\sigma\eta}{L+\sigma})^{t-i-1}\eta^2  \EX_{A}\bigl[\|y_{i+1}-g_S(x_i)\|^2\bigr] \\
        & +4C_f^2L_g^2 \sum_{i=0}^{t-1}(1-\frac{L\sigma\eta}{L+\sigma})^{t-i-1}\eta^2 \EX_{A}\bigl[\|y_{i+1}^{l,\omega}-g_S(x_i^{l,\omega})\|^2\bigr] \\
        &+16L_f^2 C_g\frac{{L+\sigma}}{L\sigma}\eta+\frac{4L_f^2 L_g^2}{m}\frac{L+\sigma}{L\sigma}\eta, \\
  \alpha_i  & =2C_fL_g(1-\frac{L\sigma\eta}{L+\sigma})^{t-i-1}\eta(\EX_{A}\bigl[\|y_{i+1}-g_S(x_i)\|^2\bigr])^{1/2}\\
       &+2C_fL_g(1-\frac{L\sigma\eta}{L+\sigma})^{t-i-1}\eta(\EX_{A}\bigl[\|y_{i+1}^{l,\omega}-g_S(x_i^{l,\omega})\|^2\bigr])^{1/2}
    +\frac{12L_gL_f}{m}(1-\frac{L\sigma\eta}{L+\sigma})^{t-i-1}\eta. 
    \end{align*}
    Applying Lemma \ref{recursion lemma} with $u_t$, $S_t$ and $\alpha_i$ defined as above to \eqref{recursionitstconve}, we get
    \begin{align*}
        &u_t\leq\sqrt{S_t}+\sum_{i=1}^{t-1}\alpha_i\\
        &\leq (4C_f^2L_g^2 \sum_{i=0}^{t-1}(1-\frac{L\sigma\eta}{L+\sigma})^{t-i-1}\eta^2  \EX_{A}\bigl[\|y_{i+1}-g_S(x_i)\|^2\bigr])^{1/2}\\
        &+(4C_f^2L_g^2 \sum_{i=0}^{t-1}(1-\frac{L\sigma\eta}{L+\sigma})^{t-i-1}\eta^2 \EX_{A}\bigl[\|y_{i+1}^{l,\omega}-g_S(x_i^{l,\omega})\|^2\bigr] )^{1/2}\\
      &+2C_fL_g\sum_{i=1}^{t-1}(1-\frac{L\sigma\eta}{L+\sigma})^{t-i-1}\eta(\EX_{A}\bigl[\|y_{i+1}-g_S(x_i)\|^2\bigr])^{1/2}\\
       &+2C_fL_g\sum_{i=1}^{t-1}(1-\frac{L\sigma\eta}{L+\sigma})^{t-i-1}\eta(\EX_{A}\bigl[\|y_{i+1}^{l,\omega}-g_S(x_i^{l,\omega})\|^2\bigr])^{1/2}\\
    &+4L_f\sqrt{C_g\frac{L+\sigma}{L\sigma}}\sqrt{\eta}+2L_fL_g\sqrt{\frac{L+\sigma}{L\sigma}}\sqrt{\frac{\eta}{m}}+\frac{12L_gL_f(L+\sigma)}{m L\sigma}, 
    \end{align*}
  where  we have used the fact that $(\sum_{i=1}^{4}a_i)^{1/2}\leq \sum_{i=1}^{4}(a_i)^{1/2}$ and  $\frac{12L_gL_f}{m}\sum_{i=0}^{t-1}(1-\frac{L\sigma\eta}{L+\sigma})^{t-i-1}\eta\leq \frac{12L_gL_f(L+\sigma)}{m L\sigma} $. 
 
Note that $\EX_{A}\bigl[\|y_{i+1}-g_S(x_i)\|^2\bigr]\leq \sup_{S}  \mathbb{E}_{A}[\|y_{i+1}-g_S(x_i)\|^2] $ and $\EX_{A}\bigl[\|y_{i+1}^{l,\omega}-g_S(x_i^{l,\omega})\|^2\bigr]\leq \sup_{S}  \mathbb{E}_{A}[\|y_{i+1}-g_S(x_i)\|^2]$. 
Consequently, with $T$ iterations, since $\mathbb{E}_{A}\bigl[\|x_{T}-x_{T}^{l,\omega}\|\bigr]\leq u_T=(\mathbb{E}_{A}\bigl[\|x_{T}-x_{T}^{l,\omega}\|^2\bigr])^{1/2}$, we further obtain 
\begin{align*}
    &\mathbb{E}_{A}\bigl[\|x_{T}-x_{T}^{l,\omega}\|\bigr]\\
    &\leq 4C_fL_g\eta\sup_S(\sum_{i=0}^{T-1}(1-\frac{L\sigma\eta}{L+\sigma})^{T-i-1} \EX_{A}\bigl[\|y_{i+1}-g_S(x_i)\|^2\bigr])^{1/2} \\
    &+4C_fL_g\eta\sup_S\sum_{i=0}^{T-1}(1-\frac{L\sigma\eta}{L+\sigma})^{T-i-1}\eta(\EX_{A}\bigl[\|y_{i+1}-g_S(x_i)\|^2\bigr])^{1/2}\\
    &+4L_f\sqrt{C_g\frac{L+\sigma}{L\sigma}}\sqrt{\eta}+2L_fL_g\sqrt{\frac{L+\sigma}{L\sigma}}\sqrt{\frac{\eta}{m}}+\frac{12L_gL_f(L+\sigma)}{m L\sigma}.\numberthis \label{beq:stabilitysrtconv}
\end{align*}
Combining the estimations for  $\mathbb{E}_{A}\bigl[\|x_{T}-x_{T}^{k,\nu}\|\bigr]$ and $\mathbb{E}_{A}\bigl[\|x_{T}-x_{T}^{l,\omega}\|\bigr]$, we obtain
\begin{align*}
\epsilon_{\nu}+\epsilon_{\omega} &\leq 8C_fL_g\eta\sup_S(\sum_{j=0}^{T-1} (1-\eta\frac{L\sigma}{L+\sigma})^{T-j-1}\EX_{A}\bigl[\left\|y_{j+1}-g_S\left(x_j\right)\right\|^2\bigr] )^{1/2}\\
    &+8C_fL_g\eta\sup_{S}\sum_{j=0}^{T-1}(1-\eta\frac{L\sigma}{L+\sigma})^{T-j-1}\bigl(\mathbb{E}_{A}[\| y_{j+1}-g_S(x_j)\|^2 ] \bigr)^{1/ 2}\\
    &+8L_f\sqrt{ C_g\frac{L+\sigma}{L \sigma}}\sqrt{\eta}+2L_fL_g\sqrt{\frac{L+\sigma}{L\sigma}}\sqrt{\frac{\eta}{n}}
    +\frac{4L_gL_f(L+\sigma)}{ nL\sigma}\\
    &+2L_fL_g\sqrt{\frac{L+\sigma}{L\sigma}}\sqrt{\frac{\eta}{m}}+\frac{12L_gL_f(L+\sigma)}{m L\sigma}\\
    &\leq16C_fL_g\eta\sup_{S}\sum_{j=0}^{T-1}(1-\eta\frac{L\sigma}{L+\sigma})^{T-j-1}\bigl(\mathbb{E}_{A}[\| y_{j+1}-g_S(x_j)\|^2 ] \bigr)^{1\over 2}\\
    &+8L_f\sqrt{ C_g\frac{L+\sigma}{L \sigma}}\sqrt{\eta}+2L_fL_g\sqrt{\frac{L+\sigma}{L\sigma}}\sqrt{\frac{\eta}{n}}
    +\frac{4L_gL_f(L+\sigma)}{ nL\sigma}\\
    &+2L_fL_g\sqrt{\frac{L+\sigma}{L\sigma}}\sqrt{\frac{\eta}{m}}+\frac{12L_gL_f(L+\sigma)}{m L\sigma}. \numberthis \label{finalstabstr}
\end{align*}
Next we will verify why the second inequality of \eqref{finalstabstr} holds true. 
With the result of SCGD update in Lemma \ref{Mengdi-lemma}, we have
\begin{align*}
    &\eta(\sum_{j=0}^{T-1} (1-\eta\frac{L\sigma}{L+\sigma})^{T-j-1}\EX_{A}\bigl[\left\|y_{j+1}-g_S\left(x_j\right)\right\|^2\bigr] )^{1/2}\\
    &\leq\eta(\sum_{j=1}^{T-1}   (1-\eta\frac{L\sigma}{L+\sigma})^{T-j-1}((\frac{c}{e})^c (j\beta)^{-c}\mathbb{E}_A[\| y_1- g_S(x_0)\|^2]+  L_f^2 L_g^3\frac{\eta^2}{\beta^2}+2V_g\beta) )^{1/2}\\
    &\leq \eta(\sum_{j=1}^{T-1} (1-\eta\frac{L\sigma}{L+\sigma})^{T-j-1}(L_f^2 L_g^3\frac{\eta^2}{\beta^2}+2V_g\beta))^{1/2}+\eta((\frac{c}{e})^cD_y\sum_{j=0}^{T-1} (1-\eta\frac{L\sigma}{L+\sigma})^{T-j-1}(j\beta)^{-c})^{1/2}\\
    &\leq \frac{L_fL_g\sqrt{L_g(L+\sigma)}}{\sqrt{L\sigma}}\frac{\eta^{3/2}}{\beta}+\sqrt{\frac{2V_g(L+\sigma)}{L\sigma}}\sqrt{\eta\beta}+(\frac{c}{e})^{\frac{c}{2}}\sqrt{D_y}\frac{\sqrt{(L+\sigma)\eta}}{\sqrt{L\sigma}}T^{-\frac{c}{2}}\beta^{-\frac{c}{2}},\numberthis \label{applymengdi1}
   \end{align*} 
   where the last inequality holds by the fact that $\sum_{j=0}^{T-1}  (1-\eta\frac{L\sigma}{L+\sigma})^{T-j-1}\leq\frac{L+\sigma}{\eta L\sigma}$ and Lemma \ref{lem:weighted_avg}. To see this, $(\sum_{j=1}^{T-1}  (1-\eta\frac{L\sigma}{L+\sigma})^{T-j-1}(j\beta)^{-c})^{1/2}\leq (\frac{\sum_{j=1}^{T-1}  (1-\eta\frac{L\sigma}{L+\sigma})^{T-j-1}\sum_{j=1}^{T-1}(j\beta)^{-c} }{T})^{1/2}\leq (\frac{T^{-c+1}\beta^{-c}(L+\sigma)}{T\eta L\sigma})^{1/2}=\frac{T^{-\frac{c}{2}}\beta^{-\frac{c}{2}}\sqrt{L+\sigma}}{\sqrt{\eta L\sigma}}$.
And
   \begin{align*}
    &\eta\sum_{j=0}^{T-1}(1-\eta\frac{L\sigma}{L+\sigma})^{T-j-1}\bigl(\mathbb{E}_{A}[\| y_{j+1}-g_S(x_j)\|^2 ] \bigr)^{1\over 2}\\
    &\leq \eta\sum_{j=1}^{T-1}(1-\eta\frac{L\sigma}{L+\sigma})^{T-j-1}((\frac{c}{e})^c (j\beta)^{-c}\mathbb{E}_A[\| y_1- g_S(x_0)\|^2]+  L_f^2 L_g^3\frac{\eta^2}{\beta^2}+2V_g\beta)^{1/2}\\
    &\leq \eta\sum_{j=1}^{T-1}  (1-\eta\frac{L\sigma}{L+\sigma})^{T-j-1}(\sqrt{L_g}L_gL_f\frac{\eta}{\beta}+\sqrt{2V_g}\sqrt{\beta})+(\frac{c}{e})^{\frac{c}{2}}\sqrt{D_y}\eta\sum_{j=1}^{T-1}  (1-\eta\sigma)^{T-j-1}(j\beta)^{-\frac{c}{2}}\\
    &\leq \frac{\sqrt{L_g}L_gL_f(L+\sigma)}{L\sigma}\frac{\eta}{\beta}+\frac{\sqrt{2V_g}(L+\sigma)}{L\sigma}\sqrt{\beta}+(\frac{c}{e})^{\frac{c}{2}}\frac{\sqrt{D_y}(L+\sigma)}{L\sigma} T^{-\frac{c}{2}}\beta^{-\frac{c}{2}},
    \numberthis \label{applymengdi2}
\end{align*}
where the last inequality holds by the fact that $\sum_{j=0}^{T-1}  (1-\eta\frac{L\sigma}{L+\sigma})^{T-j-1}\leq\frac{L+\sigma}{\eta L\sigma}$ and Lemma \ref{lem:weighted_avg}. To see this   $\sum_{j=1}^{T-1}  (1-\eta\frac{L\sigma}{L+\sigma})^{T-j-1}(j\beta)^{-\frac{c}{2}}\leq \frac{\sum_{j=1}^{T-1}  (1-\eta\frac{L\sigma}{L+\sigma})^{T-j-1}\sum_{j=1}^{T-1}(j\beta)^{-\frac{c}{2}} }{T}\leq\frac{T^{-\frac{c}{2}} \beta^{-\frac{c}{2}}(L+\sigma)}{\eta L \sigma}$. Comparing the result \eqref{applymengdi1} and \eqref{applymengdi2}, the dominating terms are \eqref{applymengdi2}. We can show that with result of SCSC update in Lemma \ref{Mengdi-lemma}, the dominating term is $\eta\sum_{j=0}^{T-1}(1-\eta\frac{L\sigma}{L+\sigma})^{T-j-1}\bigl(\mathbb{E}_{A}[\| y_{j+1}-g_S(x_j)\|^2 ] \bigr)^{1\over 2}$.\\
Since often we have $\eta\leq\min(\frac{1}{n}, \frac{1}{m})$, then $\frac{\sqrt{\eta}}{\sqrt{n}}\leq\frac{1}{n}$. Consequently, we get that $\sqrt{\frac{L+\sigma}{L\sigma}}\sqrt{\frac{\eta}{n}}\leq \frac{(L+\sigma)}{ nL\sigma}$. And $\frac{\sqrt{\eta}}{\sqrt{m}}\leq\frac{1}{m}$, $\sqrt{\frac{L+\sigma}{L\sigma}}\sqrt{\frac{\eta}{m}}\leq \frac{(L+\sigma)}{ m L\sigma}$.  We further get the final stability result for $\sigma$-strongly convex setting which holds for SCGD and SCSC in Theorem \ref{thm:stab_sconvex}
\begin{align*}
\epsilon_{\nu}+\epsilon_{\omega}= \mathcal{O}\Big(\frac{L_gL_f(L+\sigma)}{\sigma L m}+&\frac{L_gL_f(L+\sigma)}{\sigma Ln}+\frac{L_f\sqrt{C_g(L+\sigma)\eta}}{\sqrt{\sigma L}}\\
&+C_fL_g\eta\sup_{S}\sum_{j=1}^{T}(1-\eta\frac{L\sigma}{L+\sigma})^{T-j}\bigl(\mathbb{E}_{A}[\| y_{j+1}-g_S(x_j)\|^2 ] \bigr)^{1\over 2}\Bigr).\numberthis \label{finalstrstath5} 
\end{align*}
This completes the proof. 
\end{proof}

Next we move on to the Corollary \ref{cor:2}
\begin{proof}[Proof of Corollary \ref{cor:2}]
    Putting the result \eqref{applymengdi2}  to \eqref{finalstrstath5}, we get stability result of SCGD for strongly convex problems
\begin{align*}
    \epsilon_\nu + \epsilon_\omega= \mathcal{O}\bigl(n^{-1}+m^{-1}+\eta^{\frac{1}{2}}+\eta\beta^{-1}+\beta^{\frac{1}{2}}+T^{-\frac{c}{2}}\beta^{-\frac{c}{2}}\bigr).
\end{align*}
With SCSC update in Lemma \ref{Mengdi-lemma}, with a same progress, we have stability result of SCSC for strongly convex problems
\begin{align*}
    \epsilon_{\nu}+\epsilon_{\omega}=\mathcal{O}(n^{-1}+ m^{-1}+\eta^{1/2}+\eta \beta^{-1/2}+\beta^{1/2}+T^{-\frac{c}{2}} \beta^{-\frac{c}{2}}).
\end{align*}
\end{proof}

\subsection{Optimization}

\begin{lemma} \label{lem:opt_sconvex:1}
  Suppose Assumptions \ref{assum:1} \ref{assum:1b} and \ref{assum:2} \ref{part2} holds and \(F_S\) is \(\sigma\)-strongly convex. By running Algorithm \ref{alg:1}, we have for any \(x\in \mathcal{X}\)
    \begin{align*}  \label{eq:opt_sconvex:1}
      \mathbb{E}_A[\|x_{t+1}- x\|^2| \mathcal{F}_t]\leq& (1- \frac{\sigma\eta_t}{2})\|x_t- x\|^2+ \eta_t^2 \mathbb{E}_A [\|\nabla g_{\omega_{j_t}}(x_t)\nabla f_{\nu_{i_t}}(y_{t+1})\|^2| \mathcal{F}_t] \\
      &- 2\eta_t(F_S(x_t)- F_S(x))+ 2C_f^2L_g^2\frac{\eta_t}{\sigma}\mathbb{E}_A[\|g_S(x_t)- y_{t+1}\|^2| \mathcal{F}_t]. \\
      \numberthis
    \end{align*}
  where \(\mathbb{E}_A\) denotes the expectation taken with respect to the randomness of the algorithm, and \(\mathcal{F}_t\) is the \(\sigma\)-field generated by \(\{\omega_{j_0}, \ldots, \omega_{j_{t-1}}, \nu_{i_0}, \ldots, \nu_{i_{t- 1}}\}\).
\end{lemma}

The proof of Lemma \ref{lem:opt_sconvex:1} is deferred to the end of this subsection. Now we are ready to prove the convergence of Algorithm \ref{alg:1} for strongly convex problems.

\begin{proof}[Proof of Theorem \ref{thm:opt_sconvex}]
  We first present the proof for the SCGD update. Taking full expectation over \eqref{eq:opt_sconvex:1} with \(x = x_*^S\) and using Assumption \ref{assum:1}, we get
    \begin{align*} \label{eq:opt_sconvex:4}
        \mathbb{E}_A[\|x_{t+1}- x_*^S\|^2]\leq& (1- \frac{\sigma\eta_t}{2})\mathbb{E}_A[\|x_t- x_*^S\|^2]+ L_f^2L_g^2\eta_t^2- 2\eta_t\mathbb{E}_A[F_S(x_t)- F_S(x_*^S)] \\
        &+ 2C_f^2L_g^2\frac{\eta_t}{\sigma}\mathbb{E}_A[\|g_S(x_t)- y_{t+1}\|^2].
        \numberthis
    \end{align*}
  Setting $\eta_t= \eta$ and $\beta_t= \beta$, plugging Lemma \ref{Mengdi-lemma} into \eqref{eq:opt_sconvex:4}, and letting \(D_y:= \mathbb{E}_A[\| y_1- g_S(x_0)\|^2]\), we have
  \begin{align*}
    \mathbb{E}_A[\|x_{t+1}- x_*^S\|^2]\leq& (1- \frac{\sigma\eta}{2})\mathbb{E}_A[\|x_t- x_*^S\|^2]+ L_f^2L_g^2\eta^2- 2\eta\mathbb{E}_A[F_S(x_t)- F_S(x_*^S)] \\
    &+ \frac{2C_f^2L_g^2\eta}{\sigma}\left( (\frac{c}{e})^cD_y (t\beta)^{-c}+L_g^3 L_f^2 \frac{\eta^2}{\beta^2}+ 2V_g \beta\right).
  \end{align*}
  Multiplying the above inequality with \(\left( 1- \frac{\sigma\eta}{2}\right)^{T- t}\) and telescoping for \(t= 1, \ldots, T\), we get
  \begin{align*}
    &2\eta \sum_{t= 1}^{T}\left( 1- \frac{\sigma\eta}{2}\right)^{T- t}\mathbb{E}_A[F_S(x_t)- F_S(x_*^S)] \\
    \leq& \left( 1- \frac{\sigma\eta}{2}\right)^T\mathbb{E}_A[\| x_1- x_*^S\|^2]+ L_f^2L_g^2 \eta^2 \sum_{t= 1}^{T} \left( 1- \frac{\sigma\eta}{2}\right)^{T- t} \\
    &+ \frac{2C_f^2L_g^2D_y}{\sigma}\left( \frac{c}{e}\right)^c \eta\beta^{-c} \sum_{t= 1}^{T} \left( 1- \frac{\sigma\eta}{2}\right)^{T- t}t^{-c} \\
    &+ \frac{4C_f^2L_g^2V_g}{\sigma} \eta\beta \sum_{t= 1}^{T} \left( 1- \frac{\sigma\eta}{2}\right)^{T- t}+ \frac{2C_f^2L_f^2L_g^5}{\sigma} \frac{\eta^3}{\beta^2} \sum_{t= 1}^{T} \left( 1- \frac{\sigma\eta}{2}\right)^{T- t}.
  \end{align*}
  Note that we have
  \begin{equation*}
    \mathbb{E}_A[\| x_1- x_*^S\|^2]\leq \mathbb{E}_A[\| x_0- x_*^S- \eta \nabla g_{\omega_{j_0}}(x_0) \nabla f_{\nu_{i_0}}(y_1)\|^2]\leq 2\mathbb{E}_A[\| x_0- x_*^S\|^2]+ 2L_f^2L_g^2\eta_t^2
  \end{equation*}
  Combining the above two inequalities yields
  \begin{align*}
    &2\eta \sum_{t= 1}^{T}\left( 1- \frac{\sigma\eta}{2}\right)^{T- t}\mathbb{E}_A[F_S(x_t)- F_S(x_*^S)] \\
    \leq& 2\left( 1- \frac{\sigma\eta}{2}\right)^T\mathbb{E}_A[\| x_0- x_*^S\|^2]+ 2L_f^2L_g^2 \eta^2 \sum_{t= 1}^{T} \left( 1- \frac{\sigma\eta}{2}\right)^{T- t} \\
    &+ \frac{2C_f^2L_g^2D_y}{\sigma}\left( \frac{c}{e}\right)^c \eta\beta^{-c} \sum_{t= 1}^{T} \left( 1- \frac{\sigma\eta}{2}\right)^{T- t}t^{-c} \\
    &+ \frac{4C_f^2L_g^2V_g}{\sigma} \eta\beta \sum_{t= 1}^{T} \left( 1- \frac{\sigma\eta}{2}\right)^{T- t}+ \frac{2C_f^2L_f^2L_g^5}{\sigma} \frac{\eta^3}{\beta^2} \sum_{t= 1}^{T} \left( 1- \frac{\sigma\eta}{2}\right)^{T- t}.
  \end{align*}
  From Lemma \ref{lem:sum_prod} we know \((1- \frac{\sigma\eta}{2})^T\leq \exp(-\frac{\sigma\eta T}{2})\leq (\frac{2c}{e\sigma})^c (\eta T)^{-c}\). Also we have \(\sum_{t= 1}^{T} (1- \frac{\sigma\eta}{2})^{T- t}= \frac{1- (1- \frac{\sigma\eta}{2})^{T-1}}{1- (1- \frac{\sigma\eta}{2})}\leq \frac{2}{\sigma\eta}\). Dividing both sides of the above inequality by \(2\eta\), and letting \(D_x:= \mathbb{E}_A[\| x_0- x_*^S\|^2]\), we get
    \begin{align*} \label{eq:opt_sconvex:5}
      &\sum_{t= 1}^{T}\left( 1- \frac{\sigma\eta}{2}\right)^{T- t}\mathbb{E}_A[F_S(x_t)- F_S(x_*^S)] \\
      \leq& \left(\frac{2c}{e\sigma}\right)^cD_x \eta^{-c- 1}T^{-c}+ \frac{2L_f^2L_g^2}{\sigma}+ \frac{C_f^2L_g^2D_y}{\sigma}\left( \frac{c}{e}\right)^c \beta^{-c} \sum_{t= 1}^{T} \left( 1- \frac{\sigma\eta}{2}\right)^{T- t}t^{-c} \\
      &+ \frac{4C_f^2L_g^2V_g}{\sigma^2} \frac{\beta}{\eta}+ \frac{2C_f^2L_f^2L_g^5}{\sigma^2} \frac{\eta}{\beta^2}.
      \numberthis
    \end{align*}
  Dividing both sides of \eqref{eq:opt_sconvex:5} by \(\sum_{t= 1}^{T} (1- \frac{\sigma\eta}{2})^{T- t}\), noting that for \((\eta (T- 1))^{-1}\leq \frac{\sigma}{2}\) we have \((1- \frac{\sigma\eta}{2})^{T- 1}\leq \exp(- \frac{\sigma\eta (T- 1)}{2})\leq \frac{1}{2}\), and thus \(\sum_{t= 1}^{T} (1- \frac{\sigma\eta}{2})^{T- t}= \frac{1- (1- \frac{\sigma\eta}{2})^{T-1}}{1- (1- \frac{\sigma\eta}{2})}\geq \frac{1}{\sigma\eta}\), from the choice of $A(S)$ and convexity of $F_S$ we get
    \begin{align*} \label{eq:opt_sconvex:6}
      \mathbb{E}_A[F_S(A(S))- F_S(x_*^S)]\leq& (\frac{2c}{e\sigma})^{c-1}D_x (\eta T)^{-c}+ \frac{C_f^2L_g^2D_y}{\sigma}\left( \frac{c}{e}\right)^c \beta^{-c} \frac{\sum_{t= 1}^{T} \left( 1- \frac{\sigma\eta}{2}\right)^{T- t}t^{-c}}{\sum_{t= 1}^{T} \left( 1- \frac{\sigma\eta}{2}\right)^{T- t}} \\
      &+ 2L_f^2L_g^2\eta+ \frac{4C_f^2L_g^2V_g}{\sigma} \beta+ \frac{2C_f^2L_f^2L_g^5}{\sigma} \frac{\eta^2}{\beta^2}.
      \numberthis
    \end{align*}
  Note that \(\left( 1- \frac{\sigma\eta}{2}\right)^{T- t}\) is non-decreasing with respect to \(t\) and for \(c> 0\), \(t^{-c}\) is non-increasing with respect to \(t\). Then from Lemma \ref{lem:weighted_avg} we have
  \begin{equation*}
      \frac{\sum_{t= 1}^{T} \left( 1- \frac{\sigma\eta}{2}\right)^{T- t}t^{-c}}{\sum_{t= 1}^{T} \left( 1- \frac{\sigma\eta}{2}\right)^{T- t}}\leq \frac{\sum_{t= 1}^{T} t^{-c}}{T}
  \end{equation*}
  Thus \eqref{eq:opt_sconvex:6} simplifies to
  \begin{align*}
    \mathbb{E}_A[F_S(A(S))- F_S(x_*^S)]\leq& (\frac{2c}{e\sigma})^{c-1}D_x (\eta T)^{-c}+ 2L_f^2L_g^2\eta+ \frac{C_f^2L_g^2D_y}{\sigma}\left( \frac{c}{e}\right)^c \beta^{-c}T^{-1} \sum_{t= 1}^{T} t^{-c} \\
    &+ \frac{4C_f^2L_g^2V_g}{\sigma} \beta+ \frac{2C_f^2L_f^2L_g^5}{\sigma} \frac{\eta^2}{\beta^2}.
  \end{align*}
  Note that \(\sum_{t= 1}^{T} t^{-z}= \mathcal{O}(T^{1- z})\) for \(z\in (0, 1)\cup (1, \infty)\) and \(\sum_{t= 1}^{T} t^{-1}= \mathcal{O}(\log T)\). As long as \(c\neq 1\) we get
  \begin{align*}
    &\mathbb{E}_A[F_S(A(S))- F_S(x_*^S)] \\
    =& \mathcal{O}\left( D_x(\eta T)^{-c}+ 2L_f^2L_g^2\eta+ \frac{C_f^2L_g^2D_y}{\sigma}(\beta T)^{-c}+ \frac{C_f^2L_g^2V_g}{\sigma}\beta+ \frac{C_f^2L_f^2L_g^5}{\sigma}\eta^2\beta^{-2}\right).
  \end{align*}
  Then we get the desired result for the SCGD update. Next we present the proof for the SCSC update. Setting $\eta_t= \eta$ and $\beta_t= \beta$. Plugging Lemma \ref{Mengdi-lemma} into \eqref{eq:opt_sconvex:4}, and letting \(D_y:= \mathbb{E}_A[\| y_1- g_S(x_0)\|^2]\), we have
  \begin{align*}
    \mathbb{E}_A[\|x_{t+1}- x_*^S\|^2]\leq& (1- \frac{\sigma\eta}{2})\mathbb{E}_A[\|x_t- x_*^S\|^2]+ L_f^2L_g^2\eta^2- 2\eta\mathbb{E}_A[F_S(x_t)- F_S(x_*^S)] \\
    &+ \frac{2C_f^2L_g^2\eta}{\sigma}\left( (\frac{c}{e})^cD_y (t\beta)^{-c}+L_g^3 L_f^2 \frac{\eta^2}{\beta}+ 2V_g \beta\right).
  \end{align*}
  Telescoping the above inequality for $t= 1, \cdots, T$, and rearranging the terms, we get
    \begin{align*}
      &2\eta \sum_{t= 1}^{T}\left( 1- \frac{\sigma\eta}{2}\right)^{T- t}\mathbb{E}_A[F_S(x_t)- F_S(x_*^S)] \\
      \leq& \left( 1- \frac{\sigma\eta}{2}\right)^T\mathbb{E}_A[\| x_1- x_*^S\|^2]+ L_f^2L_g^2 \eta^2 \sum_{t= 1}^{T} \left( 1- \frac{\sigma\eta}{2}\right)^{T- t} \\
      &+ \frac{2C_f^2L_g^2D_y}{\sigma}\left( \frac{c}{e}\right)^c \eta\beta^{-c} \sum_{t= 1}^{T} \left( 1- \frac{\sigma\eta}{2}\right)^{T- t}t^{-c} \\
      &+ \frac{4C_f^2L_g^2V_g}{\sigma} \eta\beta \sum_{t= 1}^{T} \left( 1- \frac{\sigma\eta}{2}\right)^{T- t}+ \frac{2C_f^2L_f^2L_g^5}{\sigma} \frac{\eta^3}{\beta} \sum_{t= 1}^{T} \left( 1- \frac{\sigma\eta}{2}\right)^{T- t} \\
      \leq& 2\left( 1- \frac{\sigma\eta}{2}\right)^T\mathbb{E}_A[\| x_0- x_*^S\|^2]+ 2L_f^2L_g^2 \eta^2 \sum_{t= 1}^{T} \left( 1- \frac{\sigma\eta}{2}\right)^{T- t} \\
      &+ \frac{2C_f^2L_g^2D_y}{\sigma}\left( \frac{c}{e}\right)^c \eta\beta^{-c} \sum_{t= 1}^{T} \left( 1- \frac{\sigma\eta}{2}\right)^{T- t}t^{-c} \\
      &+ \frac{4C_f^2L_g^2V_g}{\sigma} \eta\beta \sum_{t= 1}^{T} \left( 1- \frac{\sigma\eta}{2}\right)^{T- t}+ \frac{2C_f^2L_f^2L_g^5}{\sigma} \frac{\eta^3}{\beta} \sum_{t= 1}^{T} \left( 1- \frac{\sigma\eta}{2}\right)^{T- t} \\
    \end{align*}
  From Lemma \ref{lem:sum_prod} we know \((1- \frac{\sigma\eta}{2})^T\leq \exp(-\frac{\sigma\eta T}{2})\leq (\frac{2c}{e\sigma})^c (\eta T)^{-c}\). Also we have \(\sum_{t= 1}^{T} (1- \frac{\sigma\eta}{2})^{T- t}= \frac{1- (1- \frac{\sigma\eta}{2})^{T-1}}{1- (1- \frac{\sigma\eta}{2})}\leq \frac{2}{\sigma\eta}\). Dividing both sides of the above inequality by \(2\eta\), and letting \(D_x:= \mathbb{E}_A[\| x_0- x_*^S\|^2]\), we get
    \begin{align*} \label{eq:opt_sconvex:11}
      &\sum_{t= 1}^{T}\left( 1- \frac{\sigma\eta}{2}\right)^{T- t}\mathbb{E}_A[F_S(x_t)- F_S(x_*^S)] \\
      \leq& \left(\frac{2c}{e\sigma}\right)^cD_x \eta^{-c- 1}T^{-c}+ \frac{2L_f^2L_g^2}{\sigma}+ \frac{C_f^2L_g^2D_y}{\sigma}\left( \frac{c}{e}\right)^c \beta^{-c} \sum_{t= 1}^{T} \left( 1- \frac{\sigma\eta}{2}\right)^{T- t}t^{-c} \\
      &+ \frac{4C_f^2L_g^2V_g}{\sigma^2} \frac{\beta}{\eta}+ \frac{2C_f^2L_f^2L_g^5}{\sigma^2} \frac{\eta}{\beta}.
      \numberthis
    \end{align*}
  Dividing both sides of \eqref{eq:opt_sconvex:11} by \(\sum_{t= 1}^{T} (1- \frac{\sigma\eta}{2})^{T- t}\), noting that for \((\eta (T- 1))^{-1}\leq \frac{\sigma}{2}\) we have \((1- \frac{\sigma\eta}{2})^{T- 1}\leq \exp(- \frac{\sigma\eta (T- 1)}{2})\leq \frac{1}{2}\), and thus \(\sum_{t= 1}^{T} (1- \frac{\sigma\eta}{2})^{T- t}= \frac{1- (1- \frac{\sigma\eta}{2})^{T-1}}{1- (1- \frac{\sigma\eta}{2})}\geq \frac{1}{\sigma\eta}\), from the choice of $A(S)$ and convexity of $F_S$ we get
  \begin{align*}
    &\mathbb{E}_A[F_S(A(S))- F_S(x_*^S)] \\
    \leq& (\frac{2c}{e\sigma})^{c-1}D_x (\eta T)^{-c}+ 2L_f^2L_g^2\eta+ \frac{C_f^2L_g^2D_y}{\sigma}\left( \frac{c}{e}\right)^c \beta^{-c} \frac{\sum_{t= 1}^{T} \left( 1- \frac{\sigma\eta}{2}\right)^{T- t}t^{-c}}{\sum_{t= 1}^{T} \left( 1- \frac{\sigma\eta}{2}\right)^{T- t}} \\
    &+ \frac{4C_f^2L_g^2V_g}{\sigma} \beta+ \frac{2C_f^2L_f^2L_g^5}{\sigma} \frac{\eta^2}{\beta} \\
    \leq& (\frac{2c}{e\sigma})^{c-1}D_x (\eta T)^{-c}+ 2L_f^2L_g^2\eta+ \frac{C_f^2L_g^2D_y}{\sigma}\left( \frac{c}{e}\right)^c \beta^{-c}T^{-1} \sum_{t= 1}^{T} t^{-c} \\
    &+ \frac{4C_f^2L_g^2V_g}{\sigma} \beta+ \frac{2C_f^2L_f^2L_g^5}{\sigma} \frac{\eta^2}{\beta},
  \end{align*}
  where the last inequality comes from Lemma \ref{lem:weighted_avg}. Noting that \(\sum_{t= 1}^{T} t^{-z}= \mathcal{O}(T^{1- z})\) for \(z\in (0, 1)\cup (1, \infty)\) and \(\sum_{t= 1}^{T} t^{-1}= \mathcal{O}(\log T)\), as long as \(c\neq 1\) we get
  \begin{align*}
    &\mathbb{E}_A[F_S(A(S))- F_S(x_*^S)] \\
    =& \mathcal{O}\left( D_x(\eta T)^{-c}+ 2L_f^2L_g^2\eta+ \frac{C_f^2L_g^2D_y}{\sigma}(\beta T)^{-c}+ \frac{C_f^2L_g^2V_g}{\sigma}\beta+ \frac{C_f^2L_f^2L_g^5}{\sigma}\eta^2\beta^{-1}\right).
  \end{align*}
  Then we get the desired result for the SCSC update. Then we complete the proof.
\end{proof}

\begin{proof}[Proof of Lemma \ref{lem:opt_sconvex:1}]
  From Algorithm \ref{alg:1} we have for any \(x\in \mathcal{X}\)
    \begin{align*}
      &\|x_{t+1}- x\|^2 \\
      \leq& \|x_t- \eta_t\nabla g_{\omega_{j_t}}(x_t)\nabla f_{\nu_{i_t}}(y_{t+1})- x\|^2 \\
      =& \|x_t- x\|^2+ \eta_t^2\|\nabla g_{\omega_{j_t}}(x_t)\nabla f_{\nu_{i_t}}(y_{t+1})\|^2- 2\eta_t\langle x_t- x, \nabla g_{\omega_{j_t}}(x_t)\nabla f_{\nu_{i_t}}(y_{t+1})\rangle \\
      =& \|x_t- x\|^2+ \eta_t^2\|\nabla g_{\omega_{j_t}}(x_t)\nabla f_{\nu_{i_t}}(y_{t+1})\|^2- 2\eta_t\langle x_t- x, \nabla g_{\omega_{j_t}}(x_t)\nabla f_{\nu_{i_t}}(g_S(x_t))\rangle + u_t, \\
    \end{align*}
  where
  \begin{equation*}
    u_t:= 2\eta_t\langle x_t- x, \nabla g_{\omega_{j_t}}(x_t)\nabla f_{\nu_{i_t}}(g_S(x_t))- \nabla g_{\omega_{j_t}}(x_t)\nabla f_{\nu_{i_t}}(y_{t+1})\rangle.
  \end{equation*}
  Let \(\mathcal{F}_t\) be the \(\sigma\)-field generated by \(\{\omega_{j_0}, \ldots, \omega_{j_{t-1}}, \nu_{i_0}, \ldots, \nu_{i_{t- 1}}\}\). Taking expectation with respect to the randomness of the algorithm conditioned on \(\mathcal{F}_t\), we have
    \begin{align*} \label{eq:opt_sconvex:2}
      &\mathbb{E}_A[\|x_{t+1}- x\|^2| \mathcal{F}_t] \\
      \leq& \|x_t- x\|^2+ \eta_t^2 \mathbb{E}_A [\|\nabla g_{\omega_{j_t}}(x_t)\nabla f_{\nu_{i_t}}(y_{t+1})\|^2| \mathcal{F}_t] \\
      &- 2\eta_t \mathbb{E}_A[\langle x_t- x, \nabla g_{\omega_{j_t}}(x_t)\nabla f_{\nu_{i_t}}(g_S(x_t))\rangle|\mathcal{F}_t]+ \mathbb{E}_A[u_t| \mathcal{F}_t] \\
      =& \|x_t- x\|^2+ \eta_t^2 \mathbb{E}_A [\|\nabla g_{\omega_{j_t}}(x_t)\nabla f_{\nu_{i_t}}(y_{t+1})\|^2| \mathcal{F}_t]- 2\eta_t\langle x_t- x, \nabla F_S(x_t)\rangle+ \mathbb{E}_A[u_t| \mathcal{F}_t] \\
      \leq& \|x_t- x\|^2+ \eta_t^2 \mathbb{E}_A [\|\nabla g_{\omega_{j_t}}(x_t)\nabla f_{\nu_{i_t}}(y_{t+1})\|^2| \mathcal{F}_t]- 2\eta_t(F_S(x_t)- F_S(x)+ \frac{\sigma}{2}\|x_t- x\|^2) \\
      &+ \mathbb{E}_A[u_t| \mathcal{F}_t], \\
      \numberthis
    \end{align*}
  where the last inequality comes from the strong convexity of $F_S$. Note that from Cauchy-Schwartz inequality, Young's inequality, Assumption \ref{assum:1} \ref{assum:1b} and \ref{assum:2} \ref{part2} we have
    \begin{align*} \label{eq:opt_sconvex:3}
      u_t \leq& 2\eta_t\|x_t- x\|\|\nabla g_{\omega_{j_t}}(x_t)\|\|\nabla f_{\nu_{i_t}}(g_S(x_t))- \nabla f_{\nu_{i_t}}(y_{t+1})\| \\
      \leq& 2C_f\eta_t\|x_t- x\|\|\nabla g_{\omega_{j_t}}(x_t)\|\|g_S(x_t)- y_{t+1}\| \\
      \leq& 2C_f\eta_t\left(\frac{\|x_t- x\|^2\|\nabla g_{\omega_{j_t}}(x_t)\|^2}{2\gamma}+ \frac{\gamma}{2}\|g_S(x_t)- y_{t+1}\|^2\right) \\
      \leq& \frac{C_fL_g^2\eta_t}{\gamma}\|x_t- x\|^2+ \gamma C_f\eta_t\|g_S(x_t)- y_{t+1}\|^2
      \numberthis
    \end{align*}
  for any $\gamma> 0$. Substituting (\ref{eq:opt_sconvex:3}) into (\ref{eq:opt_sconvex:2}), we get
    \begin{align*}
      \mathbb{E}_A[\|x_{t+1}- x\|^2| \mathcal{F}_t]\leq& \left(1+ \frac{C_fL_g^2\eta_t}{\gamma}- \sigma\eta_t\right)\|x_t- x\|^2+ \eta_t^2 \mathbb{E}_A [\|\nabla g_{\omega_{j_t}}(x_t)\nabla f_{\nu_{i_t}}(y_{t+1})\|^2| \mathcal{F}_t] \\
      &- 2\eta_t(F_S(x_t)- F_S(x))+  \gamma C_f\eta_tE[\|g_S(x_t)- y_{t+1}\|^2| \mathcal{F}_t]. \\
    \end{align*}
  Setting $\gamma = \frac{2C_fL_g^2}{\sigma}$, we have
    \begin{align*}
      \mathbb{E}_A[\|x_{t+1}- x\|^2| \mathcal{F}_t]\leq& (1- \frac{\sigma\eta_t}{2})\|x_t- x\|^2+ \eta_t^2 \mathbb{E}_A [\|\nabla g_{\omega_{j_t}}(x_t)\nabla f_{\nu_{i_t}}(y_{t+1})\|^2| \mathcal{F}_t] \\
      &- 2\eta_t(F_S(x_t)- F_S(x))+ 2C_f^2L_g^2\frac{\eta_t}{\sigma}\mathbb{E}_A[\|g_S(x_t)- y_{t+1}\|^2| \mathcal{F}_t]. \\
    \end{align*}
  Then we complete the proof.
\end{proof}

\subsection{Generalization}

\begin{proof}[Proof of Theorem \ref{thm:gen_sconvex}]
  We first present the proof for the SCGD update.
  From the stability results \eqref{eq:stabilitysrtconv}, \eqref{beq:stabilitysrtconv} and \eqref{finalstabstr} we get
    \begin{align*}
        &\mathbb{E}_A[\|x_{t}- x_{t}^{k, \nu}\|]+ 4\mathbb{E}_A[\|x_{t}- x_{t}^{l, \omega}\|] \\
        \leq& 40C_fL_g\eta\sup_{S}\sum_{j=0}^{t-1}(1-\eta \frac{L\sigma}{L+ \sigma})^{t-j- 1}\bigl(\mathbb{E}_{A}[\| y_{j+1}-g_S(x_j)\|^2 ] \bigr)^{1\over 2}+ 20L_f\sqrt{ C_g\frac{L+\sigma}{L \sigma}}\sqrt{\eta} \\
        & +2L_fL_g\sqrt{\frac{L+\sigma}{L\sigma}}\sqrt{\frac{\eta}{n}}
        +\frac{4L_gL_f(L+\sigma)}{ nL\sigma}+ 8L_fL_g\sqrt{\frac{L+\sigma}{L\sigma}}\sqrt{\frac{\eta}{m}}+\frac{48L_gL_f(L+\sigma)}{m L\sigma}.
    \end{align*}
    Plugging \eqref{applymengdi2} into the above inequality, we get
  \begin{align*}
    &\mathbb{E}_A[\|x_{t}- x_{t}^{k, \nu}\|]+ 4\mathbb{E}_A[\|x_{t}- x_{t}^{l, \omega}\|] \\
    \leq& 40C_f L_g \frac{\sqrt{L_g}L_gL_f(L+ \sigma)}{L\sigma}\frac{\eta}{\beta}+ 40C_fL_g \frac{\sqrt{2V_g}(L+ \sigma)}{L\sigma}\sqrt{\beta}\\
  &+ 40C_fL_g(\frac{c}{e})^{\frac{c}{2}}\frac{D_y(L+ \sigma)}{L\sigma} t^{-\frac{c}{2}}\beta^{-\frac{c}{2}}+ 20L_f\sqrt{ C_g\frac{L+\sigma}{L \sigma}}\sqrt{\eta}+2L_fL_g\sqrt{\frac{L+\sigma}{L\sigma}}\sqrt{\frac{\eta}{n}} \\
  &+ \frac{4L_gL_f(L+\sigma)}{ nL\sigma}+ 8L_fL_g\sqrt{\frac{L+\sigma}{L\sigma}}\sqrt{\frac{\eta}{m}}+\frac{48L_gL_f(L+\sigma)}{m L\sigma}.
  \end{align*}
  Using Theorem \ref{thm:1}, we have
    \begin{align*} \label{eq:opt_sconvex:7}
      &\mathbb{E}_{S, A} \left[F(x_t)- F_S(x_t)\right] \\
      \leq& 40C_f  \frac{\sqrt{L_g}L_g^3L_f^2(L+ \sigma)}{L\sigma}\frac{\eta}{\beta}+ 40C_fL_g^2L_f \frac{\sqrt{2V_g}(L+ \sigma)}{L\sigma}\sqrt{\beta}\\
      &+ 40C_fL_g^2L_f(\frac{c}{e})^{\frac{c}{2}}\frac{D_y(L+ \sigma)}{L\sigma} t^{-\frac{c}{2}}\beta^{-\frac{c}{2}}+ 20L_f^2L_g\sqrt{ C_g\frac{L+\sigma}{L \sigma}}\sqrt{\eta}+2L_f^2L_g^2\sqrt{\frac{L+\sigma}{L\sigma}}\sqrt{\frac{\eta}{n}} \\
      &+ \frac{4L_g^2L_f^2(L+\sigma)}{ nL\sigma}+ 8L_f^2L_g^2\sqrt{\frac{L+\sigma}{L\sigma}}\sqrt{\frac{\eta}{m}}+\frac{48L_g^2L_f^2(L+\sigma)}{m L\sigma}+L_f\sqrt{\frac{\mathbb{E}_{S, A}[\mathrm{Var}_\omega(g_\omega(x_t))]}{m}}.
        \numberthis
    \end{align*}
    From \eqref{eq:opt_sconvex:5} we get
      \begin{align*} \label{eq:opt_sconvex:8}
        &\sum_{t= 1}^{T}\left( 1- \frac{\sigma\eta}{2}\right)^{T- t}\mathbb{E}_{S, A}[F_S(x_t)- F_S(x_*^S)] \\
      \leq& \left(\frac{2c}{e\sigma}\right)^cD_x \eta^{-c- 1}T^{-c}+ \frac{2L_fL_g}{\sigma}+ \frac{C_f^2L_gD_y}{\sigma}\left( \frac{c}{e}\right)^c \beta^{-c} \sum_{t= 1}^{T} \left( 1- \frac{\sigma\eta}{2}\right)^{T- t}t^{-c} \\
      &+ \frac{4C_f^2L_gV_g}{\sigma^2} \frac{\beta}{\eta}+ \frac{2C_f^2L_fL_g^3}{\sigma^2} \frac{\eta}{\beta^2}.
      \numberthis
      \end{align*}
  Multiplying both sides of \eqref{eq:opt_sconvex:7} with \(\left( 1- \frac{\sigma\eta}{2}\right)^{T- t}\), telescoping from \(t= 1, \ldots, T\), then adding the result with \eqref{eq:opt_sconvex:8}, and using the fact $F_S(x_*^S)\leq F_S(x_*)$, we get
  \begin{align*}
      &\sum_{t= 1}^{T}\left( 1- \frac{\sigma\eta}{2}\right)^{T- t}\mathbb{E}_{S, A}[F(x_t)- F(x_*)] \\
      \leq& 40C_f  \frac{\sqrt{L_g}L_g^3L_f^2(L+ \sigma)}{L\sigma}\frac{\eta}{\beta} \sum_{t= 1}^{T}\left( 1- \frac{\sigma\eta}{2}\right)^{T- t}+  40C_fL_g^2L_f \frac{\sqrt{2V_g}(L+ \sigma)}{L\sigma}\sqrt{\beta} \sum_{t= 1}^{T}\left( 1- \frac{\sigma\eta}{2}\right)^{T- t} \\
      &+ 40C_fL_g^2L_f(\frac{c}{e})^{\frac{c}{2}}\frac{D_y(L+ \sigma)}{L\sigma}\beta^{-\frac{c}{2}} \sum_{t= 1}^{T}\left( 1- \frac{\sigma\eta}{2}\right)^{T- t} t^{-\frac{c}{2}}+ \frac{4L_f^2L_g^2(L+ \sigma)}{L\sigma n} \sum_{t= 1}^{T}\left( 1- \frac{\sigma\eta}{2}\right)^{T- t} \\
      &+ 20L_f^2L_g\sqrt{ C_g\frac{L+\sigma}{L \sigma}}\sqrt{\eta} \sum_{t= 1}^{T}\left( 1- \frac{\sigma\eta}{2}\right)^{T- t}+ 2L_f^2L_g^2\sqrt{\frac{L+\sigma}{L\sigma}}\sqrt{\frac{\eta}{n}}\sum_{t= 1}^{T}\left( 1- \frac{\sigma\eta}{2}\right)^{T- t} \\
      &+ \frac{48L_g^2L_f^2(L+ \sigma)}{L\sigma m}\sum_{t= 1}^{T}\left( 1- \frac{\sigma\eta}{2}\right)^{T- t}+ 8L_f^2L_g^2\sqrt{\frac{L+\sigma}{L\sigma}}\sqrt{\frac{\eta}{m}} \sum_{t= 1}^{T}\left( 1- \frac{\sigma\eta}{2}\right)^{T- t} \\
      &+ L_f\sum_{t= 1}^{T}\left( 1- \frac{\sigma\eta}{2}\right)^{T- t} \sqrt{\frac{\mathbb{E}_{S, A}[\mathrm{Var}_\omega(g_\omega(x_t))]}{m}}+ \left(\frac{2c}{e\sigma}\right)^cD_x \eta^{-c- 1}T^{-c}+ \frac{2L_fL_g}{\sigma} \\
    &+ \frac{C_f^2L_gD_y}{\sigma}\left( \frac{c}{e}\right)^c \beta^{-c} \sum_{t= 1}^{T} \left( 1- \frac{\sigma\eta}{2}\right)^{T- t}t^{-c}+ \frac{4C_f^2L_gV_g}{\sigma^2} \frac{\beta}{\eta}+ \frac{2C_f^2L_fL_g^3}{\sigma^2} \frac{\eta}{\beta^2}.
  \end{align*}
  Dividing both sides of the above inequality by \(\sum_{t= 1}^{T}\left( 1- \frac{\sigma\eta}{2}\right)^{T- t}\),
  and setting \(\eta= T^{-a}\) and \(\beta= T^{-b}\) with \(a, b\in (0, 1]\), then from the choice of \(A(S)\) and convexity of \(F\) and Lemma \ref{lem:weighted_avg}, noting that \(\sum_{t= 1}^{T} (1- \frac{\sigma\eta}{2})^{T- t}= \frac{1- (1- \frac{\sigma\eta}{2})^{T- 1}}{1- (1- \frac{\sigma\eta}{2})}\geq \frac{1}{\sigma\eta}\) for \((\eta (T- 1))^{-1}\leq \frac{\sigma}{2}\), we get
  \begin{align*}
      &\mathbb{E}_{S, A}[F(A(S))- F(x_*)] \\
        \leq& 40C_f  \frac{\sqrt{L_g}L_g^3L_f^2(L+ \sigma)}{L\sigma} T^{b- a}+ 40C_fL_g^2L_f \frac{\sqrt{2V_g}(L+ \sigma)}{L\sigma} T^{-\frac{b}{2}}+ \frac{4L_g^2L_f^2(L+\sigma)}{ nL\sigma} \\
        &+ 40C_fL_g^2L_f(\frac{c}{e})^{\frac{c}{2}}\frac{D_y(L+ \sigma)}{L\sigma} T^{\frac{bc}{2}- 1} \sum_{t= 1}^{T} t^{-\frac{c}{2}}+ 2L_f^2L_g^2\sqrt{\frac{L+\sigma}{L\sigma}}\frac{1}{\sqrt{n}} T^{-\frac{a}{2}} \\
        &+ 20L_f^2L_g\sqrt{ C_g\frac{L+\sigma}{L \sigma}} T^{-\frac{a}{2}}+ 8L_f^2L_g^2\sqrt{\frac{L+\sigma}{L\sigma}}\frac{1}{\sqrt{m}} T^{-\frac{a}{2}}+\frac{48L_g^2L_f^2(L+\sigma)}{m L\sigma} \\
      &+ L_f\left( \sum_{t= 1}^T\left( 1- \frac{\sigma\eta}{2}\right)^{T- t}\sqrt{\frac{\mathbb{E}_{S, A}[\mathrm{Var}_\omega(g_\omega(x_t))]}{m}}\right) / \left( \sum_{t= 1}^{T} \left( 1- \frac{\sigma\eta}{2}\right)^{T- t}\right) \\
        &+ (\frac{2c}{e\sigma})^{c-1}D_x T^{-c(1- a)}+ 2L_fL_g T^{-a}+ \frac{C_f^2L_gD_y}{\sigma}\left( \frac{c}{e}\right)^c T^{bc-1} \sum_{t= 1}^{T} t^{-c}+ \frac{4C_f^2L_gV_g}{\sigma} T^{-b} \\
        &+ \frac{2C_f^2L_fL_g^3}{\sigma} T^{2b- 2a}.
  \end{align*}
  Noting that \(\sum_{t= 1}^{T} t^{-z}= \mathcal{O}(T^{1- z})\) for \(z\in (-1, 0)\cup (-\infty, -1)\) and \(\sum_{t= 1}^{T} t^{-1}= \mathcal{O}(\log T)\), we have
    \begin{align*}
        &\mathbb{E}_{S, A}\Big[F(A(S)) - F(x_*) \Big] \\
      =& \mathcal{O}\left(T^{b- a}+ T^{-\frac{b}{2}}+ T^{\frac{c}{2}(b- 1)} (\log T)^{\mathbb{I}_{c= 2}}+ n^{-1}+ n^{-\frac{1}{2}}T^{-\frac{a}{2}}+ T^{-\frac{a}{2}}+ m^{-\frac{1}{2}}T^{-\frac{a}{2}}\right. \\
    &\left.+ m^{-1}+ m^{-\frac{1}{2}}+ T^{c(a- 1)}+ T^{-a}+ T^{c(b- 1)} (\log T)^{\mathbb{I}_{c= 1}}+ T^{-b}+ T^{2b- 2a}\right).
    \end{align*}
    Since \(a, b\in (0, 1]\), setting \(c= 3\), the dominating terms are
    \begin{equation*}
        \mathcal{O}(T^{b- a}), \quad \mathcal{O}(T^{-\frac{b}{2}}), \quad \mathcal{O}(T^{\frac{3}{2}(b- 1)}), \quad \mathcal{O}(T^{-\frac{a}{2}}), \quad \mathcal{O}(T^{3(a- 1)}).
    \end{equation*}
    Setting $a= \frac{9}{10}$ and $b= \frac{3}{5}$ yields
    \begin{equation*}
        \mathbb{E}_{S, A}\Big[F(A(S)) - F(x_*) \Big]= \mathcal{O}(T^{-\frac{3}{10}}).
    \end{equation*}
    Setting $T= \mathcal{O}(\max\{n^{\frac{5}{3}}, m^{\frac{5}{3}}\})$ yields the following bound
    \begin{equation*}
        \mathbb{E}_{S, A}\Big[F(A(S)) - F(x_*) \Big]= \mathcal{O}(\frac{1}{\sqrt{n}}+ \frac{1}{\sqrt{m}}).
    \end{equation*}
    Then we get the desired result for the SCGD update. Next we present the proof for the SCSC update.
  With the same derivation as the SCGD case, we get
  \begin{align*}
    &\mathbb{E}_A[\|x_{t}- x_{t}^{k, \nu}\|]+ 4\mathbb{E}_A[\|x_{t}- x_{t}^{l, \omega}\|]
    \leq 40C_f L_g \frac{\sqrt{L_g}L_gL_f(L+ \sigma)}{L\sigma}\frac{\eta}{\sqrt{\beta}}  \\ &  + 40C_fL_g \frac{\sqrt{2V_g}(L+ \sigma)}{L\sigma}\sqrt{\beta}+ 40C_fL_g(\frac{c}{e})^{\frac{c}{2}}\frac{D_y(L+ \sigma)}{L\sigma} t^{-\frac{c}{2}}\beta^{-\frac{c}{2}}\\
  &+ 20L_f\sqrt{ C_g\frac{L+\sigma}{L \sigma}}\sqrt{\eta}+2L_fL_g\sqrt{\frac{L+\sigma}{L\sigma}}\sqrt{\frac{\eta}{n}}+ \frac{4L_gL_f(L+\sigma)}{ nL\sigma}+ 8L_fL_g\sqrt{\frac{L+\sigma}{L\sigma}}\sqrt{\frac{\eta}{m}} \\
  &+\frac{48L_gL_f(L+\sigma)}{m L\sigma}.
  \end{align*}
  Using Theorem \ref{thm:1}, we have
    \begin{align*} \label{eq:opt_sconvex:21}
      &\mathbb{E}_{S, A} \left[F(x_t)- F_S(x_t)\right] 
      \leq  40C_f  \frac{\sqrt{L_g}L_g^3L_f^2(L+ \sigma)}{L\sigma}\frac{\eta}{\sqrt{\beta}}+ 40C_fL_g^2L_f \frac{\sqrt{2V_g}(L+ \sigma)}{L\sigma}\sqrt{\beta} \\
      &+ 40C_fL_g^2L_f(\frac{c}{e})^{\frac{c}{2}}\frac{D_y(L+ \sigma)}{L\sigma} t^{-\frac{c}{2}}\beta^{-\frac{c}{2}}+ 20L_f^2L_g\sqrt{ C_g\frac{L+\sigma}{L \sigma}}\sqrt{\eta} +2L_f^2L_g^2\sqrt{\frac{L+\sigma}{L\sigma}}\sqrt{\frac{\eta}{n}} \\
      &+ \frac{4L_g^2L_f^2(L+\sigma)}{ nL\sigma} + 8L_f^2L_g^2\sqrt{\frac{L+\sigma}{L\sigma}}\sqrt{\frac{\eta}{m}}+\frac{48L_g^2L_f^2(L+\sigma)}{m L\sigma} +L_f\sqrt{\frac{\mathbb{E}_{S, A}[\mathrm{Var}_\omega(g_\omega(x_t))]}{m}}.
        \numberthis
    \end{align*}
    From \eqref{eq:opt_sconvex:11} we get
      \begin{align*} \label{eq:opt_sconvex:22}
        &\sum_{t= 1}^{T}\left( 1- \frac{\sigma\eta}{2}\right)^{T- t}\mathbb{E}_{S, A}[F_S(x_t)- F_S(x_*^S)] \\
      \leq& \left(\frac{2c}{e\sigma}\right)^cD_x \eta^{-c- 1}T^{-c}+ \frac{2L_fL_g}{\sigma}+ \frac{C_f^2L_gD_y}{\sigma}\left( \frac{c}{e}\right)^c \beta^{-c} \sum_{t= 1}^{T} \left( 1- \frac{\sigma\eta}{2}\right)^{T- t}t^{-c} \\
      &+ \frac{4C_f^2L_gV_g}{\sigma^2} \frac{\beta}{\eta}+ \frac{2C_f^2L_fL_g^3}{\sigma^2} \frac{\eta}{\beta}.
      \numberthis
      \end{align*}
  Multiplying both sides of \eqref{eq:opt_sconvex:21} with \(\left( 1- \frac{\sigma\eta}{2}\right)^{T- t}\), telescoping from \(t= 1, \ldots, T\), then adding the result with \eqref{eq:opt_sconvex:22}, and using the fact $F_S(x_*^S)\leq F_S(x_*)$, we get
  \begin{align*}
      &\sum_{t= 1}^{T}\left( 1- \frac{\sigma\eta}{2}\right)^{T- t}\mathbb{E}_{S, A}[F(x_t)- F(x_*)] \\
      \leq& 40C_f  \frac{\sqrt{L_g}L_g^3L_f^2(L+ \sigma)}{L\sigma}\frac{\eta}{\sqrt{\beta}} \sum_{t= 1}^{T}\left( 1- \frac{\sigma\eta}{2}\right)^{T- t}+  40C_fL_g^2L_f \frac{\sqrt{2V_g}(L+ \sigma)}{L\sigma}\sqrt{\beta} \sum_{t= 1}^{T}\left( 1- \frac{\sigma\eta}{2}\right)^{T- t} \\
      &+ 40C_fL_g^2L_f(\frac{c}{e})^{\frac{c}{2}}\frac{D_y(L+ \sigma)}{L\sigma}\beta^{-\frac{c}{2}} \sum_{t= 1}^{T}\left( 1- \frac{\sigma\eta}{2}\right)^{T- t} t^{-\frac{c}{2}}+ \frac{4L_f^2L_g^2(L+ \sigma)}{L\sigma n} \sum_{t= 1}^{T}\left( 1- \frac{\sigma\eta}{2}\right)^{T- t} \\
      &+ 20L_f^2L_g\sqrt{ C_g\frac{L+\sigma}{L \sigma}}\sqrt{\eta} \sum_{t= 1}^{T}\left( 1- \frac{\sigma\eta}{2}\right)^{T- t}+2L_f^2L_g^2\sqrt{\frac{L+\sigma}{L\sigma}}\sqrt{\frac{\eta}{n}}\sum_{t= 1}^{T}\left( 1- \frac{\sigma\eta}{2}\right)^{T- t} \\
      &+ \frac{48L_g^2L_f^2(L+ \sigma)}{L\sigma m}\sum_{t= 1}^{T}\left( 1- \frac{\sigma\eta}{2}\right)^{T- t}+ 8L_f^2L_g^2\sqrt{\frac{L+\sigma}{L\sigma}}\sqrt{\frac{\eta}{m}} \sum_{t= 1}^{T}\left( 1- \frac{\sigma\eta}{2}\right)^{T- t} \\
      &+ L_f\sum_{t= 1}^{T}\left( 1- \frac{\sigma\eta}{2}\right)^{T- t} \sqrt{\frac{\mathbb{E}_{S, A}[\mathrm{Var}_\omega(g_\omega(x_t))]}{m}}+ \left(\frac{2c}{e\sigma}\right)^cD_x \eta^{-c- 1}T^{-c}+ \frac{2L_fL_g}{\sigma} \\
    &+ \frac{C_f^2L_gD_y}{\sigma}\left( \frac{c}{e}\right)^c \beta^{-c} \sum_{t= 1}^{T} \left( 1- \frac{\sigma\eta}{2}\right)^{T- t}t^{-c}+ \frac{4C_f^2L_gV_g}{\sigma^2} \frac{\beta}{\eta}+ \frac{2C_f^2L_fL_g^3}{\sigma^2} \frac{\eta}{\beta}.
  \end{align*}
  Dividing both sides of the above inequality by \(\sum_{t= 1}^{T}\left( 1- \frac{\sigma\eta}{2}\right)^{T- t}\),
  and setting \(\eta= T^{-a}\) and \(\beta= T^{-b}\) with \(a, b\in (0, 1]\), then from the choice of \(A(S)\) and convexity of \(F\) and Lemma \ref{lem:weighted_avg}, noting that \(\sum_{t= 1}^{T} (1- \frac{\sigma\eta}{2})^{T- t}= \frac{1- (1- \frac{\sigma\eta}{2})^{T- 1}}{1- (1- \frac{\sigma\eta}{2})}\geq \frac{1}{\sigma\eta}\) for \((\eta (T- 1))^{-1}\leq \frac{\sigma}{2}\), we get
  \begin{align*}
      &\mathbb{E}_{S, A}[F(A(S))- F(x_*)] \\
        \leq& 40C_f  \frac{\sqrt{L_g}L_g^3L_f^2(L+ \sigma)}{L\sigma} T^{\frac{b}{2}- a}+ 40C_fL_g^2L_f \frac{\sqrt{2V_g}(L+ \sigma)}{L\sigma} T^{-\frac{b}{2}}+ \frac{4L_g^2L_f^2(L+\sigma)}{ nL\sigma} \\
        &+ 40C_fL_g^2L_f(\frac{c}{e})^{\frac{c}{2}}\frac{D_y(L+ \sigma)}{L\sigma} T^{\frac{bc}{2}- 1} \sum_{t= 1}^{T} t^{-\frac{c}{2}}+ 2L_f^2L_g^2\sqrt{\frac{L+\sigma}{L\sigma}}\frac{1}{\sqrt{n}} T^{-\frac{a}{2}} \\
        &+ 20L_f^2L_g\sqrt{ C_g\frac{L+\sigma}{L \sigma}} T^{-\frac{a}{2}}+ 8L_f^2L_g^2\sqrt{\frac{L+\sigma}{L\sigma}}\frac{1}{\sqrt{m}} T^{-\frac{a}{2}}+\frac{48L_g^2L_f^2(L+\sigma)}{m L\sigma} \\
        &+ L_f\left( \sum_{t= 1}^T\left( 1- \frac{\sigma\eta}{2}\right)^{T- t}\sqrt{\frac{\mathbb{E}_{S, A}[\mathrm{Var}_\omega(g_\omega(x_t))]}{m}}\right) / \left( \sum_{t= 1}^{T} \left( 1- \frac{\sigma\eta}{2}\right)^{T- t}\right) \\
        &+ (\frac{2c}{e\sigma})^{c-1}D_x T^{-c(1- a)}+ 2L_fL_g T^{-a}+ \frac{C_f^2L_gD_y}{\sigma}\left( \frac{c}{e}\right)^c T^{bc-1} \sum_{t= 1}^{T} t^{-c}+ \frac{4C_f^2L_gV_g}{\sigma} T^{-b} \\
        &+ \frac{2C_f^2L_fL_g^3}{\sigma} T^{b- 2a}.
  \end{align*}
  Noting that \(\sum_{t= 1}^{T} t^{-z}= \mathcal{O}(T^{1- z})\) for \(z\in (-1, 0)\cup (-\infty, -1)\) and \(\sum_{t= 1}^{T} t^{-1}= \mathcal{O}(\log T)\), we have
    \begin{align*}
        &\mathbb{E}_{S, A}\Big[F(A(S)) - F(x_*) \Big] \\
      =& \mathcal{O}\left(T^{\frac{b}{2}- a}+ T^{-\frac{b}{2}}+ T^{\frac{c}{2}(b- 1)} (\log T)^{\mathbb{I}_{c= 2}}+ n^{-1}+ n^{-\frac{1}{2}}T^{-\frac{a}{2}}+ T^{-\frac{a}{2}}+ m^{-\frac{1}{2}}T^{-\frac{a}{2}}\right. \\
    &\left.+ m^{-1}+ m^{-\frac{1}{2}}+ T^{c(a- 1)}+ T^{-a}+ T^{c(b- 1)} (\log T)^{\mathbb{I}_{c= 1}}+ T^{-b}+ T^{b- 2a}\right).
    \end{align*}
    Since \(a, b\in (0, 1]\), setting \(c= 6\), the dominating terms are
    \begin{equation*}
        \mathcal{O}(T^{\frac{b}{2}- a}), \quad \mathcal{O}(T^{-\frac{b}{2}}), \quad \mathcal{O}(T^{3(b- 1)}), \quad \mathcal{O}(T^{-\frac{a}{2}}), \quad \mathcal{O}(T^{6(a- 1)}).
    \end{equation*}
    Setting $a= b= \frac{6}{7}$ yields
    \begin{equation*}
        \mathbb{E}_{S, A}\Big[F(A(S)) - F(x_*) \Big]= \mathcal{O}(T^{-\frac{3}{7}}).
    \end{equation*}
    Setting $T= \mathcal{O}(\max\{n^{\frac{7}{6}}, m^{\frac{7}{6}}\})$ yields the following bound
    \begin{equation*}
        \mathbb{E}_{S, A}\Big[F(A(S)) - F(x_*) \Big]= \mathcal{O}(\frac{1}{\sqrt{n}}+ \frac{1}{\sqrt{m}}).
    \end{equation*}
    Then we get the desired result for the SCSC update. We have completed the proof.
\end{proof}